\documentclass[11pt,a4paper]{report}
\usepackage{setspace}
\usepackage{amsmath}
\usepackage{amsfonts}
\usepackage{amssymb}
\usepackage{amsthm}
\usepackage{newlfont}
\usepackage{graphicx}
\usepackage{graphics}
\usepackage{lscape}\usepackage{algorithmic}
\usepackage{algorithm}
\usepackage{subfig}
\numberwithin{algorithm}{chapter}
\newtheorem{theorem}{Theorem}[section]
\newtheorem{cor}[theorem]{Corollary}
\newtheorem{prop}[theorem]{Property}
\newtheorem{lemma}[theorem]{Lemma}
\newtheorem{defn}[theorem]{Definition}

\newtheorem{definition}[theorem]{Definition}
\newtheorem{notation}[theorem]{Notation}
\newcommand{\setm}[2]{\{#1|#2\}}
\newcommand{\form}[1]{\left\langle #1 \right\rangle}

\DeclareGraphicsExtensions{.jpg,.eps,.EPS,.png,.pdf,.ps}

\begin{document}

\title{ \large{ Thesis for the degree Master of Science}  \\
\huge{ Theoretical and Experimental Analysis of the Canadian Traveler Problem}
\author{Submitted by: Doron Zarchy \\ \\Advisor: Prof. Eyal Shimony 
 \\ \\ \\ \\ \\ \\ 
 Department of Computer Science \\ Faculty of Science\\ Ben Gurion University of the Negev
 }}
\maketitle

\baselineskip=24pt
\pagenumbering{roman}

\begin{abstract}
Devising an optimal strategy for navigation in a partially observable environment is one of the key objectives in AI. 
One of the problem in this context is the {\em Canadian Traveler Problem} (CTP). CTP is a navigation problem where an agent is tasked to travel from source to target in a partially observable weighted graph, whose edge might be blocked with a certain probability and observing such blockage occurs only when reaching upon one of the edges end points. The goal is to find a strategy that minimizes the expected travel cost. The problem is known to be P$\#$ hard.
In this work we study the CTP theoretically and empirically.
First, we study the Dep-CTP, a CTP variant we introduce
which assumes dependencies between the edges status. We show that Dep-CTP is intractable, and further we analyze two of its subclasses on disjoint paths graph. 
Second, we develop a general algorithm that optimally solve the CTP called {\em General Propagating AO*} (Gen-PAO). Gen-PAO is capable of solving two other types of CTP called Sensing-CTP and Expensive-Edges CTP. Since the CTP is intractable, Gen-PAO use some pruning methods to reduce the space search for the optimal solution. We also define some variants of Gen-PAO, compare their performance and show some benefits of Gen-PAO over existing work.
\end{abstract}

\pagenumbering{arabic}
\tableofcontents
\pagestyle{headings}

\chapter{Introduction}
Planning under uncertainty is one of the most investigated problems in AI.
In the real world, efficient navigation requires operation in a partially unknown or dynamically changing environment.
Consider a situation where a taxi driver wants to reach his destination in the city in the shortest possible time. The experienced driver knows the road map, and length of each road. Still, the driver does not necessarily have a complete knowledge of the roads' current status. Some of the roads may be blocked due to traffic jams or police blockades. The driver needs to devise a strategy to reach the destination in the shortest expected time.

A formal model for this kind of problem is the {\em Canadian Traveler Problem} (CTP). CTP is a stochastic navigation problem, introduced by 
\cite{PAPA91} where an agent is aimed to travel in a weighted graph $G=(V,E)$ from a source vertex $s\in V$ to a target vertex $t\in V$. Each time the agent traverses an edge it pays a travel cost which is defined by the edge weight. The agent has complete knowledge of the graph structure and the edge costs. However, some of the edges may be blocked with a known probability. The agent observes such blockage only when the agent physically arrives a vertex that is incident to that edge. The goal is to find a strategy for reaching from $s$ to $t$ that minimizes the expected cost.

\cite{PAPA91}  showed that finding the optimal solution for the Canadian Traveler Problem was shown to be P$\#$ hard 
. However some special classes of CTP such as CTP on disjoint path graphs 

and CTP on directed acyclic graphs 
 are solvable in polynomial time \cite{NIKOLOVA08,BNAYA09}.


In this work, we explore certain variations of the CTP. The first variation introduced is {\em CTP with Dependencies} (Dep-CTP). In the original problem, the distribution over the edges is independent. Dep-CTP is a generalization of CTP where we assume that dependencies exist between the status of a particular edge with the status of other edges. Specifically, we are given a Bayesian network that defines the dependencies between the edges. The second variant is {\em CTP with remote sensing}(CTP with sensing), introduced by \cite{BNAYA09}.
In CTP with sensing, an agent may perform sensing on any edge, with a given sensing cost, in order to reveal its status. The third variant is {\em Expensive-Edge CTP}, a variant of CTP in which edges cannot be blocked, but are expensive and incurs a high travel cost when traversed.
      
This work contains two different approaches for studying the CTP, by theoretical analysis and by experimental analysis.  Regarding the theoretical aspect, we attempt to classify certain classes of Dep-CTP by their computational complexity using probabilistic models as belief-MDP and AND/OR graphs, and we show some general properties for CTP with sensing.
Regarding the empirical aspect, we introduce the {\em Gen-PAO algorithm}, a generalization of PAO*
 \cite{Ferguson2004} 
 that optimally solves the CTP, CTP with sensing and Exp-CTP. Gen-PAO uses several pruning methods to reduce the size of the state space search and running time. In addition, we explore the {\em value of clairvoyance} which represent the value of having full knowledge of the graph.     
  
The remainder of this work is organized as follows: Chapter 2 contains formal definitions of the Canadian Traveler Problem and its variants.  In addition it contains a background for decision and probabilistic models and reviews a number of related algorithms. Chapter 3 shows some proofs concerning the hardness of Dep-CTP and for two of its subclasses. In addition, some theoretical properties concerning the CTP with sensing are shown. Chapter 4 introduces the Gen-PAO algorithm and some of the pruning methods it uses. Chapter 5 provides empirical results, comparing the performance of Gen-PAO and some of its variants. In addition, results concerning the value of Clairvoyance are presented. Chapter 6 summarizes this work and discusses possible directions for future research. Appendix A presents some of the instances in which empirical analysis is used .
%
%

\chapter{Background}

\section{Markov Decision Process}
A Markov Decision Process(MDP) is a framework for sequential stochastic decision problems with a fully observable environment.
Formally, MDP is defined by the tuple $<S,A,T,R>$, where
\begin{itemize}
\item S is a finite set of states where $s\in S$ describes the environment at a specific time step.
\item A is a set of actions.
\item $T:(S\times A \times S)\rightarrow [0,1]$ is the transition function where 
$T(s,a,s')$ specifying the probability of entering a state $s'\in S$ given the previous state $s\in S$ and the action $a\in A$.
\item $R:(S\times A \times S)\rightarrow \mathbf{R}$ is the reward function which specifying the reward $R(s,a,s')$ that is received by transitioning from state $s\in S$ to state $s\in S'$ after performing action $a\in A$ in s. 
\end{itemize}

 At each time step, the agent is state $s\in S$ chooses an action $a\in A$, reaches the state $s'\in S$ with probability T(s,a,s'), and obtain reward R(s,a,s').
 A deterministic variant to MDP defines deterministic actions, where each pair of action $a$ and state $s$ specifies deterministically the result state $s'$ i.e. There exist a state $s'\in S$ in which $T(s,a,s')=1$
and for every $\hat{s}\neq s$ $T(s,a,\hat{s})=0$.
The model assumes that the transitions are Markovian in a sense that the probability of reaching a state depends only on the previous state and the action instead of a history of earlier states. The solution of the MDP is a policy. A policy $\pi:S\rightarrow A$ is a mapping from a set of states to a set of actions. At each time step, a given policy is executed, starting from an initial state $s_0$. By having a complete policy, the agent will always know what to do next. However, the stochastic nature of the environment will lead to a different environment history. 
The decision making problem may be a finite horizon or an infinite horizon.
A finite horizon constrains the time steps that the agent exists (or equivalently, considers the rewards after time N as zeros). 
In this case the utility function is usually an additive reward function:
\[
\sum_{t=0}^{t=\infty}{(s_t,a,s_{t+1})} \text{ Where } a=\pi(s_t)
\]
However, in infinite horizon the time sequence is unbounded.  
In infinite horizon the utility function is computed with discount rate $0<\gamma<1$:
\[
\sum_{t=0}^{t=\infty}{\gamma^t(s_t,a,s_{t+1})} \text{ Where } a=\pi(s_t)
\]
This utility is called discounted reward.
Usually, the performance of an agent is the utility of sequence of states, which is measured by the sum of rewards for the states visited. 
The utility of a policy $\pi(s)$ in state s is the expected cost over all possible state sequences, starting from s until the MDP terminates.
The utility of a policy $\pi(s)$ in finite horizon is computed using dynamic programming:
\begin{eqnarray}
V_{\pi,0}(s) &=& R(s,\pi(s),\acute{s})\nonumber\\
V_{\pi,n}(s) &=& \sum_{\acute{s}\in S} T(s,\pi (s),\acute{s}) \cdot(R(s,\pi (s),\acute{s})+ \gamma \cdot V_{\pi,n-1}(\acute{s}))\nonumber
\end{eqnarray}
Where the utility of a policy $\pi(s)$ in finite horizon is given by,
\begin{equation}\label{1} 
V_{\pi}(s)=\sum_{\acute{s}\in S}T(s,\pi (s),\acute{s}) \cdot (R(s,\pi (s),\acute{s})+ \gamma \cdot V_{\pi}(\acute{s}))
\end{equation}
In an infinite horizon we usually have a terminal state.
The optimal policy is a policy that yields the highest expected utility(or lowest, depends on the specification of the problem).
Given a policy $\pi$, the value of the policy in state s can be computed by an algorithm called value iteration.
The value iteration algorithm computes the value of every state s under policy $\pi$ using a reasoning process that goes backwards in time, from the end, in order to determine the optimal sequence of actions. Once choosing the last action, we can determine the best second-last action etc. This process continues until received a best action for all states. 
We compute the value of each state $s$ under the optimal policy $\pi^*$ using the Bellman equations:
\begin{equation}\label{2}
V^* (s)=MAX_{\pi (s) \in A} \sum_{\acute{s}\in S}T(s,\pi (s),\acute{s}) \cdot (R(s,\pi (s),\acute{s})+ \gamma \cdot V^*(\acute{s}))
\end{equation}
This is process is iterated until it reaches equilibrium which indicates the convergence of the algorithm.

\subsection{Policy Iteration}
Another approach for solving MDP is policy iteration. Policy iteration is a feedback strategy obtained by iterative search in the space of policies. 
The algorithm is based on two steps:  The first step is the evaluation where the algorithm evaluate the values of the states given a set of a action for each state is given by: 
\begin{equation}
V^{\pi_k}_{i+1}(s)\leftarrow \sum_{s'}{T(s,\pi_k(s),s')[R(s,\pi_k(s),s')+\gamma V_i^{\pi_k}(s')]}\nonumber
\end{equation}
this can be done by solving a set of linear equations. 
After the values are computed for the given actions, the algorithm makes the second step: improvement. The algorithm considers whether it can improve the policy by choosing a new action for the state. If such action exists, the policy execute the new action. 
\begin{equation}
\pi_{k+1}(s)=argmax_{a} \sum_s'{T(s,a,s')[R(s,a,s')+\gamma V^{\pi_k}(s')]}\nonumber
\end{equation}
The algorithm guarantees that each iteration strictly improves the value of the policy.
Therefore, the policy stops when there are no available actions that improve the policy cost.
The number of possible policies cannot be more than $|S|^|A|$ where $|S|$ is the number of states and $|A|$ is the number of actions. We know that the policy improves at each iteration and the number of possible policies is $|S|^|A|$, thus the algorithm finds the optimal policy within no more than $|S|^|A|$ iterations.

\section{Partially Observable Markov Decision Process}

A partially observable Markov decision process (POMDP) is a generalization of the standard MDP, such that the environment is not fully observable, and allows imperfect information about the current state of the environment. In the real world the input may not always be precise where the data may be received with a noise. In robot navigation for instance, the robot will receive its input through sensors which do not describe the environment precisely. Sonar or voice sensors most of the time will probably be a bit noisy and digital video lose information by using a discrete presentation to describe a continuous environment. 
The POMDP is used as a framework for theoretical decision making and reasoning under uncertainty. Such problems arise in a wide range of application domains including assisting technologies, mobile robotics and preference elicitation.
Many of the real POMDP problems are naturally modeled by a continuous states and observations. For instance, in a robot navigation task, the state will correspond to the coordinates in the space and the observations may correspond to the distance measured by the sonar. A common approach to a continuous model requires of discretization and approximation the continuous component of the grid. This usually leads to an important tradeoff between complexity and accuracy with the change of the coarsens of the discretization.
On discrete time POMDP, each time period the agent is in some state s, chooses an action a, and receive a reward with expected value. Performing the action, the agent makes a transition to a new state according to some state distribution and observes the environment with a given probability to each state.

Formally, POMDP is an extension of the MDP defined by the tuple $<S,A,T,\Omega,R>$ where:
S is a finite set of state that represents the current situation in the environment.
A is a set of actions where the agent choose in each state.
T(Transition function) is a function that maps $S \times A$ into a distribution over the states $\acute{S}$.
$T(\acute{s}|s,a)$ is the probability to reach $\acute{s}$ where the agent is at state s and perform action a.
R is the reward function. R maps any $S \times A \times \acute{S}$ into a number which represents the reward or the penalty.
The observation function $\Omega (s',a,o)$ describes the probability of observation o given that action a was performed in state s' was reached.

Generally, in POMDP  we do not know the current state. The only information that is given on the environment is the observations. Therefore, POMDP defines a vector of probabilities b(s) in the size of the state set, called belief state, which specify for each state s, the probability that the environment is in s.

Similarly to MDP, The goal of the POMDP is to construct a policy $\pi^*$ which maximizes the expected rewards 
$E[\sum^{T}_{t=o}{\gamma^tr(s_t,a_t,s'_t)}]$
where T is the number of time steps left in a finite horizon, or $T=\infty$ in an infinite horizon. Since the agent does not know the exact state of the environment, the reward function is given by the belief state i.e. $R(b,a)=\sum_{s\in S} b(s)R(s,a)$, or in the case of continuous belief space the sum becomes an integral.
The belief state of the environment is based on the previous belief state of the environment. Thus, the agent updates the belief b(s') after being at belief state b(s),choosing action $a$ and receiving an observation $o$ in the following way:
\begin{equation}
P(b'|a,b) = \sum_{o}P(b',o|a,b)=\sum_{o}P(b'|o,a,b)P(o|a,b)\nonumber
\end{equation}
 using the product rule we get 
\noindent
\begin{eqnarray}
P(o|a,b) &=&  \sum_{s \in S}P(o,s'|a,b)=\sum_{s \in S}P(o|s'a,b)P(s'|a,b)\nonumber\\ 
P(o|s'a,b) &=& P(o|s',a)=\Omega(o,s',a)\nonumber\\  
P(s'|a,b) &=& \sum_{s}P(s',s|a,b)\nonumber \\ &=& \sum_{s}P(s'|s,a,b)\overbrace{P(s|a,b)}^{belief state for s}=\sum_{s}P(s'|s,a)b(s)\nonumber
\end{eqnarray}
\noindent
When we put it all together we get:
\begin{equation}
P(b'|a,b)=\sum_{o}P(b'|o,a,b)\sum_{s'}\Omega(s',o,a)\sum P(s'|s,a)b(s)\nonumber
\end{equation}
$P(b'|o,a,b)=1$ where $P(b'|o,a,b)=forward(o,a,b)$

Let L be the number of time $P(b'|o,a,b)=1$

Therefore, 
\begin{equation}
P(b'|a,b)=L\cdot \sum_{s'}\Omega(s',o,a)\sum_{s} P(s'|s,a)b(s)\nonumber
\end{equation}

A generalization on the discrete POMDP is where the space of the belief state is continuous.
In this case, we still assume the the actions and observation are discrete, the propagation is defined by the integral
\begin{equation}
 P(s'|a,s)=\int_{s \in S}{P(s'|s,a)P(s)ds} \nonumber
\end{equation}

\subsection{Value Iteration}
Defining the probability update and the reward function for belief state we can can transform the POMDP into a belief state MDP by casting the POMDP problem into a fully observable MDP, where the belief state of the POMDP are reduced to simple state of the MDP. The MDP here is continuous and over $|S|$-dimensional state space.
The transformation allows applying a value function for each belief state according to the Bellman equation:
\begin{equation}
V^*(b)=MAX_{a \in A}[r(b,a)+\gamma \sum_{b \in B} \tau (b',a,b)V(b^{a}_{o})] \nonumber 
\end{equation}
This means that the value of belief state b is the reward of taking the best action in b plus the discounted expected reward of the resulting belief state $V(b^{a}_{o})$ where $b^{a}_{o}$ is the unique belief state computed based on b,a,o as in equation---. Solving the value iteration by dynamic programing will bring optimal solution at the limit, however, the space size over all the belief states that have to be backed up is enormous.
Because exact value iteration is intractable, a lot of work has focused on approximate algorithms.
One of the most promising approaches for finding an approximate solution point based value iteration (PBVI). In PBVI instead of optimizing the value function over the entire belief state, only specific reachable beliefs are considered. The belief points are selected heuristically and the values are computed only for these points. The heuristic simulate trajectories in order to find reachable beliefs.The success of PBVI depends on the selection of the belief points.  In particular the belief points should cover the space as evenly as possible. The set of belief state is expanded over time in order to cover more of the reachable belief state. Adding more point increases the accuracy of the value function.    


The key to practical implementation of a dynamic programming algorithm is a piecewise-linear and convex representation of the value function. The reward function $r(b,a)$ as defined above is linear. The exact solution of POMDP is based on Smallwood and Sondik(1973) proof which takes advantage of the fact that the exact solution is piecewise-linear convex functions and can be represented by $|S|$ hyperplanes in the space of beliefs. Each hyperplane is a value function V over $|S|$ real numbers represented by $V={v^1,v^2,...,v^k}$ where the value of each belief state is defined as follows: 
\begin{equation}
v(b)=MAX_{0\leq i \leq k}\sum{b(s)v^i(s)}\nonumber
\end{equation}
 Each hyperplane correspond to a single action, and the value iteration updates can be performed directly on these hyperplanes.

\section{The Canadian Traveler Problem}

In the {\em Canadian traveler problem}(CTP) [Papadimitriou and Yannakakis, 1991] 
a traveling
agent is given a tuple (G,P,w,s,t) as input where $G=(V,E)$ connected weighted graph that consists initial source vertex ($s \in V$), and a target vertex ($t \in V$).
The input graph $G$ may undergo changes, that are not known to the agent,
{\em before} the agent begins to act, but remains fixed subsequently.
In particular, some of the edges in $E$ may become {\em
blocked} and thus untraversable. Each edge $e$ in $G$ has a weight, or
cost, $w(e)$, and is blocked with a probability $P(e)$, where $P(e)$ is
known to the agent.\footnote{Note that it is sufficient to deal only with
blocking of {\em edges}, since a blocked vertex would have all of its
incident edges blocked.} The agent can perform {\em move} actions along an
unblocked edge which incurs a {\em travel cost} $w(e)$.
Traditionally, the CTP was defined such that the status of an
edge can only be revealed upon arriving at a node incident to that edge,
i.e., only {\em local sensing} is allowed. In this paper we call this
variant the {\em basic} CTP variant. The task of the agent is to travel
from $s$ to $t$ while aiming to minimize the total travel cost
$C_{travel}$. As the exact travel cost is uncertain until the end, the
task is to devise a traveling strategy which yields a small (ideally
optimal) {\em expected} travel cost.

A somewhat more general version of CTP is {\em CTP with sensing}. CTP with sensing
is a tuple $(G,P,SC,w,s,t)$, where in this variant, in addition to move actions (and local sensing), an agent
situated at a vertex $v$ can also perform a {\em Sense} action and query
the status of any edge $e \in G$. This action is denoted {\em
$sense(v,e)$}, and incurs a cost $SC(v,e)$, or just
$SC(e)$ when the cost does not depend on $v$. The cost function is
domain-dependent, as discussed below. The task of the agent is to
travel to the goal while minimizing a total cost $C_{total} =
C_{travel}+C_{sensing}$.

We further generalize CTP to allow dependencies between edges,
and non-binary edge weight distributions. In this general form,
CTP-Gen is a 5-tuple $(G, W, SC, s, t)$ where $G=(V,E)$ is a graph,
$W$ is a distribution over weights of the edges $E$, 
$SC:V\times E \rightarrow {\cal R}^+$ is a
sensing cost function, $s, t \in V$ are the start and goal vertices,
respectively. The distribution model $W$ is over random variables indexed
by the edges in $E$, abusing notation we will use the edges in
place of the respective random variables. The domain of these random variables
are arbitrary weights or cost sets. $W$ is usually specified as a structured
distribution model over the random variables $e\in E$. Henceforth we assume
that $W$ is specified as a Bayes network 
$(E, A, P)$ over these random variables, where $E$ is the set of random
variables, $A$ is a set of directed arcs so that $(E,A)$ is a directed acyclic
graph, and $P$ are the conditional probability tables, one for each $e\in E$.

We mostly limit ourselves to the
binary case where the edges can be blocked (``infinite weight'') or
open (some known weight, possibly different for each edge).
In these cases, and to simplify the resentation of
the distribution, we use
a uniform binary domain $\{ Blocked, Unblocked \} $
for the edges, and describe the weight of the (unblocked) edges separately,
by a weight function $w : E \rightarrow {\cal R}^+$.
In the degenerate binary case where $W$ is a
Bayes network with no arcs ($A = \emptyset $),
i.e. all random variables are independent, 
the problem reduces back to the basic CTP with sensing. In this case we
usually specify the distribution as a function $p : E \rightarrow [0,1] $,
the probability that each edge $e\in E$ is blocked.

\subsection{CTP with Dependencies in Disjoint Path Graphs}

As CTP-Gen is extremely complicated, we focus on some special cases
w.r.t. the topology of the graph $G$. Specifically, we examine the basic CTP
with no remote sensing where $G$ is a disjoint-path graph (w.r.t. $s, t$).
As this case is known to be solvable in closed form in polynomial time,
we generalize it to the case where edges are dependent, and edge weights
are binary (blocked/unblocked) random variables. Thus we consider CTP-DEP,
defined by the 5-tuple $(G, W, w, s, t)$ where $G$ is an undirected CTP
graph, $W$ is a Bayes network representing the edge blocking
distribution model, $w$ is a function denoting the edge weights
(for unblocked edges), and
$s, t$ are the start and goal vertices respectively, as usual.

As we will show, finding an optimal problem for CTP-DEP is intractable
even for special cases, and we will thus consider cases where $W$
has dependencies only between edges on the same path. Thus the Bayes network
$W$ representing the distribution model has one (or more) unconnected component,
for each set of edges composing a path. We call this simplified
variant CTP-PATH-DEP.

In disjoint path graph, we index the edges such that each edges has two indexes where the first index indicates the path and second index indicates the serial location of vertex in the path.
For instance $e_{i,1}, e_{i,2}, ...e_{i,ki}$ are the edges composing the $i$'th path.
Similarly to edges, we index the vertices such that the first vertex indicates the path and second index indicate the serial location of the vertex in the path.
$s,v_{i,1}, v_{i,2}, ...v_{i,mi},t$ are the vertices composing the $i$'th path. Note that each edge $e_{i,j}$ can be represented by $(v_{i,j-1},v_{i,j})$

\section{AND/OR Graphs}\label{sec:AO*}
Many problems in artificial intelligence can be formulated as a framework for problem solving in a state space search. The AND/OR graph is a directed graph that represent a problem solving process. The solution of AND/OR is a sub-graph of the AND/OR, called {\em solution graph}, that is a derivation for the optimal solution of the original problem. In this work, we use the AND/OR graph for finding optimal solutions to probabilistic reasoning problems and CTP in particular. With a slight abuse of notation we use the same notation \textit{graph} to indicate both CTP graph and AND/OR graph. 

Formally, an AND/OR graph is a tuple $G_{AO} = (N_{AND} \cup N_{OR}, E, T, c, p)$ defines as follows: 
\begin{itemize}
\item $N = N_{AND} \cup N_{OR}$ where $N_{AND},  N_{OR}$ are finite sets of nodes. 
\item $T \subset N_{OR}$ is the set of terminal leaf nodes. 
\item $E \subset (N_{AND}\times N_{OR}) \bigcup (N_{OR}\times T\backslash N_{AND})$ is a set of directed edges between 
the nodes.
\item $c : E \cup T \rightarrow \cal{R}$ is a cost function over the edges and terminal cost. 
\item The graph associate probabilities $p(n, n_0)$ over the  edges $(n,n_0)$ such that,  
 for every $n \in N_{AND}$  $\sum_{\form{n,n'}\in E}{p\form{n,n'}}=1$.
\end{itemize}
 The root node of $G_{AO}$ is denoted by $n_0$.
A policy graph of the AND/OR graph is a subgraph $H=(N_H,E_H)$ of $G_{AO}$ such that 

\begin{itemize}
\item $n_0\in N_H$
\item If $n \in N_H \backslash T$ is AND node, all its children are in H.
\item If $n \in N_H$ is OR node then only one of its children is in H. 
\item Every leaf node (node with no children) in $G_{AO}$ is terminal.
\end{itemize}
and,
\begin{itemize}
\item If $n \in N_H \backslash T$ is AND node, all outgoing edges are in $E_H$.
\item If $n \in N_H$ is OR node, and $n'\in H$ is a child of n, then $\form{n,n'}\in E_H$. 
\end{itemize}

Define $subpolicy$ of node $n$ denoted by $SH_n$ to be a subgraph of $G_{AO}$ that satisfies all properties of $H$ except that the root of $SH_n$ is an arbitrary node $n$ instead of $n_0$. 

The value of each node $n\in N_H$ is defined as follows:
\[
C_H(n) = \left\{ 
\begin{array}{l l}
  c(n) & \quad n\in T\\
  \min\limits_{\form{n,n'}\in E_H} c(\form{n,n'})+C_H(n') & \quad n\in N_{OR} \\
	\sum\limits_{\form{n,n'}\in E_H} p(\form{n,n'})(c(\form{n,n'})+C_H(n'))  & \quad n\in N_{AND}\\
\end{array} \right.
\]
The child of OR node with the minimal value is called called {\em preferred son}.
The cost of the policy graph is defined as $C_H(n_0)$. The policy graph is optimal if there are no other policy graphs with lower cost.
\newline We define a policy subgraph to be a subgraph of $N_H$ such that  
\begin{itemize}
\item $n_0\in N_H$
\item If $n \in N_H \backslash T$ is AND node, all its children are in H.
\item If $n \in N_H$ is OR node then only one of its children is in H. 
\end{itemize}
We define policy subgraph $H_s$ to be a subgraph of $G_{AO}$ that satisfies all properties except that the leafs are not necessarily terminal. 
The cost of the policy subgraph is defined as $C_{H_s}(n_0)$. The best policy subgraph is a policy subgraph with the minimal cost in the AND/OR graph.
 \subsection{AO*}
The AO* is an heuristic based search algorithm that performs a search in the AND/OR graph for finding the optimal policy graph. The AO* performs a search in the AND/OR graph, gradually building up a partial policy graph, assigning heuristic values to the leaves, and propagating the heuristic values up to the root. 
 The heuristics, used to evaluate the real cost of the nodes in AND/OR graph, are admissible, and therefore, finding the optimal solution is guaranteed. The AO* is beneficial when solving problems with a large state space. The AO* algorithm assumes that the AND/OR graph that represents the problem is not given, however the algorithm construct the AND/OR graph by expanding it each iteration, and thereby develop the optimal policy graph subgraph  each iteration. The process ends when all the leaf nodes of the partial policy graph are terminal.   
 \newline The AO* takes advantage of the fact that once a node is known to be in the optimal policy graph it does not required any further expansion. 
Thus, the algorithm maintains a boolean parameter called "SOLVED" for each node in the AND/OR graph which signs the algorithm if node is a part of the optimal policy, i.e. a node n is set {\em SOLVED}, performed by the operation MarkSolved(n), if $n$ is known to be in the optimal policy subgraph. Once a node is SOLVED, it remains SOLVED. A node $n$ is SOLVED if and only if all the nodes in the subpolicy $SH_n$ spanned from $n$, are solved. Hence, when a node $n$ is set SOLVED, the subpolicy $SH_n$ spanned from this node does not require any further update or expansion.
Implementing the ``solving'' process, the AO* performs MarkSolved(n), if node $n$ satisfies one of the following: 
\begin{itemize} 
\item $n$ is a terminal node
\item $n$ is an AND node and all its children are are set {\em SOLVED}.
\item	$n$ is an OR node and its preferred son is set {\em SOLVED}.
\end{itemize}

Basically, each iteration of the AO* algorithm has two phases: expansion and propagation, described as follows:
\begin{itemize}
\item Expansion phase:
\begin{enumerate}
\item Trace down the marked edges (directed edges) from $n_0$ and go downwards until reaching a non-terminal leaf node n  and expands it.
(Finding the expansion nodes requires recurrence exploration through the AND/OR graph since the partial policy graph is changing each iteration.) 

\item For each child $n_i\in n_1,n_2,...,n_k$ of n, if $n_i$ has not been generated, then add it to the policy graph and assign it admissible heuristic. If $n_i$ is a terminal node then assign 0 to its heuristic value, and perform $MarkSolved(n_i)$.
\end{enumerate}
\item propagation phase: 
 In the propagation phase, the heuristic values and marked edges of the expansion nodes are propagated from the leaves onward up to the root. The propagation processed as follows:
\begin{enumerate}
			\item[1] If n is OR node then its heuristic value is updated by,
				\begin{equation}\label{updateOR}
					h(n)=\min\limits_{i} c(\form{n,n_i})+C_H(n_i)
				\end{equation}
					The marked edge is directed from n to the child $n_i$ which achieves the minimum in equation \ref{updateOR}, and n is set SOLVED if and 					only if $n_i$ is set SOLVED.
			\item[2] If n is AND node then its heuristic value is updated by, 
				\begin{equation}\label{updateAND}
						\sum_{i} p(\form{n,n_i})(c(\form{n,n_i})+C_H(n_i)) 
				\end{equation}
				The node n is set SOLVED if and only if all its children are set SOLVED.
\end{enumerate}
The procedure of updating the heuristic values and marking the edges is repeated for all nodes ancestors of n.
 \end{itemize}

\begin{algorithm}
\caption{AO*}
\label{alg1}
procedure main(Graph $G_{AO}$, Node s)
\begin{algorithmic}[1]
\STATE \textbf{if} (s is terminal) return $G_{AO}$;
\WHILE {s is not marked SOLVED}
\STATE Trace down the marked edges in $G_{AO}$ from s, until reaching to all the non terminal leaf node $n_1,n_2,...,n_k$;
\STATE $L\gets \{n_1,n_2,...,n_k\}$;
\WHILE{$L\neq \phi$}
\STATE $n_i\gets Expand(G_{AO},s)$
\STATE remove $n_i$ form L;
\STATE $Z\gets ExtractAnsectors(n_i)\cup n_i$;
\STATE Propagate($G_{AO}$,Z)
\ENDWHILE
\ENDWHILE 
\newline
\end{algorithmic}

vertex function Expand(Graph $G_{AO}$, Node s)
\begin{algorithmic}[1]\label{expansionProcedure}
\STATE select $n_i\in L$  expand it and add its children $C=\{c_1,c_2,...,c_m\}$ to $G_{AO}$; 
\FOR {each $c_i\in C$}
\IF {$c_i$ is terminal}
\STATE $f(c_i)\gets 0$;
\STATE MarkSolved($c_i$);
\ELSE
\STATE $f(c_i)\gets h(c_i)$;
\ENDIF
\ENDFOR
\STATE return $n_i$
\newline
\end{algorithmic}

procedure Propagate(Graph $G_{AO}$, NodeSet Z) 
\begin{algorithmic}[1]
\WHILE{$Z\neq \phi$}
\STATE select $z_i\in Z$ such that $z_i$ has no children in Z;
\IF {$z_i$ is AND node}
\STATE $ f(z_i)\gets \sum\limits_{\substack{j\in successor(z_i)\\ a=(z_i,j)}}[tr(z_i,a,j)h(j)+c(a)]$; 
\IF {all successors of $z_i$ are marked SOLVED}
\STATE MarkSolved($z_i$)
\ENDIF
\ENDIF
\IF {$z_i$ is OR node}
\STATE $f(z_i)\gets \min\limits_{\substack{z_j\in successor(z_i)\\ a=(z_i,z_j)}}{[tr(z_i,a,z_j)h(z_j)+c(a)]}$;
\IF {$z_j$ is marked SOLVED}
\STATE mark the edge represents the chosen action a;
\STATE MarkSolved($z_i$);
\ENDIF
\ENDIF
\STATE remove $z_i$ from Z;
\ENDWHILE
\end{algorithmic}
\end{algorithm}
 
%
Properties of the AO*:
\begin{itemize}
\item	The heuristic values are optimistic estimations (lower bound) to the real value of the state, where each update raises up the heuristic value and reduces its imprecision relatively to the real value. 
\item	The AO* is beneficial when it applied to a large state space. One reason for this is that AO* considers only states that are reachable from the initial state. Secondly, the informative heuristic function directs the focus on states that are in the course of a good policy graph(partial policy graph). As a result, the AO* may find an optimal solution by exploring a small fraction of the entire state space.
\end{itemize}

\subsection{CTP and AND/OR graphs}
AND/OR graph is a natural structure for representing the state space of CTP, where the policy of CTP is represented by the policy graph. The problem solving process is a search for an optimal policy graph in a policy graph space. Here, the OR nodes represents the agent's decision in a current state out of all its available actions. Where, in basic CTP, the available actions are all the moves available from a certain vertex, while in Sensing-CTP, the available actions are all the available remote sensing actions, in addition to all available moves from a certain vertex(the remote sensing is defined available if it is performed on an unknown edge). The AND nodes represent the actions. Since the CTP is a stochastic problem, each action may result several possible states, which is represented by the AND node's children. The states of the environment in CTP are represented by the OR nodes. The states are the belief states of the agent in a current time step, where each belief state is represented by its $form$ (i.e every belief state $b$ is represented by the tuple $\form{b}=\form{Loc(b),stb(e,b_1),...,stb(e,b_n)}$).
Henceforth, all functions,predicate and lemmas presented in section $\ref{secHeuristics}$ can be applied to the states in the AND/OR graph.
We call the set of states that appears in the AND/OR graph {\em the expanded states} and denote it by $Z$.
Although the AND nodes do not represent the states(they are called ``semi state''), they maintain heuristic values as described in AO* algorithm which is specified for propagation. Since the environment is static, once an agent observes an edge, its status is remained unchanged. A terminal state $b$ is a state in which its location variable is the target($Loc(b)=t$). A node is a terminal leaf node if the state with which it associate is terminal. 
\begin{defn}
A belief state $b$ is expanded belief state if there is an OR node in the AND/OR graph that is associated with $b$.
\end{defn}

\begin{figure}[h!]
        \centering
        \centerline{\includegraphics[width=7.5 in,height=3.5in]{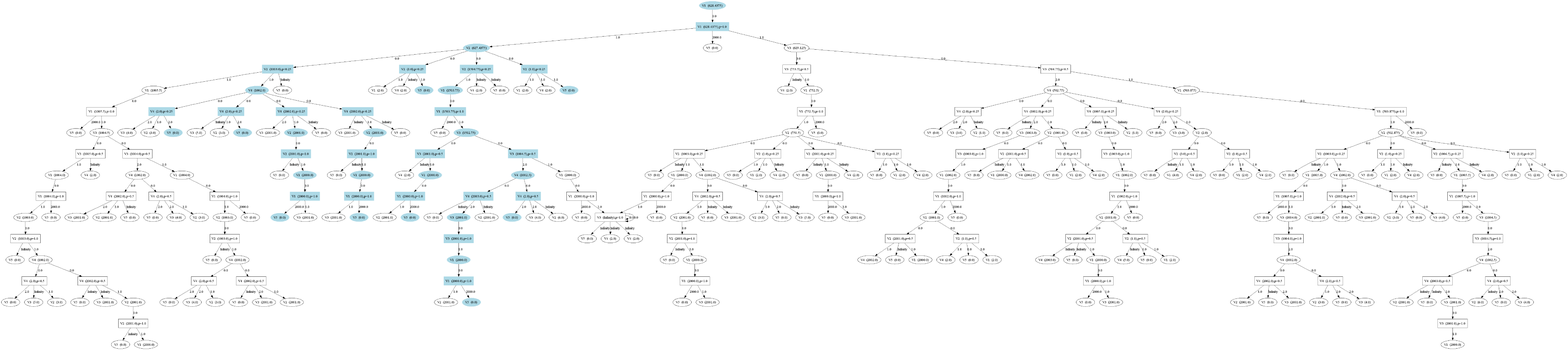}}
        \caption{Demonstration of SIMBlocked}
        \label{fig:ContBfig}
\end{figure}

\section{Models for the Canadian Traveler Problem}
\subsection{POMDP for CTP}
In this section we show that CTP can be modeled by POMDP.
Let $I=(G,P,w,\hat{s},\hat{t})$ be an instance of of basic CTP, and $I'=(G,P,w,SC,\hat{s},\hat{t})$ be an instance of CTP-with-sensing, where $G=(V,E)$. Given POMDP $M=(S,A,Tr,R,Z,O,s_0)$, we show how I and I' can be modeled by M as follows: 
\begin{itemize}
\item The state space S. The states space S of I (or I') represent the possible environment of the world. Each state $s$ indicates the location of the agent, and the status of all edges in E. The location of the agent in state s is denoted by $Loc_S(s)$ where $v=Loc_S(s)$ is the vertex that the agent is situated. Each edge $e\in E$ is associated with a status variable $st_e\in \{O_e,B_e\}$ where $st_e(s)=O_e$ indicates that $e$ is open in s, and $st_e(s)=Blocked$ indicates that $e$ is blocked in s. Thus, we define the state space S to be $V\times \prod_{e\in E}{\{O_e,B_e\}}$.
\item The action set A. In the basic CTP the set of actions A includes only one type of action $a=Move(e)$ in which agent that performs $a$, moves along an edge $e\in E$ if $e$ is open. While in the CTP-with-sensing, the set of actions A includes in addition to the Move actions, the sensing actions, in which agent that performs $a=Sense(e)$, senses an edge $e\in E$. This action can be performed from any vertex $v\in V$. 
\item The transition function Tr. Given $s,s'\in S$, and $a=move(v_i,v_j)$, we define Tr by the following:
$Tr(s,a,s')=1$ if it satisfies the following:
\begin{itemize}
\item For all edges $e\in E$ the status of the edge in s is equal to the state in s', i.e $st_e(s)=st_e(s')$.
\item $v_i=Loc_S(s)$ and $v_j=Loc_S(s')$.
\item The edge e=($v_i,v_j$) is open in s, i.e. $st_e(s)=O_e$.
\end{itemize}
Otherwise Tr(s,a,s')=0.\newline
If $a=Sense(e)$ where $e\in E$ we get $Tr(s,a,s')=1$ if and only if $s=s'$ since the Sense action does not change the state of the environment.
\item The reward(cost) function R. Given $s,s'\in S$, we define R as follows:
In case that $a=move(e)$ then, 
\[
R(s,a,s')=\left\{
\begin{array}{l l}
w(e) & \quad \text{if }Tr(s,a,s')=1 \\
0\; & \quad otherwise
\end{array}
\right.
\]
In case that $a=Sense(e)$ then,
\[
R(s,a,s')=\left\{
\begin{array}{l l}
SC(e) & \quad \text{ if }Tr(s,a,s')=1 \\
0 & \quad otherwise
\end{array}
\right.
\]

\begin{notation}
Let X be a set. Denote the power set of X by $\cal{P}(X)$.
\end{notation}

\item The observation set Z. Let $Z'=\{O_e,B_e\}^E$. We define Z to be the power set of Z', Namely $Z=\cal{P}(Z')$ .  
\item The observation function O. Given $s\in S$ ,$a\in A$ and $z\in Z$ we define O as follows:
In case that $a=move(e)$ where $e=(v_i,v_j)$, the only observation that received are the edges incident to vertex $v_j$, then, 
\[
O(s,a,z)=\left\{
\begin{array}{l l}
1 & \quad \text{ if } z=\setm{st_e(s)}{e\in Inc(v_j)} \\
0& \quad  otherwise
\end{array}
\right.
\]

Where in case that $a=sense(e)$, the only observation that received is the sensed edge e, then, 
\[
O(s,a,z)=\left\{
\begin{array}{l l}
1 & \quad \text{ if } z=st_e(s)  \\
0& \quad  otherwise
\end{array}
\right.
\]

\item $s_0$ is the initial state.

\end{itemize}

\begin{notation}
The optimal policy of $M_S$ is denoted by $\pi^*$.
\end{notation}

\begin{notation}
Let X be a set. Denote the power set of X by $\cal{P}(X)$.
\end{notation}

\subsection{Belief State for Representing the Environment of CTP}
A belief state, which is defined as a distribution over all possible states, is a representation of the agent's knowledge about the environment. 
In CTP, the belief states can be represented by 
the location of the agent and the status of each edge in the graph.

\begin{defn}
We say that status of edge $e$ is:
\begin{itemize}
	\item  ``known to be blocked'' if $e$ has been already sensed and found to be blocked.
	\item  ``known to be open'' if $e$ has been already sensed and found to be open.
	\item  ``unknown'' if $e$ has not been sensed.
\end{itemize}

\end{defn}

\begin{defn}
Define $stb:E\times B\rightarrow \{Open,Blocked,Unkown\}$ as follows:\newline
 $stb(e,b)$ is the edge status of $e$ in belief state $b$, where
\[
stb(e,b)=\left\{
\begin{array}{l l}
Open &\quad $if edge $e$ is known to be Open in $b \\
Blocked &\quad $if edge $e$ is known to be Blocked in $b\\
Unknown &\quad\; otherwise $ (if the status of edge $e$ is Unknown in $b) \\
\end{array}
\right.
\]
\end{defn}

\begin{defn}\label{locb}
Define $Loc:B\rightarrow V$ as the location of an agent in a belief state, where $Loc(b)$ outputs the physical location of an agent that is in belief state $b$ ,i.e. $Loc(b)=Loc_S(s)$ where $s$ is an arbitrary state $s\in S$ which satisfies $b(s)>0$.
\end{defn}
Note that definition $\ref{locb}$ assumes that there cannot be two state $s_1,s_2$ which satisfy $b(s_1)>0$, $b(s_2)>0$ such that $Loc_S(s_1)\neq Loc_S(s_2)$ since by definition, the agent always knows its own location, and thus, for every belief state $b$, if exists $s\in S$ in which $Loc(b)\neq Loc_S(s)$ then $b(s)=0$.

Thus, we can define an alternative way for representing a belief state,
\begin{defn}\label{formB}
Let $n=|E|$. The form of $b$, denoted by $\form{b}$, is defined to be the tuple $\form{b}=\form{Loc(b),stb(e_1,b),stb(e_2,b),...,,stb(e_n,b)}$, 
\end{defn}

\begin{definition}
Let $b$ be a belief state, we define the following sets:
\begin{enumerate}
	\item $Unknown(b)$ is the set of all edges $e\in E$ in which $stb(e,b)=Unknown$  
	\item $Blocked(b)$ is the set of all edges $e\in E$ in which $stb(e,b)=Blocked$  
	\item $Open(b)$ is the set of all edges $e\in E$ in which $stb(e,b)=Open$ 
\end{enumerate}

Let $b$ be a belief state. Then, there is a mapping from $\form{b}$ to $b$. Namely, for every $s\in S$ there is a mapping from $\form{b}=\form{Loc(b),stb(e_1,b),stb(e_2,b),...,,stb(e_n,b)}$ to $b(s)$ as follows: 
\[
b(s)=\left\{
\begin{array}{l l}
0 & \quad\text{ if } G(e,s,b)=0 \\
\prod\limits_{\setm{e\in Unknown(b)}{st(e,s)=Blocked}}p(e)\prod\limits_{\setm{e\in Unknown(b)}{st(e,s)=Open}}{1-p(e)} & \quad Otherwise\\
\end{array}
\right.
\] 
\end{definition}

Where $G:E\times S \times B\rightarrow \{0,1\}$ is defined as follows: \newline
$G(e,s,b)=0$ if one of the following is satisfied:
\begin{enumerate}
\item $st(e,s)=Open,stb(e,b)=Blocked$
\item $st(e,s)=Blocked,stb(e,b)=Open$
\item $Loc(b)\neq Loc_S(s)$
\end{enumerate}
otherwise $G(s,b)=1$

\begin{cor}\label{mappingb}
Since there is a mapping from $\form{b}$ to $b$, we can use the {\em form} $\form{b}$ instead of the belief state $b$ itself, for representing the belief state of an agent.
\end{cor}


\begin{defn}\label{probB}
Let $P_b(e,b)$ be the probability that edge $e$ is blocked given that the agent is in belief state $b$. Namely 
$P_b(e,b)=\sum_{i}{b(s_i)}$ such that $st(e,s_i)=Blocked$. 
\end{defn}

In the basic variant of CTP, the probabilities associated with the edges are independent, and hence, as long as $stb(e,b)=Unknwon$, we have $P_b(e,b)=P(e)$.
\begin{cor}
From definition \ref{probB} we get that the probability that edge $e$ is blocked given the agent is in belief state $b$ is given by:
\[
P_b(e,b)=\left\{
\begin{array}{l l}
0 &\quad $if $stb(e,b)=Open \\
1 &\quad $if $stb(e,b)=Blocked\\
P(e) &\quad\; $if $stb(e,b)=Unknwon\\
\end{array}
\right.
\]
\end{cor}

\newpage
\subsection{Belief MDP for CTP}\label{sec:beliefMDP}
Given a POMDP $M=(S,A,Tr,R,Z,O,s_0)$ of instance $I=(G,P,w,\hat{s},\hat{t})$ of CTP (or $I'=(G,P,w,SC,\hat{s},\hat{t})$ of CTP-with sensing). Let $B$ be the belief state space of M, we define a belief MDP $M_S=(B,A,Tr,R,b_0)$ of I, based on M, where the states space B is over the state space S. CTP is a special case of POMDP(called Det-POMDP) where transition function Tr and reward function R can be simplified here as follows:
\begin{itemize}
\item The transition function Tr. In general belief MDP, given $b,b'\in B$, $a\in A$, Tr is given by:
\[
Tr(b',a,b)=P(b'|a,b) = \sum_{z}P(b',z|a,b)=\sum_{z}P(b'|z,a,b)P(z|a,b)
\]
Given $a=move(v_i,v_j)$, we define $\hat{E}$ to be the set of all edges incident to $v_j$ that are unknown in $b$ and known in $b'$ (the edges that are revealed by the local sensing), i.e.
\[
\hat{E}=\setm{e\in E}{stb(e,b)=Unknown, stb(e,b')\neq Unknown, e\in Inc(v_j),b\in B,b'\in B}
\]
Then $Tr(b',a,b)>0$ if and only if,
\begin{itemize}
\item For all $e\in E\backslash \hat{E}$ $stb(b,e)=stb(e,b')$. The status of the edges do not change as well as the information about any unraveled edge that is not sensed.
\item $v_i=Loc(b)$ and $v_j=Loc(b')$.
\item The edge e=($v_i,v_j$) is open in b, i.e. $stb(e,b)=Open$.The edge has to be open in order to traverse it.
\end{itemize}
In this case,
\[
Tr(b',a,b)=\prod_{e\in \hat{E},stb(b',e)=Blocked}{p(e)}\prod_{e\in \hat{E},stb(b',e)=Open}{1-p(e)}
\]
Given $a=Sense(e')$, $Tr(b',a,b)>0$ if and only if,
\begin{itemize}
\item $stb(b',e')\neq Unknown$. The edge e' is known after the performing Sense(e'). 
\item For all $e\in E\backslash e'$ $stb(b,e)=stb(e,b')$. The state is not effected by Sense action and the only information that received is the status of e'.
\item $Loc(b)=Loc(b')$. The location is not effected by Sense action.
\end{itemize}
In this case,
\[
Tr(b',a,b)=\left\{
\begin{array}{l l}
p(e) & \quad \text{ if } stb(e,b)=Blocked\\
1-p(e)& \quad  otherwise
\end{array}
\right.
\]

\item The reward function R. In general, the reward function is defined by $R(b,a,b)=\sum_{b'\in B}{b(b')}\sum_{b\in B}{b(b)R(b,a,b')}$.
Denote the action cost of $a$ by $C(a)$, where $C(a)=w(e)$ if $a=Move(e)$ and $C(a)=SC(e)$ if $a=Sense(e)$. Hence, 
\begin{equation}\label{eq:rewardFunction}
R(b,a,b')=\left\{
\begin{array}{l l}
C(a) & \quad \text{ if and only if  }Tr(b,a,b')>0 \\
0 & \quad otherwise
\end{array}
\right.
\end{equation}

We define $R(b,a)=\sum_{b'\in B}{R(b,a,b')Tr(b,a,b')}$. Therefore, $R(b,a)=C(a)$ if there exist ${b'\in B}$ such that $Tr(b,a,b')>0$.
Otherwise $R(b,a)=0$.
Note that in case that $a=Sense(e)$, there always exist $b'$ reachable from $b$ in which $Tr(b,a,b')>0$, thus $R(b,a,b')=SC(e)$ always holds.
\item $b_0$ is the initial belief state.
\end{itemize}

\begin{defn}
We say that action $a$ can be performed in belief state $b$ if there is a belief state $b'$ such that $Tr(b,a,b')>0$
\end{defn}


\section{Related Work}
\subsection{Different Variation of CTP}
The Canadian traveler problem is known to be $p\#$ hard \cite{PAPA91}. In the lack of approximation solutions, different versions of special classes of graphs have been suggested where the exact solution can be found in polynomial time. \cite{BARNOY91}  have investigated the case of Recoverable CTP , where each vertex is associated with a specific recovery time to reopen any blocked edge that is incident to it. When an agent finds a blocked edge $e$ it can either traverse another edge or wait a period of time and check if $e$ has been opened. The basic CTP is a special case of the Recoverable CTP where all the recovery times are infinitely large. There are two variation to the Recoverable CTP, deterministic and stochastic.  In the deterministic variation the assumption is that the number of edges that may be blocked is bounded. In the stochastic variation, each edge is associated with a probability of being blocked while it assumes that the recovery time is not long relative to the travel time. The two cases were proved to be polynomial in the number of edges and vertexes and in the maximal number of blocked edges.
\cite{NIKOLOVA08} investigated a CTP variant where the environment is dynamic, in a sense that the status of each edge $e\in E$ is generated randomly with a given probability $P(e)$ whenever the agent reaches an incident vertex of $e$. This variant can be modeled by MDP, where the states represent only the current location of the agent. Since MDP is solvable in polynomial time, this variant is solvable in polynomial time as well. 
Notice that basic CTP is much harder to solve, since the edges status is remained fixed and thus the state space is exponentially larger (in the number of edges). 
Nikolova et al. have shown that CTP on \textit{directed acyclic graph} (DAG) can be solved in polynomial time by using a dynamic programming.

\subsection{Disjoint Path Graphs}
\textit{Disjoint path graph} is an undirected graph $G=(V,E)$ with source $s\in V$ and destination $t\in V$ such that all paths $p_1,...,p_k$ in G are between $s$ and $t$, and these paths are pairwise disjoint.
\cite{BNAYA09} have shown that CTP on disjoint path graph is solvable in polynomial time. The proof is based on the property the the optimal policy is committing . This guarantees that whenever an agent follows a path, the optimal action is to continue the path until reaching the target unless it hits a blocked edge. 
The optimal policy of CTP on disjoint path is to follow the paths by their order of $D_i$ ($D_i$ is parameter associated with each path in G) Meaning,  the optimal policy is to travel the path with the minimal $D_i$ till reaching the target unless the path is blocked. If the path is blocked then return to $s$ and travel the path with second minimal $D_i$ and so on.
$D_i$ is defined as,
\begin{equation} 
D_i=\frac{E[BC_i]}{Prob(path\; i\; is\; traversable)}+{W_{i,k_i}} \nonumber\\
\end{equation} 
Where $BC_i$ denotes the backtracking cost of path i which is the cost of traversing path i until hitting a blocked edge and then returning back to the $s$ when the path is not traversable , or 0 when the path is traversable. The expected cost of $BC_i$ is
\begin{equation}
 E[BC_i]=2\sum_{j=2}^{k_i-1}{W_{i,j-1}p(e_{ij})}\prod_{m=1}^{j-1}{1-p(e_{im})}\nonumber\\
 \end{equation}
 Where $W_{i,j}=\sum_{j}^{m=1}{w\;(e_{im})}$ . 

Another variation of CTP on disjoint path graphs is when the edges cannot be blocked but instead have two possible finite costs: a \textit{cheap} and and \textit{expensive} \cite{NIKOLOVA08}. A simple case of this variation is when the value edges is binary, i.e., 0 or 1. In this case the optimal policy would be to explore all the edges with cost 0 of each path until it reaches an edge with cost 1 on the path, and then return to the path with the fewest unexplored edges and follow it until reaching the target. 
A more general case of this variation is when the edges are associated with the cost 1 or K. In this case the optimal policy has the property that once an edge with cost $k$ has been crossed, it is optimal to continue along the same path until reaching the target. Taking advantage of the special structure of the policy induced by this property, allows to define an MDP with concise representation that decides in what order to explore the paths and how many, before committing a path. This two cases were proved to be solvable in polynomial time.

\subsection{CTP with Sensing}
The CTP with sensing is a harder problem than the basic CTP since a simple reduction can be constructed from any instance of CTP: The graph of the basic CTP is the graph of the CTP with sensing, however, the sensing cost of all edge are large enough, such that sensing an edge is never worthwhile. As such, the expected cost of the two optimal policies is equal.
\subsubsection{Heuristic search algorithms}\label{sec:FSSN}
In order to facilitate the search for solution of CTP with sensing, some heuristic based algorithm have been suggested. The algorithms do not provide an optimal solution, however, they may be much simpler. 
\cite{BNAYA09} have suggested the FSSN algorithm that is based on the \textit{free space assumption} heuristic.
The free space assumption 
\cite{KOENIG03}
assumes that edges are traversable unless specifically known otherwise. 
 The FSSN plans a path $p$ from some vertex $v\in V$ to $t$ with the shortest path under the free space assumption. The agent can either attempt to traverse P without sensing or may decide to interleave sensing    
actions into the movement actions, according to a sensing policy that is embedded in the algorithm.

Number of sensing policies to the FSSN have been suggested:

\textsl{Never Sense} is a brute force policy that never senses any remote edge. This policy never incurs any sensing cost but it may lead to an increase travel cost. 
 
\textsl{Always Sense} is a brute force policy that senses all the unknown edges in the path before it moves along it.

\textsl{Value of information} a policy that decides what edges to be sensed according to their value of information.
\subsection{Propagating AO*}
AO* harness the benefits of the heuristic search to avoid searching states that are undesirable. However, in many situations AO* examines far more states than necessary. 
Propagating AO*(PAO*) \cite{Ferguson2004} is an extension of the AO* that takes one step forward for facilitating the search. PAO* propagates the heuristic values on a larger scale in which minimizes the expansion of non-terminal nodes.
PAO* is based on a specific variation of the AO* algorithm;
Ferguson et al. constructed an algorithm that solves a variation of the CTP where most of the graph (edges) is observable such that only a single unknown edge (called pinch point) can be incident to a vertex (In the original paper the pinch points are called ``faces'' ).
As such, any chance node (AND node) has at most two children that represent a traversable edge and a blocked edge.
PAO* is described as follows:
The expansion phase is processed exactly as the AO* where the PAO* grows the best partial policy graph by expanding the non terminal leaf nodes, and assigning heuristic values to its children. 
Similarly to AO*, PAO* propagates the heuristic values onward up to the root. However, PAO* propagate the heuristic values sideways and downwards to the children as well.
Furthermore, the algorithm takes advantage of the fact that the AND node has only two children (traversable and blocked) such that the parent node heuristic value should never be less than the traversable child value. Thus, PAO* propagate the heuristic value of the parent to the traversable child 
if the heuristic value of the traversable child is higher. Similarly, the heuristic value of the parent should never be greater than the value of the blocked child. Therefore, PAO* propagate the value to the non-traversable child in case that the heuristic value of the parent is higher.

\chapter{Theoretical Analysis of CTP}

\section{CTP with Dependencies}

\begin{theorem}\label{path-dep}
CTP with dependencies is at least as hard as CTP with sensing.
\end{theorem}
\textbf{Proof outline:}
By reduction from CTP-with-sensing to CTP-with-dependencies.

Proof: 
%
%
%
%
%
%
%
%
%
%
%
%
%

\textbf{Proof.}
Let I=(G,W,C,SC,s,t) be an instance of CTP-with-sensing. We construct an equivalent instance I'=(G',W',C',s',t') of CTP-with-dependencies and show that there is a
one-to one equivalence between I and I'. Construction of I' is as follows, G' contains G entirely, and in addition, each vertex in G is attached to two-edge dead-end path, that simulate the sensing operation of I. One path for each possible sensing operation in I.

Formally, the construction of I' is as follows:\newline
First, we construct $\hat G(\hat{V},\hat{E})$ by copying the graph G(V,E) using the following functions: 
\begin{itemize}
\item $g_v:V \rightarrow \hat{V}$ is a bijection function that copies V into $\hat{V}$ such that for each $v_i\in V$, $g_v(v_i)$ is the copied vertex of $v_i$.
\item $g_e:E \rightarrow \hat{E}$ is a bijection function that copies E into $\hat{E}$ such that for each $e_j\in E$, $g_e(e_j)$ is the copied edge of $e_j$.
\end{itemize}
Let $\hat{V}$ be the set of all the vertices that were copied from V, meaning $\hat{V}=\bigcup_{1\leq i\leq n}{g_v(v_i)}$. 
Let $\hat{E}$ be the set of all the edges that were copied from E, meaning $\hat{E}=\bigcup_{1\leq j\leq m}{g_e(e_j)}$.
\begin{notation}
$\hat{v}_i\in \hat{V}$ denote the copied vertex $f(v_i)$.
\end{notation}

\begin{notation}
$\hat{e}_i\in \hat{E}$ denote the copied edge $g_e(e_i)$
\end{notation}

We construct a new graph G'(V',E') by extending $\hat{V}$ and $\hat{E}$ using the following functions:
\begin{itemize}
\item $f_{v1}:\hat{E}\times \hat{V}\rightarrow V$ and $f_{v2}:\hat{E}\times \hat{V}\rightarrow V$ are one to one functions that generates a vertex for each element in $\hat{E}\times \hat{V}$. Meaning, given $v_i\in \hat{V}$ and $e_i\in \hat{E}$, $f_{v1}(v_{i},e_j)=v_{ij1},f_{v2}(v_{i},e_j)=v_{ij2}$.
\item $f_{e1}:\hat{E}\times \hat{V}\rightarrow E$ and $f_{e2}:\hat{E}\times \hat{V}\rightarrow E$ are one to one functions that generates an edge for each element in $\hat{E}\times \hat{V}$ such that given $v_i\in \hat{V}$ and $e_i\in \hat{E}$,  $f_{e1}(v_i,e_j)=e_{ij1},f_{e2}(v_i,e_j)=e_{ij2}$ and in addition, $e_{ij1}=(v_i,v_{ij1})$ and $e_{ij2}=(v_{ij1},v_{ij2})$.
\end{itemize}

Let graph $G'(V',E')$ defined as follows:
\begin{eqnarray}
V' &=& \left\{\hat{V}\cup V'_{ij1}\cup V'_{ij2}\right\} \nonumber\\
E' &=& \left\{\hat{E}\cup E'_{ij1}\cup E'_{ij2}\right\} \nonumber\\
\end{eqnarray}

where,
\begin{eqnarray}
V'_{ij1}=\bigcup_{1\leq i\leq n,1\leq j\leq m}{v'_{ij1}} \nonumber \\
V'_{ij2}=\bigcup_{1\leq i\leq n,1\leq j\leq m}{v'_{ij2}} \nonumber \\
E'_{ij1}=\bigcup_{1\leq i\leq n,1\leq j\leq m}{e'_{ij2}} \nonumber \\ 
E'_{ij2}=\bigcup_{1\leq i\leq n,1\leq j\leq m}{e'_{ij2}} \nonumber \\
\end{eqnarray}
\begin{notation}
Given $e_{ij1}\in E_{ij1}, e_{ij2}\in E_{ij1}$, we define a two edge dead end path $p_{ij}=\left\langle e_{ij1},e_{ij2}\right\rangle$. 
\end{notation}
Note that $G'(V',E')$ can be viewed as "attachment" of paths $\bigcup_{1<j\eqslantless|E|}{P_{ij}}$ to each $g(v_i) \in \hat{V}$.

$W(X,Y)$ is the Bayesian network that is associated with edges in G where X is the set of nodes and Y is the set of arcs. Similarly $W'(X',Y')$ is the Bayesian network that is associated with edges in G' where X' is the set of nodes and Y' is the set of arcs. Let x be a node in X, and x' be a node in X'. We define W' by W as follows:

\begin{itemize}\label{bayesianDef}
\item For each $x\in X$, $x'_{g_e(e_i)}=0 \Leftrightarrow x_{e_i}=0$($st_{e_i}=st_{g_e(e_i)}$).
\item For each $e'_{ij1}\in E_{ij1}$, $P(x'_{e'_{ij1}}=0)=1$, i.e. all edges in $E_{ij1}$ are open.
\item For each j, $x'_{e'_{1j2}}=0,x'_{e'_{2j2}}=0,...,x'_{e'_{nj2}}=0\Leftrightarrow x'_{g_e(e_j)}=0$, i.e. for each j, all edges $e'_{ij2}\in E'_{ij2}$ are open if and only if $x_{g_e(e_j)}$ is open.
\end{itemize}

The weight function $C'$ is defined by:
\begin{enumerate}
\item $\forall e_i\in E, C'(g_e(e_i))=C(e_i)$. 
\item $\forall e_{ij1}\in E'_{ij1}, C'(e_{ij1})=\frac{SC(e_j)}{2}$ where $SC(e_j)$ denote the sensing cost of $e_j$.
\item $\forall e_{ij2}\in E'_{ij2},C'(e_{ij2}) = \infty $. 
\end{enumerate}

The computational time that takes to generate this reduction is polynomial, since the size of $|E'|=|E'_{ij1}|+|E'_{ij2}|+|\hat{E}|$, where $|E'_{ij1}|=|E'_{ij2}|=|E|\times|V|$ and $|\hat{E}|=|E|$, therefore $|E'|=|E|+2|E|\cdot|V|=|E|(1+2|V|)$. Furthermore, the size of $|X'|$ is $|E'|$  since each node in X' is associated with an edge in E' , and $|Y|=|V||E|$ since each node in X' that is associated with edge in $\hat{E}$ is connected to $|E|$ nodes.


\begin{figure}[h]
        \centering
        \centerline{\includegraphics[width=5.5in,height=3.4in]{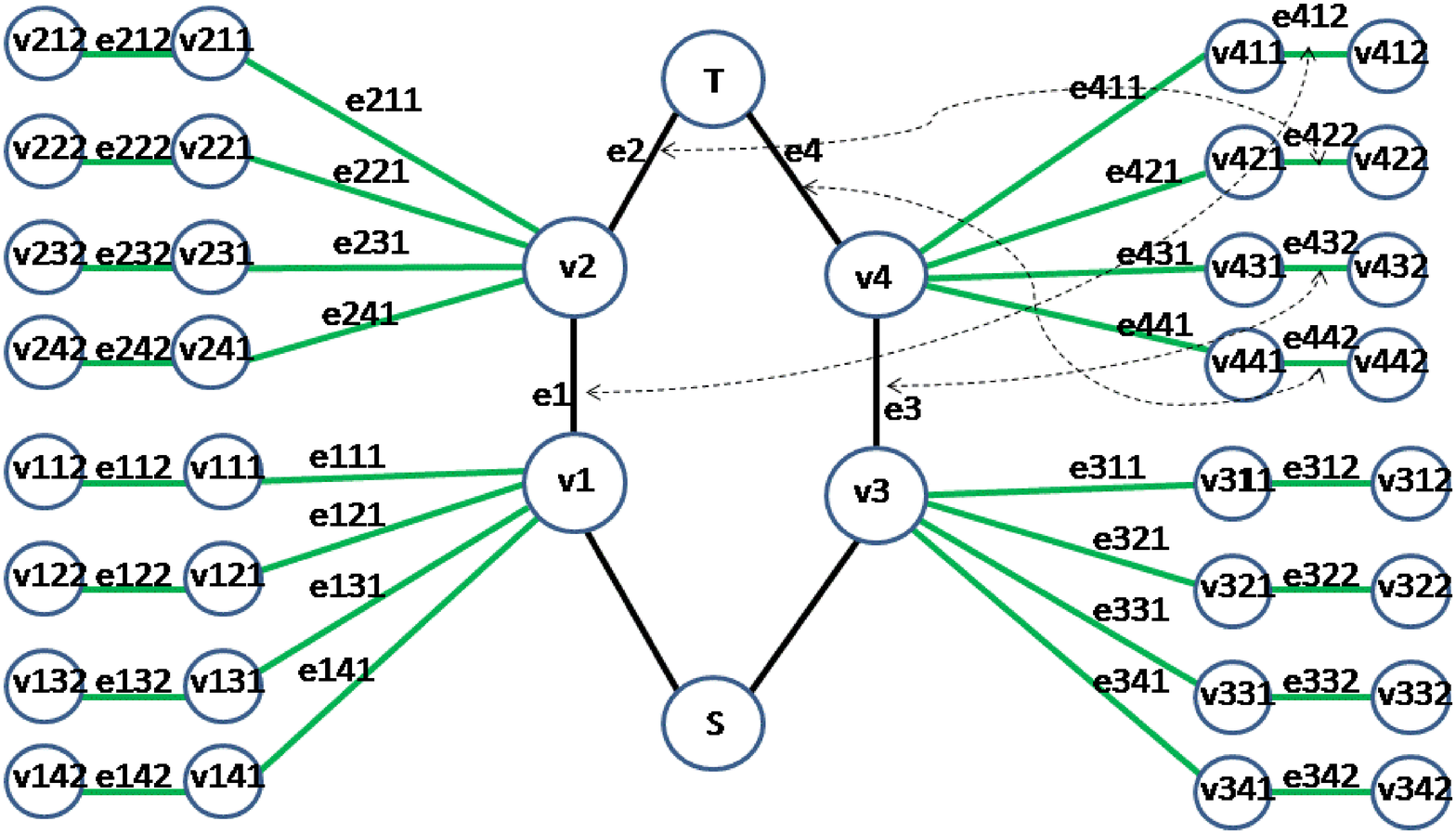}}
        \caption{Presented here graph G'(V',E') (CTP with dependencies) which is reduced from from graph G(V,E)(CTP with sensing). Each vertex $v'\in V'$ is attached to 4 paths which are correlated with each of the 4 preknown edges.  The agent in I' can "simulate" the sensing action in I, by moving in one of these paths. For instance, an agent that simulate a sensing action on $e_1$ from $v_3$ will perform $move(v_1,v_{131})$ and then $move(v_{131},v_1)$.
         The indexes of the $V'_{i,j,k},E'_{i,j,k}$ on the paths attached to vertexes in $\hat{V}$ indicates, i: the vertex that the path is attached to, j: the edge that is correlated with the path, and k: the id of the vertex/edge, i.e. the first vertex/edge in the path or the second.
        The dotted arrows indicate the dependency between edges.}
        \label{fig:ReductionToSenisng}
\end{figure} 

Let $M_1=(S_1,A_1,Z_1,Tr_1,O_1,R_1)$ be a POMDP that modes I, where $S_1,A_1,Z_1,Tr_1,O_1,R_1$ are finite sets of states, actions, observations, transition functions, observation functions and reward functions respectively. Similarly, let $M_2=(S_2,A_2,Z_2,Tr_2,O_2,R_2)$ be a POMDP that models I' where $Z_2$ is the set of observations in I', $S_2$ is special subsets of the states set in I', and $A_2$ is a special meta-action set in I'(a set of series of action in I') which will be defined later.
$Tr_2,O_2,R_2$  are the transition functions, observation functions and reward functions in I'. Let $\pi_1^*$ be the optimal policy of I and $\pi_2^{'*}$ be the optimal policy of I'. In order to prove theorem \ref{path-dep}, it is suffice to show that $Exp(\pi_1^*)=Exp(\pi_2^{*})$. In the remainder of this proof we prove this property by showing that $M_1$ is equivalent to $M_2$ and that $M_2$ actually models I'.

%
%

We want to define the subset $S_2\subset S'$ that contains all states in which the agent is located in a "`copied"' vertex. Formally,
\begin{definition}\label{tildes}
Given $S'$, we define $S_2$ to be the subset of S' such that $s_i\in S_2$ if and only if $s_i\in S'$ and $Loc(s_i)\in \hat{V}$. Meaning $S_2=\setm{s_i}{s_i\in S',Loc(s_i)\in \hat{V}}$.
\end{definition}
\begin{lemma} 
Let $\tilde{V}$ be the location space of $S_2$, (i.e. $\tilde{V}=\setm{v}{v=Loc(s_2),s_2\in S_2}$) then $\tilde{V}=\hat{V}$. 
\end{lemma}

\begin{proof}

 $=>$ $\tilde{V}\subset \hat{V}$. According to $\tilde{V}$, for every $s_i\in S'$ $Loc(s_i)\in \tilde{V}$ if $Loc(s_i)\in \hat{V}$.
\newline
 $<=$ $\hat{V}\subset \tilde{V}$.  V' contains all possible locations that agent can be in G', where $\tilde{V}$ is subset of V' which contains all possible location in $\hat{G}$. Thus every element in $v_i\hat{V}$ is in $\tilde{V}$
\end{proof}

\begin{cor} \label{coro}
$\hat{V}$ is the location space of $S_2$.
\end{cor}

Now, we want to show that $S_2$ and $S_1$ are equivalent in a sense that there exist a one-to-one correspondence between $S_2$ and $S_1$ . 
In order to show that we need to make the following definitions and statements,

\begin{defn}
We define EStatusSet(S) be the set of $Estatus(s_i)$ of all elements $s_i \in S$, meaning $EstatusSet(S)=\setm{Estatus(s_i)}{s_i\in S}$.
\end{defn}
In fact, $EstatusSet(S)$ is the set of all the possible status vectors of E, as such $EstatusSet(S)=\{Open,Blocked\}^{|E|}$. 


\begin{lemma}\label{Vcor}
There exist a one-to-one correspondence between V and $\hat{V}$
\end{lemma} 
\begin{proof}
Since $g_v$ is a bijection, the exist a one to one correspondence between $v_i\in V$ and $g_v(v_i)\in \hat{V}$
\end{proof}

%
%

%

\begin{lemma}\label{Ecor}
There exists a one-to-one correspondence between $EstatusSet(S_1)$ and $EstatusSet(S_2)$. 
\end{lemma}    
\begin{proof}  
There exist a one to one correspondence between $EstatusSet(S_1)$ and $EstatusSet(S')$ since for all $s_1\in S_1, s'\in S'$, each element $Estatus(s_1)\in EstatusSet(S_1)$ can be mapped into a different element $Estatus(s')\in EstatusSet(S')$. This is due to the following facts:
\begin{enumerate}
\item[1] (injective)According to definition of W', $\forall{i}\;x'_{g_e(e_i)}=0 \Leftrightarrow x_{e_i}=0$, in other words every edge $e_i\in E$ has equal status as its copied edge $g_e(e_i)$ and thus there exist a one to one correspondence between each set of edge status $Estatus(s_1)$ and each set of edge status $ste(\hat{E},s_1)$ (In fact for each $s_1\in S_1$ $Estatus(s_1)=ste(\hat{E},s_1)$).  Since for every $s'\in S'$, $ste(\hat{E},s')$ is a subset of $EstatusSet(s')$, there exist a one to one mapping between $EstatusSet(S_1)$ and $EstatusSet(S')$
\item[2] (surjective)  The status of edges $ste(E'-\hat{E},s')$ is completely determined and unique, given edges status of edges $ste(\hat{E},s')$, \\ i.e. $ste(\hat{E},s')=\{st(g_e(e_1),s'),st(g_e(e_2),s'),..,st(g_e(e_n),s')\}$. In particular, there exist exactly one element in $EstatusSet(s')\in EstatusSet(S')$ with a given edges status $ste(\hat{E},s')$, since each variable associated with edges in $E'_{ij2}$ depends completely on variables associated with edges in $\hat{E}$ ($\forall{i,j}\;x'_{e'_{ij2}}=0\Leftrightarrow x'_{g_e{e_j}}=0$) and the status of all edges in $E'_{ij1}$ are predetermined to be open ($\forall{i,j}\; P(x'_{e'_{ij1}}=0)=1$). 
 \end{enumerate}
We are left to show that there exist a one to one correspondence between $EstatusSet(S')$ and $EstatusSet(S_2)$. 
Since the location of the agent is independent to the edges status, we can represent $S_2$ as a cartezian product $S_2=\hat{V}\times EstatusSet(S_2)$ (According to corollary \ref{coro} $\hat{V}$ is the location space of $S_2$) but 
according to definition $\ref{tildes}$ we can represent $S_2$ as $S_2=\hat{V}\times EstatusSet(S')$, hence $EstatusSet(S')=EstatusSet(S_2)$.
Therefore, there exists a one-to-one correspondence between $EstatusSet(S_1)$ and $EstatusSet(S_2)$.
\end{proof}

\begin{lemma}
There exists one-to-one correspondence between $S_1$ and $S_2$. 
\end{lemma}
\begin{proof}
According to lemma \ref{Vcor}, there exist a one-to-one correspondence between V and $\hat{V}$. According to lemma \ref{Ecor}, there exists a one-to-one correspondence between $EstatusSet(S_1)$ and $EstatusSet(S_2)$. Since $S=V\times EstatusSet(S_1)$ and  $S_2=\tilde{V}\times EstatusSet(S_2)$, there exists one-to-one correspondence between $S_1$ and $S_2$.
\end{proof}


\begin{defn}
Let $v_i,v_j\in V$, $v'_{ij1}\in V_{ij1}$, $\hat{v}_i,\hat{v}_j\in \hat{V}$ and $e_j\in E$.
We define $a_2$ in I' to be the equivalent meta-action to action $a_1\in A_1$($a_1\sim a_2$) if and only if:\newline
$
a_2 = \left\{ 
\begin{array}{ll}
move (\hat{v}_i,\hat{v}_j) & \quad if\; a \in move(v_i,v_j)\\
move (\hat{v}_i,v'_{ij1}), move(v'_{ij1},\hat{v}_i) & \quad if\; a\in sense(v_i,e_j)
\end{array}
\right. 
$
\end{defn}

\begin{defn}
We define $A_2$ to be the set of all equivalent actions of actions in $A_1$. Meaning $A_2=\setm{a2}{a2\sim a1, a1\in A1}$
\end{defn}
\begin{defn}
We define the set $\tilde{st}_{e_i}$ to be the following:\newline
$
\tilde{st}_{e_i} = \left\{ 
\begin{array}{ll}
\{O_{g_e({e}_i)},O_{e_{1i2}},O_{e_{2i2}},...,O_{e_{ni2}}\}& \quad if\; st_{e_i}=O_{e_i}\\
\{B_{g_e({e}_i)},B_{e_{1i2}},B_{e_{2i2}},...,B_{e_{ni2}}\}& \quad if\; st_{e_i}=B_{e_i}\\
\end{array}
\right.
$
\end{defn}

\begin{defn}\label{observationDef}
Let $Z_2$ be a set of observations in I' and $Z_1$ be a set of observations in I. We define $z_2 \in Z_2$ is the equivalent observation of $z_1\in Z_1$(denoted by $z_1\sim z_2$) if and only if:\newline
\begin{equation}
z_2=\{\tilde{st}_{e_{i_1}}\cap \tilde{st}_{e_{i_2}}\cap...\cap\tilde{st}_{e_{i_m}}\} \quad if\; z_1=\{st_{e_{i_1}},st_{e_{i_2}},...,st_{e_{i_m}}\}
\end{equation}
\end{defn}

\begin{lemma}\label{lemma 2}
The cost of action in $A_1$ is equal to the cost of the equivalent meta-action in $A_2$. 
\end{lemma}

\textbf{Proof.} 
\begin {enumerate}
\item[1] $C(move(v_i,v_j))=C(move (\hat{v}_i,\hat{v}_j))$ (by definition of the weight function).
\item[2] $C(sense(v_i,e_j))=SC(e_j)$, and $C(move(\hat{v}_i,v'_{ij1}))=C(move(v'_{ij1},\hat{v}_i))=\frac{SC(e'_j)}{2}$ . Therefore, $C(sense(v_i,e_j))=C(move(\hat{v}_i,v'_{ij1}))+C(move(_{ij1},\hat{v}_i)))$.
\end {enumerate}

\begin{lemma}\label{lemma 1}
Given $s_1\in S_1,a_1\in A_1,z_1\in Z_1$, and $s_2\in S_2,a_2\in A_2,z_2\in Z_2$ such that $s_1\sim s_2, a_1\sim a_2, z_1\sim z_2$, then 
$O_1(s_1,a_1,z_1)=O_2(s_2,a_2,z_2)$.
\end{lemma}
\textbf{Proof.}
\begin{enumerate}
\item[1] In case that $a_1=move(v_j,v_i)$. Let $e_{i_1},e_{i_2},...,e_{i_n}\in E_{v_i}$(the edges incident to $v_i$) and let $\tilde{E}_{v_i}=\{\tilde{st}_{e_{i_1}}\cap \tilde{st}_{e_{i_2}}\cap...\cap\tilde{st}_{e_{i_m}}\}$. By definition of CTP, if $a_1=move(v_j,v_i)$ then the agent observes $z_1=st_{e_{i_1}},st_{e_{i_2}},...,st_{e_{i_n}}$ (the pre-known edges incident to $v_i$ which are revealed by the action). Therefore  $O_1(s_1,a_1,z_1)=1$ if $z_1=ste(E_{v_i},s_1)$ and $O_1(s_1,a_1,z_1)=0$ otherwise. The equivalent action of $a_1$ is $a_2=move (\hat{v}_i,\hat{v}_j)$, hence, taking $a_2$, the agent directly observes $st_{\hat{e}_{i_1}},st_{\hat{e}_{i_2}},...,st_{\hat{e}_{i_n}}$, but in addition, according to definition \ref{bayesianDef} $x'_{e'_{1i2}}=0,x'_{e'_{2i2}}=0,...,x'_{e'_{ni2}}=0\Leftrightarrow x'_{\hat{e_i}}=0$, hence 
the agent also indirectly observes edges $e_{1i2},e_{2i2},...,e_{ni2}$ . Thus, the agent's overall observation is $z_2=\tilde{E}_{v_i}$.
Since $z_1\sim z_2$, by definition \ref{observationDef} if $z_1=ste(E_{v_i},s_1)$ then $z_2=\tilde{E}_{v_i}$. Hence $O_2(s_2,a_2,z_2)=1$ if $z_1=ste(E_{v_i},s_1)$ and $O_2(s_2,a_2,z_2)=0$ otherwise. Thus, in this case $O_1(s_1,a_1,z_1)=O_2(s_2,a_2,z_2)$
\item[2] In case that $a_1=sense(e_j)$ the agent observes $z_1=st_{e_j}$ therefore $O_1(s_1,a_1,z_1)=1$ if $z_1=st_{e_j}$ and $O_1(s_1,a_1,z_1)=0$ otherwise.
Since $a1\sim a_2$, $a_2=(move(\hat{v}_i,v'_{ij1}),move(v'_{ij1},\hat{v}_i))$ where agent observes $st_{\hat{e}_j}$ directly
and observes $st_{e_{1j2}},st_{e_{2j2}},...,st_{e_{nj2}}$ indirectly(the same cause as in case 1). Thus, the agent's overall observation is $z_2=\tilde{st}_{e_j}$.
Similarly to case 1, since $z_1\sim z_2$, by definition \ref{observationDef} if $z_1=st_{e_j}$ then $z_2=tilde{E}_{v_i}$. Hence $O_2(s_2,a_2,z_2)=1$ if $z_1=st_{e_j}$ and $O_2(s_2,a_2,z_2)=0$ otherwise. Thus, in this case $O_1(s_1,a_1,z_1)=O_2(s_2,a_2,z_2)$ as well.
\end{enumerate}

\begin{lemma}\label{lemma 1}
Given $s_{1a}\in S_1,s_{1b}\in S_1,a_1\in A_1,z_1\in Z_1$, and $s_{2a}\in S_2,s_{2b}\in S_2,a_2\in A_2,z_2\in Z_2$ such that action $a_1$ is taken in $s_{1a}$ and meta-action $a_2$ is taken in $s_{2a}$, then if $s_{1a}\sim s_{2a},s_{1b}\sim s_{2b}, a_1\sim a_2$ then $Tr_1(s_{1a},a_1,s_{1b})=Tr_2(s_{2a},a_2,s_{2b})$.
\end{lemma}
\begin{proof}
WLOG, let $Loc(s_{1a})=v_i$, and $Loc(s_{1b})=v_j$. 
Since $s_{1a}\sim s_{2a}, s_{1b}\sim s_{2b}$ we get $Loc(s_{1a})=\hat{v}_i$ and $Loc(s_{1b})=\hat{v}_j$.
\begin{itemize}
\item WLOG, in case that $a_1=move(v_i,v_j)$. If $(v_1,v_2)\in E$ we get $Tr_1(s_{1b},a,s_{1a})=1$ otherwise $Tr_1(s_{1b},a,s_{1a})=0$. Furthermore, if $(v_i,v_j)\in E$ then $(g_v(v_i),g_v(v_j))\in \hat{E}$ which incurs  $Tr_2(s_{2b},a,s_{2a})=1$, otherwise $Tr_2(s_{2b},a,s_{2a})=0$. Thus, in case that $a_1=move(v_i,v_j)$ we get $Tr_1(s_{1b},a,s_{1a})=Tr_2(s_{2b},a,s_{1a})$.
\item WLOG, in case that $a_1=sense(e_j,v_i)$. Since the sense action does not change the location of the agent we get $s_{1a}=s_{1b}$. Since $a_1\sim a_2$ $a_2=move(\hat{v}_i,v_{ij1}),move(v_{ij1},\hat{v}_i)$. In this case $Loc(s_{2a})=Loc(s_{2b})$ since the agent return to it original location $\hat{v}_i$. This incurs $s_{2a}=s_{2b}$ and thus $Tr_1(s_{1b},a,s_{1a})=Tr_2(s_{2b},a,s_{2a})=1$.
\end{itemize}
\end{proof}

\begin{lemma}\label{lemma 1}
Given states $s_1\in S_1,s_2\in S_2,a_1\in A_1,a_2\in A_2$ then if $s_1\sim s_2$ and $a_1\sim a_2$ then $R_1(s_1,a_1)=R_2(s_2,a_2)$.
\end{lemma}

\begin{proof}
WLOG, let $Loc(s_1)=v_1$. Given that $s_1\sim s_2$ then $Loc(s_2)=\hat{v_1}$. 
\begin{itemize}
\item In case that $a_1=move(v_i,v_j)$, $R_1(s_1,a_1)=C(move(v_i,v_j))$. If $a_1\sim a_2$ then $a_2=move(\hat{v}_i,\hat{v}_j)$.
Since $C(move(v_i,v_j))=C(move(\hat{v}_i,\hat{v}_j))$ we get $R_1(s_1,a_1)=R_2(s_2,a_2)$.
\item In case that $a_1=sense(v_i,e_j)$, $R_1(s_1,a_1)=C(sense(v_i,e_j))$. If $a_1\sim a_2$ then $a_2=move(v_i,v_{ij1}),move(v_{ij1},v_i)$.
Since $C(sense(v_i,e_j))=C(move(\hat{v}_i,v'_{ij1}))+C(move(v'_{ij1},\hat{v}_i)))$ we get $R_1(s_1,a_1)=R_2(s_2,a_2)$.
\end{itemize}
\end{proof}

\begin{lemma}
$M_1$ is equivalent to $M_2$.
\end{lemma}
\begin{proof}
We have shown that there exist a one to one correspondence between $S_1$ and $S_2$. 
By defining the set $A_2$ which consist of equivalent action in $A_1$, and by defining the set $Z_2$ which consist of equivalent observations in $Z_1$,
we have shown that functions $Tr_1=Tr_2$, $O_1=O_2$, and $R_1=R_2$ when generated on equivalent set of states, observation and actions.
\end{proof}

\begin{lemma}
$M_2$ models the problem of I'. 
\end{lemma}
\begin{proof}
Here we show that although $M_2$ models a subproblem of I' ($M_2$ is defined on subsets of states, actions of I'), it actually models the exact problem of I'. 
For every state $s\in S'-S_2$, $Loc(s)=v_{ij1}$ where $v_{ij1}\in V_{ij1}$. An agent located in $v_{ij1}$ can only move to $v_i$. In addition, in order that agent would be located in $v_{ij1}$ it has to move from $v_i$. Thus $A_2$ replace the two move actions in to one meta-action and thus we can reduce the state set of S' into the subset $S_2$. Therefore, $M_2$ models I'.
\end{proof}

\label{Dependency Graph}
\newpage
\section{CTP-Forward-Arcs} 
\begin{defn}
Let $G=(V,E)$ be a disjoint paths graph of CTP-PATH-DEP and $W=(X,Y)$ be its associated Bayesian network. Let $x_{e_{ij}},x_{e_{ik}}\in X$ be the associated node of edges $e_{ij}$ and $e_{ik}$ (note that the edges are in the same path $i$ in G). Then the arc $\form{x_{e_{i,j}},x_{e_{i,k}}}\in Y$ is \textit{Forward-Arc} if $j<k$, i.e. if $e_{ij}$ is closer to s than $e_{ik}$.
\end{defn}

\begin{defn}\label{defForward}
 \textit{CTP-Forward-Dependency}(CTP-FOR-DEP) is a special case of CTP-PATH-DEP such that all the arcs in W are Forward-Arcs.
\end{defn}

\begin{theorem}\label{theoremForward}
CTP-FOR-DEP is solvable in polynomial time.
\end{theorem}

\textbf{Proof outline.}
CTP on disjoint paths graph with independent distribution over the edges(CTP-PATH-IND) is shown to be solvable in polynomial time [Bnaya, Felner and Shimony, 2009]. We show that we can transform  CTP-FOR-DEP into an instance of CTP-PATH-IND with new distribution over the edges such that the optimal policy of the new CTP-PATH-IND can be applied to CTP-FOR-DEP.

\begin{proof}
Let $I=(G,W,w,s,t)$ be an instance of CTP-FOR-DEP. We construct a new instance $I'=(G,W',w,s,t)$ of CTP-PATH-IND by constructing a new Bayesian network W'(X',Y') of I' such that 
\begin{itemize}
\item $Y'=\left\{\right\}$. In other words W' is ``arc free'' where each node is an independent component in the BN. 
\item $\forall x'_{e_{ik}}\in X'$ $P(x'_{e_{ik}}=1)=P(x_{e_{ik}}=1|x_{e_{i1}}=0,x_{e_{i2}}=0,...,x_{e_{i(k-1)}}=0)$.
\end{itemize}
Let $M=(B,A,Tr,R)$ be a belief state MDP of I, where B is the set of belief states , A is the set of actions, Tr is a set of  transition probabilities, R is the reward function. 
We construct a new belief state MDP $M'=(B',A,Tr',R')$ of I' where  
B' is the set of belief states, A is a set of actions which is common to the set of action in I (since it refers to the same graph G), Tr' is a set of  transition probabilities, and R' is the reward function .
\begin{defn}
Let $f:B\rightarrow B'$ be a function defined as follows:
Let $b\in B$ and $b'\in B'$ such that $f(b)=b'$ then $\form{b}=\form{b'}$.
\end{defn}
Notice that $f(b)$ is well defined since there is a one to one mapping from $\form{b}$ to $b$ and from $\form{b'}$ to $b'$.

\begin{lemma}\label{lemma:TrEq}
Let $b,\hat{b}$ be reachable belief states in B and let $a\in A$ be an action. Then $Tr(\hat{b}|a,b)=Tr(f(\hat{b})|a,f(b))$ 
\end{lemma}
\begin{proof}
Let $a=Move(e)$ where $e=\form{v_{i(k-1)},v_{i(k)}}$ and let $e_f=for(e)$
\begin{enumerate}
\item[1] In case that $Tr(\hat{b}|a,b)=0$ (i.e. action $a$ not performable in $b$), then $Tr'(f(\hat{b})|a,f(b))=0$.
If $Tr(\hat{b}|a,b)=0$ then one of the following cases must satisfied:
\begin{itemize}
	\item Edge $e$ is not adjacent to the location of the agent in $b$, i.e. $e\notin Inc(LocB(b))$. If $e\notin Inc(LocB(b))$ then $e\notin Inc(LocB(f(b))$ since $LocB(b)=LocB(f(b))$. Thus $Tr'(f(\hat{b'})|a,f(b))=0$.
	\item Edge $e$ is not adjacent to location of the agent in $\hat{b}$, i.e. $e\notin Inc(LocB(\hat{b}))$. If $e\notin Inc(LocB(\hat{b}))$ then $e\notin Inc(LocB(f(\hat{b}))$, since $LocB(\hat{b})=LocB(f(\hat{b}))$. Thus $Tr'(f(\hat{b})|a,f(b))=0$.
	\item Edge $e$ is blocked in belief state $b$, i.e. $stb(b,e)=B_e$. If $stb(b,e)=B_e$ then $stb(f(\hat{b}),e)=B_e$ since $stb(b,e)=stb(f(\hat{b}),e)$. Thus $Tr'(f(\hat{b})|a,f(b))=0$.
	\item There exist an edge $e'\neq For(e)$ such that $stb(b,e')\neq stb(\hat{b},e')$. Since $stb(b,e)=stb(f(b),e)$ and  $stb(f(\hat{b}),e)=stb(\hat{b},e)$
	, there exist an edge $e'\neq For(e)$ Since is blocked in belief state $b$, i.e. $stb(b,e)=B_e$. If $stb(b,e)=B_e$ then $stb(\hat{b},e)=B_e$, since $stb(b,e)=stb(f(\hat{b}),e)$. Thus $Tr'(f(\hat{b})|a,f(b))=0$.
	\end{itemize}
	\item[2] In case that $Tr(\hat{b}|a,b)>0$ (i.e action $a$ is performable in $b$) then edge $e$ has to be open and one of the following cases must satisfied:
	\begin{itemize}
		\item Edge $e_f$ is Open in b (i.e $stb(b,e_f)=O_{e_f}$). If $stb(b,e_f)=O_{e_f}$ then the status of all edges in $b$ must be the same as in $\hat{b}$, i.e. ($\forall e\in E stb(e,b)=stb(e,\hat{b})$) since the agent does not sense any unknown edge when performing $a$ and hence $Tr(\hat{b}|a,b)=1$. If $stb(b,e_f)=O_{e_f}$ then $stb(f(b),e_f)=O_{e_f}$ and the status of all edges in $b$ must be the same as in $\hat{b}$ from the same reasons as before. Hence $Tr'(f(\hat{b})|a,f(b))=1$,  and we have $Tr(\hat{b}|a,b)=Tr'(f(\hat{b})|a,f(b))$.
		\item Edge $e_f$ is Unknown in b (i.e $stb(b,e_f)=U_{e_f}$). 
		Since W is the belief network of CTP-FOR-DEP and b is reachable from $b_0$, the status of all edge $e_{i1},e_{i2},...,e_{ik}$ have to be Open (In order to reach $v_{ik}$ all edges in path i from s to $v_{ik}$ must be traversable).
		Thus, 
		\[
		Tr(\hat{b}|a,b)=P(x_{e_{f}}=1|x_{e_{i1}}=0,x_{e_{i2}}=0,...,x_{e_{ik}}=0)
	 	\]
		In addition $stb(f(b),e_f)=U_{e_f}$ (since $stb(f(b),e_f)=stb(b,e_f)$).
		There is no dependencies in W' (i.e. $Y'=\{\}$). Therefore,
		\[
		Tr'(\hat{b}|a,b))=P(x'_{e_{f}}=1).
		\]
		 However, by definition of X' we have, 
		 \[
		 P(x'_{e_{f}}=1)=P(x_{e_{f}}=1|x_{e_{i1}}=0,x_{e_{i2}}=0,...,x_{e_{ik}}=0)
		 \]
		 Hence, 
		 \[
		  Tr(\hat{b}|a,b)=Tr'(f(\hat{b})|a,f(b))
		\]
		\end{itemize}
		\end{enumerate}
		\end{proof}

\begin{lemma}\label{lemma:REq}
Let $b,\hat{b}$ be reachable belief states in B and let $a\in A$ be an action. Then $R(b,a,\hat{b})=R'(f(b),a,f(\hat{b}))$. 
\end{lemma}
\begin{proof}
From definition \ref{eq:rewardFunction} it follows that: \newline 
$R(b,a,\hat{b})=R'(f(b),a,f(\hat{b}))$ if and only if $Tr(\hat{b}|a,b)=Tr'(f(\hat{b})|a,f(b))$.  \newline
But we proved in lemma \ref{lemma:TrEq} that $Tr(\hat{b}|a,b)=Tr'(f(\hat{b})|a,f(b))$\newline
Thus, $R(b,a,\hat{b})=R'(f(b),a,f(\hat{b}))$.
\end{proof}

\begin{defn}\label{def:reachable}
We define the predicate $REACHABLE(b_n,b_0)$ to be true if and only if $b_n$ is reachable from $b_0$ in belief-MDP M. i.e. there exist $b_1,...,b_{n-1}$ such that $\prod \limits_{0\leq i\leq n-1,a\in A}Tr(b_i,a,b_{i+1})>0$
\end{defn}

\begin{lemma}
Let $b\in B$. Then $REACHABLE(b,b_0)$ is true if and only if $REACHABLE(f(b),f(b_0))$ is true.  
\end{lemma}
\begin{proof}
Follows from definition \ref{def:reachable} and lemma \ref{lemma:TrEq}.
\end{proof}

\begin{defn}
Define the set $B_{reach}\subset B$ to be the set of all belief states $b\in B$ that satisfy $REACHABLE(b,b_0)$. Namely,
\[
B_{reach}=\setm{b}{b\in B,REACHABLE(b,b_0)}
\]
\end{defn}
Next, we define an analogue set for B',
\begin{defn}
Define the set $B'_{reach}\subset B$ to be the set of all belief states $f(b)\in B'$ , for $b\in B$, that satisfy $REACHABLE(f(b),f(b_0))$. Namely,
\[
B'_{reach}=\setm{f(b)}{f(b)\in B',b\in B, REACHABLE(f(b),f(b_0))}
\]
\end{defn}

Let $M_r=(B_{reach},A,Tr,R)$ be a belief MDP over belief state $B_{reach}$ and let $M'_r=(B'_{reach},A,Tr',R')$ be a belief MDP over belief state $B'_{reach}$ 

\begin{lemma}\label{lemma:isomorphism}
$M_r$ and $M'_r$ are isomorphism.
\end{lemma}
\begin{proof}
Follows from definition \ref{def:reachable} that $F$ is a bijection over $B_{reach}$ and $B'_{reach}$. 
In addition, F preserves the function Tr , Tr' lemma \ref{lemma:TrEq} as well as R,R' lemma \ref{lemma:REq} .
\end{proof}

\begin{cor}\label{eqPolicies}
Let $\pi^*$ be the optimal policy of I, and $\pi'^*$ be the optimal policy of I'.
Then for every reachable belief state $b\in B_{reach}$ we have $\pi^*(b)=\pi'^*(f(b))$.
\end{cor}
\begin{proof}
Since $M_r$ and $M'_r$ are isomorphism, the problems are equivalent and their optimal solutions are equivalent.
\end{proof}

Therefore, we can transform any instance of CTP-FOR-DEP into an instance of CTP-PATH-IND, apply the algorithm which solves CTP-PATH-IND in polynomial time, and equivalent optimal solution is guaranteed (corollary \ref{eqPolicies}). 

Now, we show that determining the probability $P(x_{e_{i,k}}=1|x_{e_{i,1}}=0,x_{e_{i,2}}=0,...,x_{e_{i,k-1}}=0)$ for all nodes can be computed in polynomial time.
We use the Bayesian theorem to get:
\begin{eqnarray}\label{bayes1}
& &P(x_{e_{i,k}}=1|x_{e_{i,1}}=0,x_{e_{i,2}}=0,...,x_{e_{i,k-1}}=0)=\nonumber\\
& &P(x_{e_{i,k}}=1,x_{e_{i,1}}=0,x_{e_{i,2}}=0,...,x_{e_{i,k}}=0)\cdot P(x_{e_{i,k}}=1|x_{e_{i,1}}=0,x_{e_{i,2}}=0,...,x_{e_{i,k-1}}=0)\nonumber \\
\end{eqnarray}
We use the chain rule to get:
\begin{eqnarray}\label{bayes2}
& &P(x_{e_{i,k}}=1,x_{e_{i,1}}=0,x_{e_{i,2}}=0,...,x_{e_{i,k}}=0)=\nonumber\\
& &P(x_{e_{i,k}}=1|x_{e_{i,1}}=0,x_{e_{i,2}}=0,...,x_{e_{i,k-1}}=0)\cdot P(x_{e_{i,k}}=1,x_{e_{i,1}}=0,x_{e_{i,2}}=0,...,x_{e_{i,k-1}}=0)\nonumber\\
\end{eqnarray}
The variables in the Bayesian network are topologically ordered by their order in the path and hence each probability $P(x_{e_{ik}})$ can be iteratively computed given that its ancestors values have already been determined(using equations \ref{bayes1},\ref{bayes2}). Therefore, inferring the probability of each edge takes linear time and inferring the probability of all edges takes $O(|E|^2)$. Thus computing the optimal policy takes polynomial time.  
\end{proof}
\label{Forward Arcs}
\newpage
\section{CTP-PATH-DEP} 
\begin{definition}

\end{definition}
CTP-PATH-DEP is a special case of CTP-DEP where the associated Bayesian network has dependencies only between edges on the same path.
\begin{theorem}
CTP-PATH-DEP is NP-hard.
\end{theorem}

\textbf{Proof outline}
By reduction from 3-SAT to CTP-PATH-DEP. 
\begin{proof}
Let $L$ be a set boolean variables $l_1,...,l_n$. Let the 3CNF formula $F$ be a conjunction of the clauses $C_1, C_2,..., C_k$ where each clause $C_i$ is a disjunction of three literals ${l'_i}^1, {l'_i}^2, {l'_i}^3$ and for each literal ${l'_i}^j$ it holds that ${l'_i}^j\in L$ or $\neg{{l'_i}^j}\in L$. 
We construct the instance $I=(G,W,w,s,t)$ of CTP-PATH-DEP from F, such that F is satisfiable if and only if the expected cost of the optimal policy is greater than some given constant.
I is defined as follows:
$G=(V,E)$ is a graph consisting two disjoint paths $p_1,p_2$, where

\begin{enumerate}
\item $p_1=\left\langle e_Y, e_{d1},...,e_{d(k-1)},e_{c_1},...,e_{c_k}, e_{l_1},...,e_{l_n}, e_R \right\rangle$ (The edges are ordered from the edge incident to s to the edge incident to t). Edges $e_{c_1},...,e_{c_k}$ correspond to clauses $C_1,...,C_k$ respectively, and edges $e_{l_1},...,e_{l_n}$ correspond to variable $l_1,...,l_n$ respectively. The correspondence will be define later in this proof.
\item $p_2$ consist of a single edge $e_L$.
\end{enumerate}

w is the weight function over the edges, is defined by:
\begin{itemize}
\item $w(e_Y)=1$
\item $w(e_L)=(1+\frac{k}{2^{n+1}})$
\item $w(e)=0$ for all other edges.
\end{itemize}

$W=(X,Y)$ is a Bayesian network. 

\begin{defn}
For every edge $e\in E$ in path $p_1$, we define the variable $x_e\in X$ to be the variable corresponded to edge $e$ such that $x_e=0$ if and only if $e$ is Open.
\end{defn}
%


The set of node $X$ of $W$ is a union of the following sets:
\begin{itemize}
\item $X_Y=x_Y$. $X_Y$ is a set that contains the single variable $x_Y$.
\item $X_R=x_R$. $X_R$ is a set that contains the single variable $x_R$.
\item $X_l=\{x_{l_1},x_{l_2},...,x_{l_n}\}$. $X_l$ is a set that contains all nodes that correspond to variables $l_1,...,l_n$.
\item $X_c=\{x_{c_1},x_{c_2},...,x_{c_k}\}$. $X_c$ is a set that contains all nodes that correspond to variables $c_1,...,c_k$..
\item $X_d=\{x_{d_1},x_{d_2},...,x_{d_{k-1}}\}$. $X_c$ is a set that contains all nodes that correspond to variables $d_1,...,d_{k-1}$..
\end{itemize}  

Namely, $X=X_Y\cup X_R \cup X_l \cup X_c \cup X_d$.


The arcs in $Y$ are defined by the followng sets:  
\begin{itemize}
\item $Y_{Rl_i}=\{\form{x_R,x_{l_i}}\}$. An arc from node $x_R$ to node $x_{l_i}$.
\item $Y_{Rc_i}=\{\form{x_R,x_{c_i}}\}$. An arc from node $x_R$ to node $x_{c_i}$.
\item $Y_{Rd_i}=\{\form{x_R,x_{d_i}}\}$. An arc from node $x_R$ to node $x_{d_i}$.
 
\item $Y_{li}=\{\form{x_{l_i^1},x_{c_i}},\form{x_{l_i^2},x_{c_i}},\form{x_{l_i^3},x_{c_i}}\}$.
 A set of three arcs from each variable node $x_{l_i^j}$ (for $1\leq j \leq 3$) to clause node $x_{c_i}$ such that  $l_i^j$ is the variable corresponding to literal ${l'_i}^j$.  For instance $C_5=\{{l'_5}^1\vee \neg {l'_5}^2 \vee \neg {l'_5}^3\}$ then $l_5^1={l'_5}^1,{l'_5}^2=\neg {l'_5}^2$, and $l_5^3=\neg {l'_5}^3$.
\item $\forall_{1\leq i\leq k}$ $(x_{ci}, x_{di})\in Y$ - Arc from each node $x_{ci}$ to a corresponding node $x_{di}$
\item $\forall_{1\leq i\leq k}$ $(x_{di}, x_{Y})\in Y$ - Arc from each node $x_{di}$ to node $x_Y$
\end{itemize}

The condition probabilities of W are as follows:

\begin{enumerate}
\item[1] $P(x_{e_R=0})=0.5$ ($x_R$ is an independent variable).
\item[2] For every variable node $x_i$ it holds that $P(x_i=0|x_R=0)=1$, i.e. if $x_R=0$ then path $p_1$ is always open with probability 1. 
\item[3] Given $x_R=1$, W is specified as follows:
\begin{enumerate}
\item $x_{ci}=0)\Leftrightarrow \bigwedge_{j=1}^{3}{x_{l_{ij}}=0}$
\item $x_{d1}=0\Leftrightarrow x_{c_1}=0$ 
\item $\forall i>1\; x_{d(i+1)}=0\Leftrightarrow x_{c_i}=0 , x_{di}=0$ 
\item $x_Y=1 \Leftrightarrow \bigwedge_{i}{x_{di}=0}$ 
\end{enumerate}
\end{enumerate}

The reduction maps each variable of F to a variable of W such that,

\begin{itemize}
\item Each boolean SAT variable $l_i$ is mapped to a binary variable in the Bayes network $x_{li}$, such that $x_{li}=0$ if and only if $l_i=T$ 
\item Each clause $C_i$ is mapped to binary variable $x_{ci}$  such that $x_{ci}=0$ if and only if $C_i=T$.
\end{itemize}

\begin{figure}[h]
        \centering
        \centerline{\includegraphics[width=6.5in,height=4.0in]{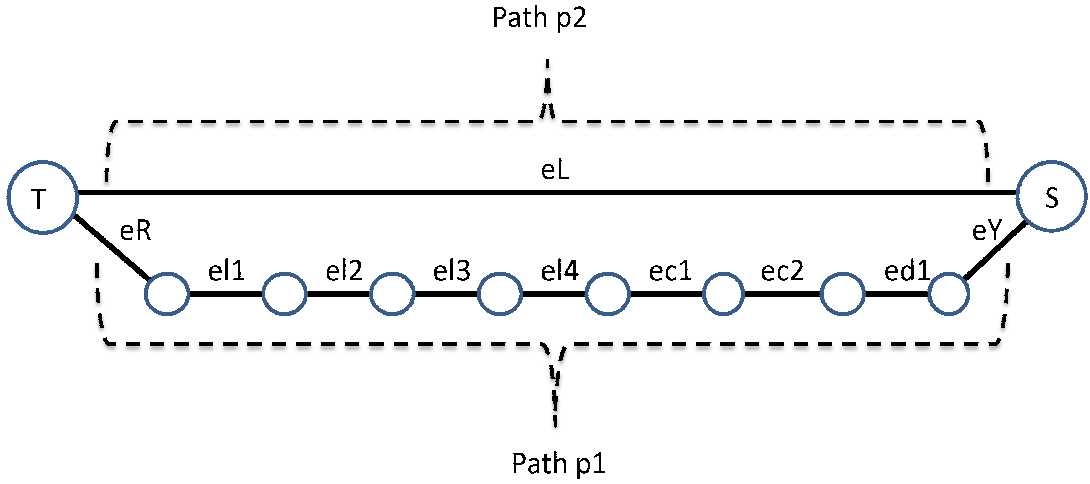}}
        \caption{Graph G. An instance I is constructed based on reduction from formula $F=C1\wedge C2$ where $C1=l1\vee l2\vee l3, C2=l2\vee l3\vee l4$. In this case path $p_1=\left\langle e_Y, e_{d1},e_{c_1},e_{c_2}, e_{l1},e_{l2},e_{l3},e_{l4},e_R \right\rangle$, and path $p_2=\left\langle e_L \right\rangle$}
        \label{fig:graphG}
\end{figure}

\begin{figure}[h]
        \centering
        \centerline{\includegraphics[width=6.5in,height=4.0in]{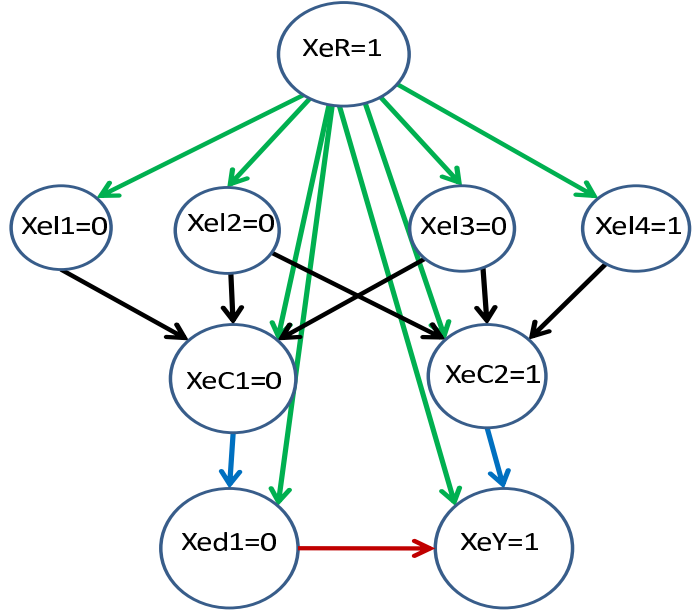}}
        \caption{Bayes network W of instance I based on reduction from formula F (which presented in figure \ref{fig:graphG}) in which its instantiations are: $l1=T,\;l2=T,\;l3=T,\;l4=F\Rightarrow C1=T,\;C2=F$. In case that $x_R=1$(presented in the figure), variables in W would be as follows: $x_{l1}=0,x_{l2}=0,x_{l3}=0,x_{l4}=1,x_{c_1}=0,x_{c_2}=1,x_{d1}=0,x_{Y}=1$. (otherwise if $x_R=0$ all variables associated with edges in $p_1$ would have been 0) }
        \label{fig:BayesNet.eps}
\end{figure}

\begin{lemma}
Given $x_R=1$, then F is satisfiable $\Leftrightarrow x_Y=0$.
\end{lemma}

\begin{proof}
If $x_R=1$ then F is satisfiable $\Leftrightarrow C_1=T,C_2=T,...,C_n=T\Leftrightarrow x_{c_1}=0,x_{c_2}=0,...,x_{c_n}=0$
in addition, $x_{c_1}=0\Leftrightarrow x_{d1}=0$ and $\forall i>1\; x_{d(i+1)}=0\Leftrightarrow x_{c_i}=0 , x_{di}=0$. 
Thus, $x_{c_1}=0,x_{c_2}=0,...,x_{c_n}=0 \Leftrightarrow x_{d_1}=0,x_{d_2}=0,...,x_{d_n}=0 \Leftrightarrow x_Y=0$
\end{proof}

For simplicity $w(e_Y)$ is denoted by CY and $w(e_L)$ is denoted by CL. Note that $CL>CY$.

The construction of the reduction is computable in polynomial time since the graph G contains $O(n)$ vertices, $O(n)$ edges and the Bayes network W contains $O(n)$ nodes and $O(n)$ arcs. In addition, function $g$, which maps each variable in $F$ to variable in $W$, is computable in polynomial time as well.

The optimal policy is committing in a sense that after the agent chooses a path, it keeps following this path until reaching t, unless agent hits a blocked edge. This is caused due to the fact that if agent chooses to traverse $p_1$ first, after traversing the first edge $e_Y$, it is optimal to keep following $p_1$ toward $t$ since the rest of the edges in $p_1$ are 0, and thus if $p_1$ is traversable, no extra travel cost is paid. On the other hand, if $p_1$ is not traversable then the agent pays extra $CY$ regardless of how many edges did it traversed in $p_1$. 
Therefore the decision problem of the optimal policy here is simply whether to choose $p_1$ as a first path to try or $p_2$.

\begin{notation}
Let $\pi_{12}$ denote a committing policy that chooses $p_1$ as a first path to try, and $\pi_{21}$ denote a committing policy that chooses $p_2$ as a first path to try.
\end{notation}

\begin{lemma}
Let C be a constant, such that $1+\frac{k}{2^{n+2}}< C <1+\frac{k}{2^{n+1}}$, where k is the number of models in F and n is the number of boolean-SAT-variables in F. Let $\pi^*$ be the optimal policy of I. F is satisfiable if and only if $Exp(\pi^*)> C$
\end{lemma}
\begin{proof}
 $=>$ Suppose that F is satisfiable. The probability that $e_Y$ is open is 
\begin{equation}
P(x_Y=0)=P(x_Y=0|x_R=0)P(x_R=0)+P(x_Y=0|x_R=1) P(x_R=1)\nonumber\\
\end{equation}
	
by construction of W: 
\begin{eqnarray}
& P(x_Y=0|x_R=0)=1 \nonumber \\
& P(x_R=0)=P(x_R=1)=0.5 \nonumber \\
\end{eqnarray}
The probability $P(F=True)=\frac{k}{2^n}$ since there are k sets of literals of F such that its instantiation gives F=true, and the domain size is the number of all possible instantiations to $l_1,...,l_n$, which equals $2^n$. Thus,
\begin{equation}\label{0}
  P(x_Y=0|x_R=1)=\frac{k}{2^n}  
\end{equation} 

\begin{equation}\label{1}
  => P(x_Y=0)=0.5+0.5\cdot \frac{k}{2^n}.
\end{equation} 
Let PY denote the probability $P(x_Y=0)$.

Now, we want to calculate the probability that path $p_1$ is open given that $e_y$ is open. 

\begin{eqnarray}\label{2}
& &P(p_1\; open|x_Y=0) =\frac {P(p_1\; open,x_Y=0)}{P(x_Y=0)}= \nonumber\\
& &\frac {P(p_1\; open,x_Y=0|x_R=0) P(x_R=0)+P(p_1\; open,x_Y=0|x_R=1)P(x_R=1)}{P(x_Y=0)} \nonumber\\	
\end{eqnarray}
According to W:
\begin{equation}\label{3}
 P(p_1\; open,x_Y=0|x_R=0)=1 
\end{equation}
if $x_R=1$ then $p_1$ is blocked
\begin{equation}\label{4}
 P(p_1\; open,x_Y=0|x_R=1)=0 
 \end{equation}
 
Setting equations \ref{1},\ref{3},\ref{4} in equation \ref{2} gives:

\begin{equation}
P(p_1\; open|x_Y=0) =  \frac{0.5}{0.5+0.5\cdot \frac {k}{2^n}} = \frac{1}{1+\frac{k}{2^n}} = \frac{0.5}{PY}.
\end{equation}
Denote $w(p_1)$ to be the sum cost of all edge in $p_1$ and $w(p_2)$ to be the sum cost of all edge in $p_2$.
The expected cost of the policy when choosing first path $p_1$ is
\begin{eqnarray}
  Exp(\pi_{12})&=& \overbrace{P(p_1\; open|x_Y=0)P(x_Y=0) \cdot w(p_1)}^{Ey\;is\;open\;and\;path\;p1\;is\;open} \nonumber\\
&+&\overbrace{ P(p_1\; blocked|x_Y=0)P(x_Y=0) \cdot (2w(p_1)+w(p_2))}^{Ey\;is\;open\;and\;path\;p1\;is\;blocked} + \overbrace{P(x_Y=1)\cdot w(p_2)}^{Ey\;is\;blocked} \nonumber\\
&=& \frac{0.5}{PY}\cdot PY \cdot w(e_Y) + (1-\frac{0.5}{PY}) \cdot PY \cdot (2w(e_Y) + w(e_2)) + (1-PY)\cdot w(e_2) \nonumber\\
&=&2PYCY+0.5(CL-CY) \nonumber\\
\end{eqnarray}
Note that in case that the agent traverses $e_Y$ and $p_1$ is blocked, the agent hits a blocked edge and is forced to pay another CY extra, when the agent moves backward to s.
The expected cost of $\pi_{21}$ is simply CL. Since $CL<2PYCY+0.5(CL-CY)$, the optimal policy is $\pi_{21}$ and $Exp(\pi^*)=CL$. 
It is given that $CL=1+\frac{k}{2^{n+1}}>C$, therefore if F is satisfiable then $Exp(\pi^*)>C$.

$<=$ Suppose that F is not satisfiable. Now, the calculation of the probability is easier because we know that the only case where $e_Y$ is open is when $e_R=0$. 

Therefore, 
\begin{eqnarray}
& &P(p_1\; open,x_Y=0|x_R=0)=1 \label{6}\\
& &P(p_1\; open,x_Y=0|x_R=1)=0 \label{7}
\end{eqnarray}

\begin{eqnarray}\label{8}
P(x_Y=0)&=&P(x_Y=0|x_R=0)P(x_R=0)+P(x_Y=0|x_R=1)P(x_R=1) \nonumber\\
&=& 1 \cdot 0.5 + 0 \cdot 0.5 = 0.5 
\end{eqnarray}

According to equations \ref{4},\ref{6},\ref{7},\ref{8}
\begin{eqnarray}
&&P(p_1\; open|x_Y=0)=\frac {1 \cdot 0.5}{0.5}=1 \\
\Rightarrow &&P(p_1\; blocked|x_Y=0)=0 
\end{eqnarray}
Thus if $e_Y$ is open then $p_1$ is open. 
The expected cost of $\pi_{12}$ is 
\begin{eqnarray}
Exp(\pi_{12})&=&\overbrace{P(p_1\; open|x_Y=0) P(x_Y=0) \cdot w(p_1)}^{Ey\;is\;open\;and\;path\;p1\;is\;open} \nonumber\\
&+& P\overbrace{(p_1\; blocked|x_Y=0)P(x_Y=0) \cdot (2w(p_1)+w(p_2))}^{Ey\;is\;open\;and\;path\;p1\;is\;blocked} + \overbrace{P(x_Y=1)\cdot w(p_2)}^{Ey\;is\;blocked} \nonumber\\
&=& 1\cdot 0.5 \cdot CY + 0\cdot 0.5\cdot (2CY+CL)+0.5\cdot CL \nonumber\\
&=& 0.5\cdot(CY+CL) \nonumber\\
 \end{eqnarray}
Again the expected cost of $\pi_{21}$ is CL. Since $CL>CY\Rightarrow CL>0.5\cdot(CY+CL)$ the optimal policy is $\pi_{12}$ and therefore $Exp(\pi^*)=0.5\cdot(CY+CL) =0.5(1+1+\frac{k}{2^{n+1}})=1+\frac{k}{2^{n+1}}$. It is given that $1+\frac{k}{2^{n+1}}<C$. and thus if F is not satisfiable then $Exp(\pi^*)<C$
\end{proof}
\end{proof}\label{Backward Arcs}
\newpage
\section{Theoretical Properties of Belief-MDP for CTP}\label{secHeuristics}
In the following section, we are given an instance $I=(G,P,w,s,t)$ of CTP, where $G=(V,E)$. We construct a belief state MDP $M_S=(B,A,Tr,R,b_0)$ of I, 
 where S is the state set of I. 
%

\begin{definition}
Policy $\pi$ is called {\em finite} if the AO-graph for $\pi$ is acyclic(DAG).
\end{definition}
\begin{notation}
Denote the expected cost of the optimal policy of $M_S$ in belief state $b$ as $C^{*}(b)$; namely, $C^{*}(b)\equiv C^{\pi^*}(b)$. 
\end{notation}
If the AO-graph for policy $\pi$ is acyclic then $C^{\pi}$ is finite [Bonet, 2010]. 
By definition, there is a traversable edge $\form{s,t}$ in G. Therefore, there is a policy $\pi$ with finite cost and hence $C^{\pi^*}$ is finite [Bonet, 2010]. 
It should be noted that all policies referred to this section are finite.

\begin{defn}\label{MoreBlocked}
The predicate $MoreBlocked(b_1,b_2)$, defined over $b_1,b_2$, is true if and only if the following properties are satisfied: 
\begin{enumerate}
\item $Loc(b_1)=Loc(b_2)$ 
\item For all $e\in E$,
 		\begin{itemize}
		\item $stb(e,b_1)=Open$ if and only if $stb(e,b_2)=Open$.
		\item $stb(e,b_1)=Blocked$ if $stb(e,b_2)=Blocked$.
		\item $stb(e,b_1)=Unknown$ if $stb(e,b_2)=Unknown$ or if $stb(e,b_2)=Blocked$.
 		\end{itemize}
\end{enumerate}
\end{defn}

The predicate $MoreBlocked(b_1,b_2)$ indicates that ``$b_1$ is at least as blocked as $b_2$'', meaning if the pair $b_1,b_2$ satisfies $MoreBlocked(b_1,b_2)$ then $Blocked(b_1)\subseteq Blocked(b_2)$.

Let $E=\{e_1,e_2,e_3\}$. We demonstrate $MoreBlocked(b_1,b_2)$ by the following table:
\footnotesize{
\begin{center}
    \begin{tabular}{ | c | c | c | p{5cm} |}
    \hline
    $\form{b1}$ & $\form{b2}$ & $MoreBlocked(b_1,b_2)$ & Reason \\ \hline
    $\form{v_1,B_{e_1},O_{e_2},U_{e_3}}$ & $\form{v_1,B_{e_1},O_{e_2},U_{e_3}}$ & TRUE & All properties are satisfied\\ \hline 
   $\form{v_1,B_{e_1},O_{e_2},U_{e_3}}$ & $\form{v_1,B_{e_1},O_{e_2},B_{e_3}}$ & TRUE & All properties are satisfied\\ \hline 
   $\form{v_1,U_{e_1},O_{e_2},U_{e_3}}$ & $\form{v_1,O_{e_1},O_{e_2},B_{e_3}}$ & FALSE & $stb(e_1,b_1)=Unknown$ and $stb(e_1,b_2)=Open$ \\ \hline
    \end{tabular}
\end{center}
}
\normalsize

\begin{defn}\label{MoreOpen}
Let $b_1,b_2\in B$. Define the predicate $MoreOpen(b_1,b_2)$ to be true if and only if the following properties are satisfied: 
\begin{enumerate}
\item $Loc(b_1)=Loc(b_2)$ 
\item For all $e\in E$,
 		\begin{itemize}
		\item $stb(e,b_1)=Blocked$ if and only if $stb(e,b_2)=Blocked$.
		\item $stb(e,b_1)=Open$ if $stb(e,b_2)=Open$.
		\item $stb(e,b_1)=Unknown$ if $stb(e,b_2)=Unknown$ or if $stb(e,b_2)=Open$.
		\end{itemize}
\end{enumerate}
\end{defn}

Intuitively, $MoreOpen(b_1,b_2)$ means that ``$b_2$ is at least as open as $b_1$'', where the set of known open edges in $b_1$ is contained in the set of known open edges in $b_2$.

We demonstrate $MoreOpen(b_1,b_2)$ by the following table:
\footnotesize{
\begin{center}
    \begin{tabular}{ | c | c | c | p{5cm} |}
    \hline
    $\form{b1}$ & $\form{b2}$ & $MoreOpen(b_1,b_2)$ & Reason \\ \hline
    $\form{v_1,B_{e_1},O_{e_2},U_{e_3}}$ & $\form{v_1,B_{e_1},O_{e_2},U_{e_3}}$ & TRUE & 
    All properties are satisfied\\ \hline
   $\form{v_1,B_{e_1},O_{e_2},U_{e_3}}$ & $\form{v_1,B_{e_1},O_{e_2},O_{e_3}}$ & TRUE & 
   All properties are satisfied\\ \hline
   $\form{v_1,U_{e_1},O_{e_2},B_{e_3}}$ & $\form{v_1,B_{e_1},O_{e_2},B_{e_3}}$ & FALSE & $stb(e_1,b_1)=Unknown$ and $stb(e_1,b_2)=Blocked$ \\ \hline
    \end{tabular}
\end{center}
}
\normalsize

%

\begin{defn}\label{BlockEdges}
We define the function $BlockEdges:\cal{P}(E)\times B\rightarrow B$ as follows:
Let $b_1,b_2\in B$ such that $b_1=BlockEdges(\hat{E},b_2)$, then $\form{b_1}$ is defined by the following(by its elements): 
\begin{enumerate}
\item $Loc(b_1)=Loc(b_2)$. 
\item For all $e\in \hat{E}$ $stb(e,b_1)=Blocked$. 
\item For all $e\notin \hat{E}$ $stb(e,b_1)=stb(e,b_2)$ 
\end{enumerate}
\end{defn}
Note that by corollary $\ref{mappingb}$, $b_1$ can be determined from $\form{b_1}$.
The function is called $BlockEdges(\hat{E},b_2)$ since it ``blocks'' all the edges in $\hat{E}$ (i.e. for every edge $e\in \hat{E}$ the function $BlockEdges(\hat{E},b_2)$ ``changes'' the status of edge $e$ in $b_2$ to $stb(e,b_1)=Blocked$) where all the other element in $\form{b_2}$ are remained unchanged in $\form{b_1}$.  

For example, we are given belief state $b$ such that $\form{b}=\form{v_1,B_{e_1},U_{e_2},O_{e_3},U_{e_4},U_{e_5}}$ and $\hat{E}=\{e_2,e_4,e_5\}$. Hence, if $b'=BlockEdges(b,\hat{E})$ then $\form{b'}=\form{v_1,B_{e_1},B_{e_2},O_{e_3},B_{e_4},B_{e_5}}$

\begin{prop}\label{corBlockEdges}
For every belief state $b$ and a set of edges $\hat{E}\subset E$ we have $MoreBlocked(\hat{E},BlockEdges(\hat{E},b))$.
\end{prop}
\begin{proof}
Follows immediately from definition \ref{BlockEdges}.
\end{proof}

\begin{defn}\label{OpenEdges}
We define the function $OpenEdges:\cal{P}(E)\times B\rightarrow B$ as follows:
Let $b_1,b_2\in B$ such that $b_1=OpenEdges(\hat{E},b_2)$, then $\form{b_1}$ is defined by the following(by its elements): 
\begin{enumerate}
\item $Loc(b_1)=Loc(b_2)$. 
\item For all $e\in \hat{E}$ $stb(e,b_1)=Open$. 
\item For all $e\notin \hat{E}$ $stb(e,b_1)=stb(e,b_2)$ 
\end{enumerate}
\end{defn}
The function is called $OpenEdges(\hat{E},b_2)$ since it ``open'' all the edges in $\hat{E}$ (i.e. for every edge $e\in \hat{E}$ the function $OpenEdges(\hat{E},b_2)$ ``changes'' the status of edge $e$ in $b_2$ to $stb(e,b_1)=Open$) where all the other element in $\form{b_2}$ are remained unchanged in $\form{b_1}$.  

For example, we are given belief state $b$ such that $\form{b}=\form{v_1,B_{e_1},U_{e_2},O_{e_3},U_{e_4},U_{e_5}}$ and $\hat{E}=\{e_2,e_4,e_5\}$. Hence, if $b'=OpenEdges(b,\hat{E})$ then $\form{b'}=\form{v_1,B_{e_1},O_{e_2},O_{e_3},O_{e_4},O_{e_5}}$

\begin{prop}\label{corOpenEdges}
For every belief state $b$ and a set of edges $\hat{E}\subset E$ we have $MoreOpen(\hat{E},OpenEdges(\hat{E},b))$.
\end{prop}
\begin{proof}
Follows immediately from definition \ref{OpenEdges}.
\end{proof}

\begin{defn}\label{BlockEdges-1}
We define the function $BlockEdges^{-1}:\cal{P}(E)\times B\rightarrow \cal{P}(B)$ as follows: $B_2=BlockEdges^{-1}(\hat{E},b_1)$ is the set of all belief states $b_2\in B$ such that $b_1=BlockEdges(\hat{E},b_2)$. Meaning $BlockEdges^{-1}(\hat{E},b_1)=\setm{b_2}{b_1=BlockEdges(\hat{E},b_2)}$.
\end{defn}

Note that the function $BlockEdges^{-1}$ is somehow a generalization of an inverse function in a way that for every $b\in B$ and $\hat{E}\in \cal{P}(E)$ we get $b=BlockEdges(BlockEdges^{-1}(b,\hat{E}),\hat{E})$ .

For instance, let $B'=BlockEdges^{-1}(\hat{E},b)$ where $\hat{E}=\{e_2\}$ and $b\in B$ such that $\form{b}=\form{O_{e1},B_{e2}}$.
Then $B'=\{b_1,b_2\}$ such that,
\begin{itemize} 
\item $\form{b_1}=\form{O_{e1},B_{e2}}$
\item $\form{b_2}=\form{O_{e1},U_{e2}}$
\end{itemize}

\begin{defn}\label{defRelation}
Let $\hat{E}\in E$. The equivalence relation $\sim_{\hat{E}}$ is defined as follows: \newline
The belief states $b_1,b_2$ satisfy $b_1\sim_{\hat{E}} b_2$ if and only if $MoreBlocked_{\hat{E}}(b_1)=MoreBlocked_{\hat{E}}(b_2)$.
\end{defn}

\begin{defn}
Let $DiffEStatus:B\times B\rightarrow \cal{P}(E)$ be a function. $E'=DiffEStatus(b,b')$ is defined to be the set of all edges incident to $Loc(b')$ which are unknown in $b$ and known in $b'$, i.e.
\[
DiffEStatus(b,b')=\setm{e\in Inc(Loc(b'))}{stb(e,b)=Unknown, stb(e,b')\neq Unknown}
\]
\end{defn}

\begin{defn}
Let $\Pi$ be a set of finite policies over B. We define the function $SimBlocked:\Pi\times \cal{P}(E) \rightarrow \Pi$ as follows:
For every pair of belief states $b,b'\in B$ such that $b'=BlockEdges(b,E')$, the policy $\pi'=SimBlocked(\pi,\hat{E})$ satisfies $\pi'(b)=\pi(b')$.
\end{defn}


\begin{defn}
We define the function $Next:B\times A\rightarrow \cal{P}(B)$ as follows: $B'=Next(b,a)$ is the set of all possible belief state that can be reached from belief state $b$ immediately after taking action $a$. Meaning 
\[ 
Next(b,a)=\setm{b'\in B}{Tr(b,\pi(b),b')>0}
\]
\end{defn}

\begin{defn}
Let $B_1\subset B$ and $\hat{E}\subset E$. Define $Set-Blocked_{\hat{E}}(B_1)=\setm{b_2}{b_2=BlockEdges(\hat{E},b_1),b_1\in B_1}$
\end{defn}

\begin{lemma}\label{GreaterThanZeroBlocked}
Let $b_1,b_2\in B$ such that $MoreBlocked(b_1,b_2)$. Let $b_1',b_2'\in B,\hat{E}\subset E$ such that $b_2'=BlockEdges(\hat{E},b_1')$. Let $a\in A$ such that $Tr(b_1,a,b_1')>0$, then $Tr(b_2,a,b_2')>0$.
\end{lemma}
\begin{proof}
Let $e=\form{v_i,v_j}$. If $a=Sense(e)$ then $a$ can be performed in any belief state.
However, if $a=Move(e)$ then, for every $b\in B$, $a$ can be performed in $b$ if and only if $Loc(b)=v_i$ and $stb(e,b)=Open$.
By definition \ref{MoreBlocked} we have $Loc(b_1)=Loc(b_2)$. 
All belief state in B are consistent with a given realization(all belief states in B describes the knowledge about the same environment), and since $e$ is known in $b_1$($e\in Inc(v_i)$), we get $st(e,b_1)=Open$ if and only if $stb(e,b_2)=Open$. 
Thus an agent in $b_1$ can perform $a=Move(e)$ if and only if an agent in $b_2$ can be perform $a=Move(e)$. 
But we are given that $Tr(b_1,a,b_1')>0$, hence $a$ can be performed in $b_2$ as well.
Let $\hat{E_1}=DiffEStatus(b_1,b_1')$ and let $\hat{E_2}=DiffEStatus(b_2,b_2')$.
We are left to show that the status of all edges in $E\backslash \hat{E_2}$ is equal in $b_2$ and in $b_2'$ i.e. for all $e\in E\backslash \hat{E_2}$ $stb(e,b_2)=stb(e,b_2')$.
By definition \ref{MoreBlocked} we have $UnknownEdges(b_2')\subseteq UnknownEdges(b_1')$, and thus $\hat{E_2}\subseteq \hat{E_1}$.
Thus, by definition of DiffEStatus, the status of all edges in $E\backslash \hat{E_2}$, is equal in $b_2$ and in $b_2'$.
This satisfies all conditions for having $Tr(b_2,a,b_2')>0$. 
\end{proof}

\begin{lemma}\label{wasAppendix2}
Let $b_1,b_2\in B$ such that $MoreBlocked(b_1,b_2)$. Let $B_1=NEXT(b_1)$ and $B_2=Set-Blocked_{\hat{E}}(B_1)$. Let $b_{2i}\in B_2$ and $B_1'\subset B_1$ such that $B_1'=MoreBlocked^{-1}(\hat{E},b_{2i})$. Then, $Tr(b_2,a,b_{2i})=\sum_{b_{1i}\in B_1'}{Tr(b_1,a,b_{1i})}$.
\end{lemma}

\begin{proof}
Let $b_1'\in B_{1}$. By definition of $B_1$, $Tr(b_1,a,b_1')>0$, hence, by lemma \ref{GreaterThanZeroBlocked}, $Tr(b_2,a,b_2')>0$.
Let $E'=DiffEStatus(b_2,b_2')$, we define the probability $P_{2E'}$ by,
\begin{equation}
P_{2E'}=\prod\limits_{e\in E',stb(b_2',e)=Blocked}{p(e)}\prod\limits_{e\in E',stb(b_2',e)=Open}{1-p(e)}\\
\end{equation}
By definition of transition function $Tr(b_2,a,b_2')=P_{2E'}$.
For every $0\leq i\leq n$, where $n=|B_1'|$, define $E_i''=DiffEStatus(b_1,b_{1i})$. For every $i$, $UnknownEdges(b_2')\subseteq UnknownEdges(b_{1i})$, hence $E'\subseteq E_i''$.
We define the probabilities $P_{1E'},P_{1E_i''}$ as follows,
\begin{eqnarray}\nonumber
P_{1E'}=&\prod\limits_{e\in E',stb(b_{1i},e)=Blocked}{p(e)}\prod\limits_{e\in E',stb(b_{1i},e)=Open}{1-p(e)}\nonumber\\
P_{1E_i''}=&\prod\limits_{e\in E_i''\backslash E',stb(b_{1i},e)=Blocked}{p(e)}\prod\limits_{e\in E_i''\backslash E',stb(b_{1i},e)=Open}{1-p(e)}\nonumber\\
\end{eqnarray}\nonumber
By definition of transition function, $Tr(b_1,\pi(b_1),b_{1i})=P_{1E'}P_{1E_i''}$.\newline
Summing up $Tr(b_1,a,b_{1i})$ over all $b_{1i}\in B_1'$ gives,
\begin{eqnarray}
 \sum\limits_{b_{1i}\in B_1'}{Tr(b_1,a,b_{1i})}
=\sum\limits_{b_{1i}\in B_1'}{P_{1E'}P_{1E_i''}}
\overbrace{=}^{*}P_{1E'}\sum\limits_{b_{1i}\in B_1'}P_{1E_i''}
\overbrace{=}^{**}P_{1E'}\nonumber
\end{eqnarray}
* $P_{1E'}$ is equal for all $b_{1i}\in B_1'$.\newline
** $\sum\limits_{b_{1i}\in B_1'}P_{E_i''}=1$ (The sum of all marginal probabilities equals 1).\newline
For all $e\in E'$ we have $stb(e,b_{1i})=stb(e,b_2')$, hence $P_{1E'}=P_{2E'}$.\newline Thus, 
\[
\sum_{b_{1i}\in B_1'}{Tr(b_1,a,b_{1i})}= P_{1E'}=P_{2E'}= Tr(b_2,a,b_2')\\
\]
\end{proof}

\begin{theorem}\label{theorem:MoreBlocked}
Let $b_1,b_2\in B$ such that $MoreBlocked(b_1,b_2)$. Then $C^*(b_2)\geq C^*(b_1)$. 
\end{theorem}
\textbf{Proof outline:}
We prove that for every finite policy $\pi$ there is a finite policy  $\pi'$ such that $C^{\pi}(b_2)=C^{\pi'}(b_1)$.
\begin{proof}
By induction. Let $\hat{E}$ be the subset of all edges $e\in E$ such that $e$ is unknown in $b_1$ and blocked in $b_2$. i.e.
\[
\hat{E}=\setm{e}{e\in E,stb(e,b_1)=Unknown,stb(e,b_2)=Blocked}
\]
Let $B_1=NEXT(b_1)$ and $B_2=Set-Blocked_{\hat{E}}(B_1)$ where $B_2=\{b_{21},b_{22},...,b_{2n}\}$.
Let $\pi$ be a finite policy and let $\pi'$ be a policy defined as follows:\newline
For every belief state $b\in B$ $\pi'(b)=\pi(BlockEdges(\hat{E},b))$. Meaning, $\pi'$ maps every belief state $b$ to an action $a$ by simulating $\pi$ on belief state $b'=BlockEdges(\hat{E},b)$ and output $a=\pi(b')$.     
Clearly, an agent acting according to $\pi'$ will never traverse any edge in $\hat{E}$.
We show by induction that $C^{\pi}(b_2)=C^{\pi'}(b_1)$ as follows,
\begin{itemize}
\item Base case: If $b_1,b_2$ are terminal states then $C^{\pi}(b_2)=C^{\pi'}(b_1)\equiv 0$ (by definition of terminal states).
\item 
Assume by induction that for every $b_1'\in B_1$ and $b_2'\in B_2$ we have $C^{\pi}(b_2')=C^{\pi'}(b_1')$.
By definition of $\pi'$ we have $\pi'(b_1)=\pi(b_2)$. Let $a=\pi'(b_1)$. 
Since $\pi'(b_1)=\pi(b_2)$ we have $a=\pi(b_2)$ 
Hence, according to bellman equations,
\begin{eqnarray}\label{BellmanDouble}
&C^{\pi'}(b_1)=R(b_1,a)+\sum_{b_1'\in B_1}{Tr(b_1,a,b_1')C^{\pi'}(b_1')}\nonumber\\
&C^{\pi}(b_2)=R(b_2,a)+\sum_{b_2'\in B_2}{Tr(b_2,a,b_2')C^{\pi}(b_2')}\nonumber\\
\end{eqnarray}
In order to show that $C^{\pi}(b_2)=C^{\pi'}(b_1)$ we show the equivalence in the right sides of the equations above.\newline
Given $a=Sense(e)$ then,
\begin{itemize}
\item $R(b_1,a)=R(b_2,a)=SC(e)$ (action Sense is always performable).
\item $\sum_{b_1'\in B_1}{Tr(b_1,a,b_1')C^{\pi'}(b_1')}=\sum_{b_2'\in B_2}{Tr(b_2,a,b_2')C^{\pi}(b_2')}$. Let $b_{1B}\in B_{1},b_{2B}\in B_{2}$ be the belief states that are reached immediately after the agent has sensed $e$ in $b_1,b_2$ respectively, and e was found to be blocked. By definition of transition function, $Tr(b_1,a,b_{1B})=Tr(b_2,a,b_{2B})=p(e)$.
By assumption of induction, $C^{\pi}(b_{1B})=C^{\pi'}(b_{2B})$. Hence, $Tr(b_1,a,b_{1B})C^{\pi}(b_{1B})=Tr(b_2,a,b_{2B})C^{\pi'}(b_{2B})$. 
Similarly, let $b_{1O}\in B_{1},b_{2O}\in B_{2}$ be the belief states that are reached immediately after the agent has sensed $e$ in $b_1,b_2$ respectively, and e was found to be open.
Then $Tr(b_1,a,b_{1O})=Tr(b_2,a,b_{2O})=1-p(e)$. By assumption of induction, $C^{\pi}(b_{1O})=C^{\pi'}(b_{2O})$. Hence, $Tr(b_1,a,b_{1O})C^{\pi}(b_{1O})=Tr(b_2,a,b_{2O})C^{\pi'}(b_{2O})$. \newline
Thus,
\[
\sum_{b_1'\in B_1}{Tr(b_1,a,b_1')C^{\pi'}(b_1')}=\sum_{b_2'\in B_2}{Tr(b_2,a,b_2')C^{\pi}(b_2')}=p(e)V_B+(1-p(e))V_O
\]
\end{itemize}
Where $V_B$ denotes $C^{\pi'}(b_{1B})$ and $V_O$ denotes $C^{\pi'}(b_{1O})$(recall that $C^{\pi'}(b_{1B}),C^{\pi'}(b_{2B})$ as well as $C^{\pi'}(b_{1O}),C^{\pi'}(b_{2O})$ are interchangeable).\newline
Given $a=move(e)$, where $e=\form{v_i,v_j}$, then,
\begin{itemize}
\item $R(b_1,a)=R(b_2,a)$. By definition of reward function, for every $b\in B$ $R(b,a)>0$ if and only if $Tr(b,a)>0$ if and only if $Loc(b)=v_i$ and $stb(e,b)=Open$.
Thus, $R(b_1,a)=R(b_2,a)>0$ if and only if $Loc(b_1)=Loc(b_2)$ and $stb(e,b_1)=stb(e,b_2)=Open$.
From definition \ref{MoreBlocked} it follows that $Loc(b_1)=Loc(b_2)$. 
In addition, $stb(e,b_1)=Open$ if and only if $stb(e,b_2)=Open$, due to the following:
\begin{enumerate}
\item All belief states reachable from $b_0$, and in particular the belief states $b_1,b_2$, are referred to the same unknown given environment $s$. Hence, $stb(e,b_1)=Open$ only if $stb(e,b_2)=Open$ or $stb(e,b_2)= Unknown$ and similarly $stb(e,b_2)=Open$ only if $stb(e,b_1)=Open$ or $stb(e,b_1)= Unknown$.
\item $stb(e,b_1)\neq Unknown$, $stb(e,b_2)\neq Unknown$. Since $Loc(b_1)=Loc(b_2)$ we have $e\in Inc(Loc(b_1))$ if and only if $e\in Inc(Loc(b_2))$. 
By definition, an agent located in vertex $v$, knows the status of all edges incident to v. Thus $stb(e,b_1)\neq Unknown$, $stb(e,b_2)\neq Unknown$.
\end{enumerate}
\item $\sum_{b_1'\in B_1}{Tr(b_1,a,b_1')C^{\pi'}(b_1')}=\sum_{b_2'\in B_2}{Tr(b_2,a,b_2')C^{\pi}(b_2')}$.\newline
Let $B'_1=\{B_{11},B_{12},...,B_{1n}\}$ be a partition of B by the equivalence relation $\equiv_{\hat{E}}$ such that without loss of generality $B_{1i}=MoreBlocked^{-1}(\hat{E},b_{2i})$ for every $0\leq i \leq n$.
Let $V_i=C^{\pi'}(b_{2i})$ for every $0\leq i \leq n$. Then, by assumption of induction, for every $b_{1i}\in B_{1i}$, we have $C^{\pi}(b_{2i})=C^{\pi'}(b_{1i})$, hence $V_i=C^{\pi'}(b_{1i})$ as well.
 

According to lemma {\ref{wasAppendix2}}, summing up $Tr(b_1,a,b_{1i})$ over all $b_{1i}\in B_{1i}$ gives,
\begin{equation}\label{eqAppendix2}
\sum_{b_{1i}\in B_{1i}}{Tr(b_1,a,b_{1i})}=Tr(b_2,a,b_{2i})
\end{equation}
Hence,
\begin{eqnarray}\label{eqTr1}
 &Tr(b_2,a,b_{2i})C^{\pi}(b_{2i})\quad & \nonumber\\
=&Tr(b_2,a,b_{2i})V_i\quad &\nonumber\\
=&V_i\sum\limits_{b_{1i}\in B_{1i}}{Tr(b_1,a,b_{1i})}\quad & \text{(by equation \ref{eqAppendix2})}\nonumber\\
=&\sum\limits_{b_{1i}\in B_{1i}}{Tr(b_1,a,b_{1i})}V_i\quad & \text{($V_i$ is constant for every $b_{1i}\in B_{1i}$)}\nonumber\\
=&\sum\limits_{b_{1i}\in B_{1i}}{Tr(b_1,a,b_{1i})C^{\pi'}(b_{1i})}\quad &\nonumber\\
\end{eqnarray}
Thus, summing up over all transition functions gives,
\begin{eqnarray}
 &\sum\limits_{{b_{1}'}\in B_{1}}{Tr(b_1,a,b_{1'})C^{\pi'}(b_{1'})}\quad & \nonumber\\
=&\sum\limits_{B_{1i}\in B'_1} \sum\limits_{b_{1i}\in B_{1i}}{Tr(b_1,a,b_{1i})C^{\pi'}(b_{1i})}\quad &\text{($B'_1$ is a partition of $B_1$)}\nonumber\\
=&\sum\limits_{{b_{2i}}\in B_{2i}}{Tr(b_2,a,b_{2i})C^{\pi}(b_{2i})}\quad & \text{(by equation $\ref{eqTr1}$)}\nonumber\\
=&\sum\limits_{{b_{2'}}\in B_{2}}{Tr(b_2,a,b_{2'})C^{\pi}(b_{2'})}\quad & \text{(without loss of generality $b_{2i}=b_{2'}$)}\nonumber\\
\end{eqnarray}
This completes the induction proof.
  \end{itemize}
\end{itemize}
We have shown that for every finite policy $\pi$ we can define a finite policy $\pi'$ which satisfies $C^{\pi}(b_2)=C^{\pi'}(b_1)$. Since the optimal policy is also finite, the equation $C^*(b_2)= C^{\pi'}(b_1)$ holds. 
Thus, in general $C^*(b_2)\geq C^*(b_1)$.
\end{proof}

In figure \ref{fig:ContBfig}, we demonstrate the ``simulation'' of policy $\pi'$ presented in theorem \ref{theorem:MoreBlocked}. Here, we are given a graph G=(V,E), where $V=\{s,v_1,v_2,v_3,t\}$ (abusing notation, we denote one vertex as s, and one as t), \newline $E=\{(s,v_1),(v_1,v_2),(v_1,v_3),(v_2,v_3),(v_1,t),(v_2,t),(v_3,t)\}$, where w, which is noted with each edge in the figure, represents the edge weight. 
        In addition, two belief states $b_1,b_2\in B$ are given with the following forms:
\newline $\form{b_1}=\form{s,O_{(s,v_1)},B_{(v_1,v_2)},O_{(v_1,v_3)},O_{(v_2,v_3)},O_{(v_1,t)},O_{(v_2,t)},U_{(v_3,t)}}$
\newline $\form{b_2}=\form{s,O_{(s,v_1)},B_{(v_1,v_2}),O_{(v_1,v_3)},O_{(v_2,v_3)},O_{(v_1,t)},O_{(v_2,t)},B_{(v_3,t)}}$.
\newline  On the upper left of the figure, the edges status are based on $b_1$ and on the upper right the edges status are based on  $b_2$ , where the green lines represent open edges, black lines represent blocked edges, and red lines represent unknown edges. Notice that $b_1,b_2$ satisfy $MoreBlocked(b_1,b_2)$ and thus $b_2=BlockEdges(b_1)$. The lower figures represent the execution of policy $\pi'$ on $b_1$ where $\pi'(b_1)=\pi(b_2))$ and $\pi^*$ on $b_2$. We see the equivalence between the policies(the same sequence of actions). Notice that agent acting according to $\pi'$(as shown by the doted line), does not perform the action $move(v_1,t)$ although it is optimal, since $\pi'$ treats all edges in $Blocked(b_2)$ as blocked in $b_1$(edge $(v_1,t)$ in this figure)

\begin{figure}[h!]
        \centering
        \centerline{\includegraphics[width=3.5in,height=2.5in]{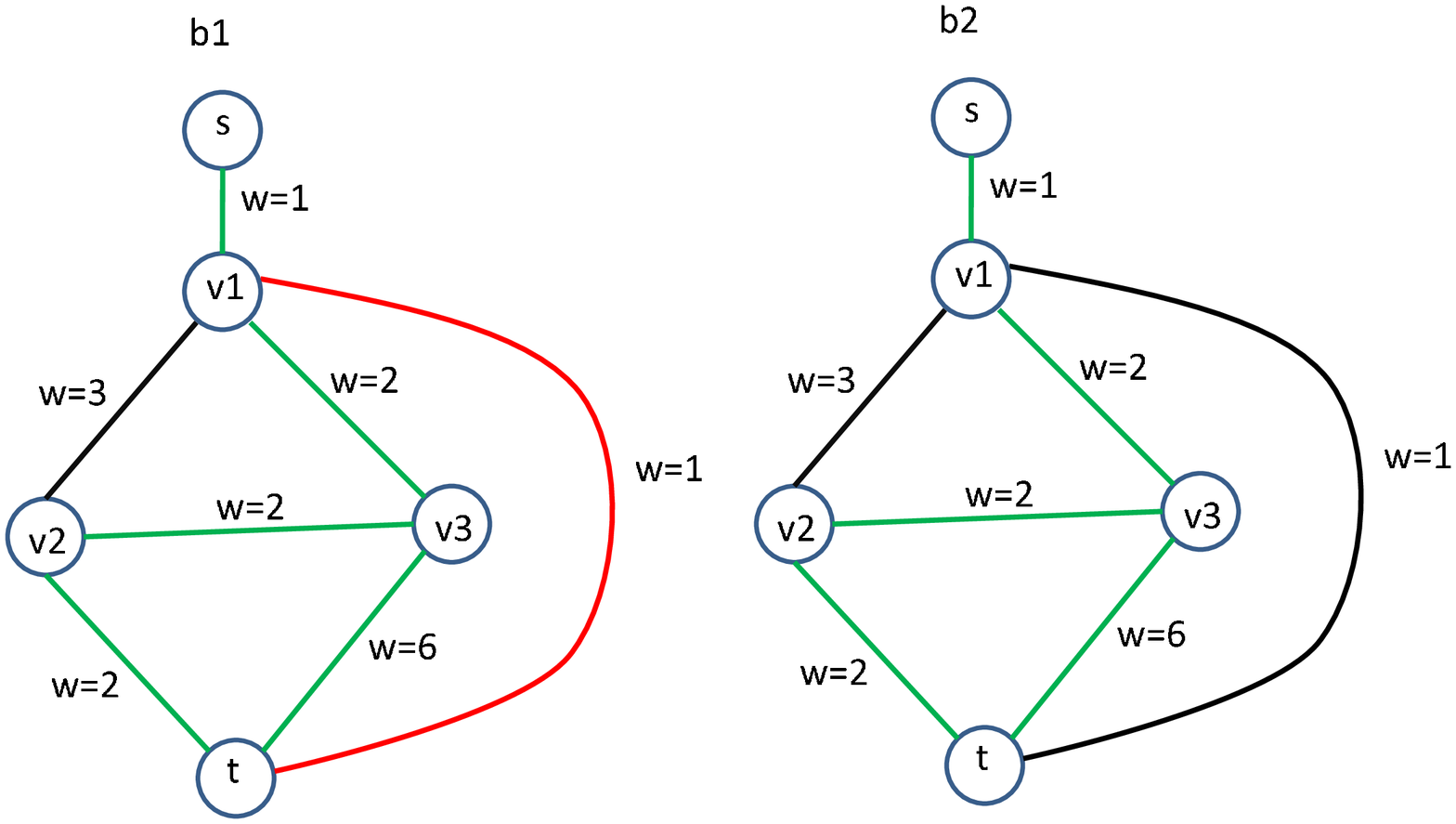}}
        \centerline{\includegraphics[width=3.5in,height=2.5in]{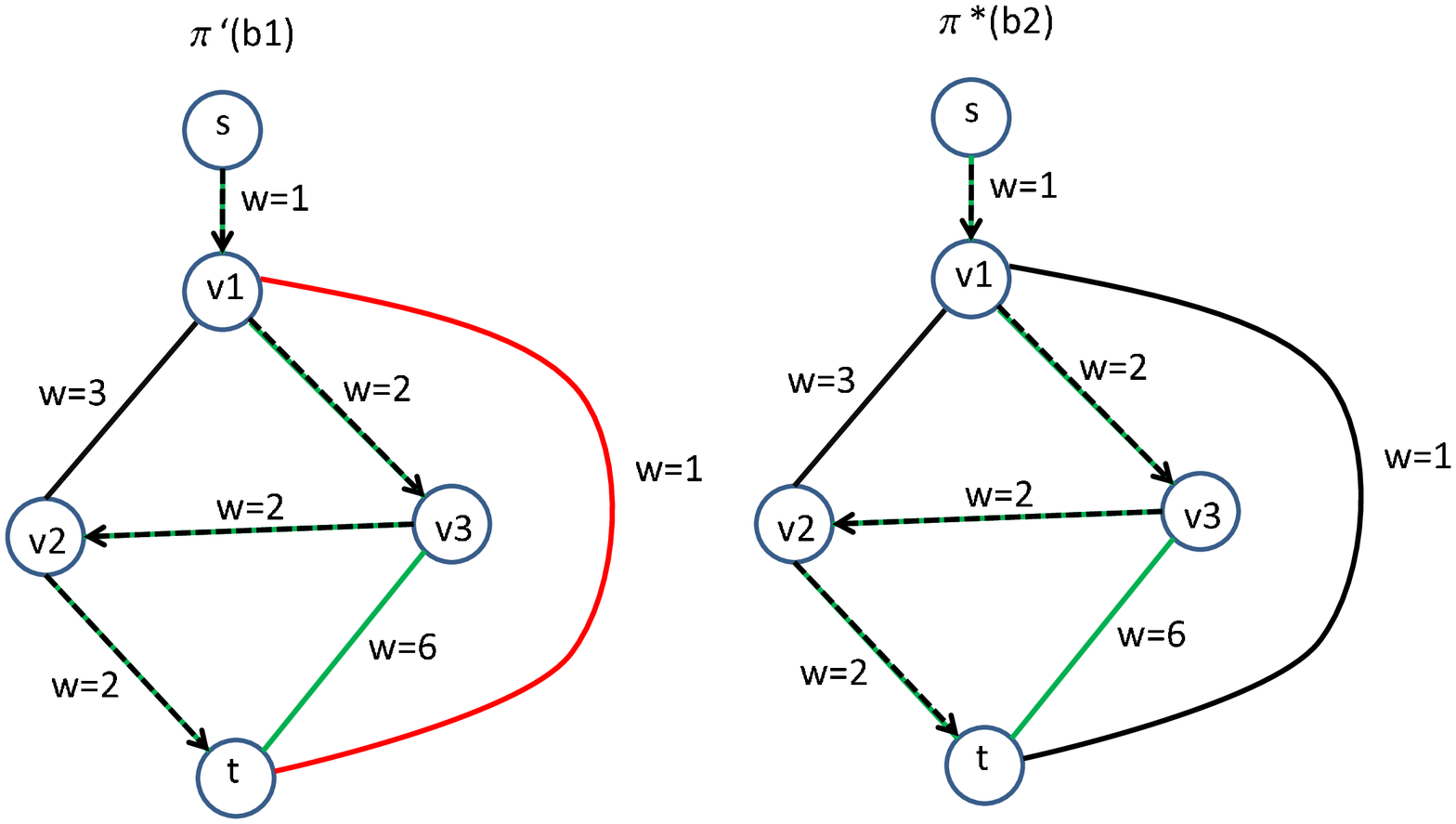}}
        \caption{Demonstration of SIMBlocked}
        \label{fig:ContBfig}
\end{figure} 

\begin{cor}\label{corMoreBlocked}
Suppose that $MoreBlocked(b_1,b_2)$ is true for $b_1,b_2\in B$, then if $h(b_1)$ is an admissible heuristic(optimistic) of $b_1$, then $h(b_1)$ is an admissible heuristic of $b_2$ as well.
\end{cor}
\begin{proof}
Follows from theorem \ref{theorem:MoreBlocked}
\end{proof}

\newpage

In the following statement, we use theorem \ref{theorem:MoreBlocked} to show that if $MoreOpen(b_1,b_2)$ then  $C^*(b_2)$ is a lower bound of $C^*(b_1)$.

\begin{lemma}\label{MoreOpenTheorem}
Let $b,b_{open}\in B$ such that $MoreOpen(b,b_{open})$. Then, $C^*(b)\geq C^*(b_{open})$. 
\end{lemma}
\begin{proof}
Let $b,b_{open},b_{blocked}\in B$ such that $b,b_{open},b_{blocked}$ differs only by the status of edge $\hat{e}$, where $\hat{e}$ is Unknown in $b$, Open in $b_{open}$ and Blocked in $b_{blocked}$.
We prove that $C^*(b)\geq C^*(b_{open})$. 
\begin{itemize}
\item By the law of total probability we can express $C^*(b_1)$ as follows:
\begin{equation}\label{meanP}
C^*(b_1)=P(\hat{e})C^*(b_{blocked})+(1-P(\hat{e}))C^*(b_{open})
\end{equation}
From lemma \ref{MoreBlocked} we get that $C^*(b_1)\leq C^*(b_{blocked})$. Thus, there is $R \geq 0$ such that, 
\[
C^*(b_{blocked})=C^*(b_1)+R
\]
We can express equation \ref{meanP} as follows:
\[
C^*(b_1)=p(\hat{e})(C^*(b_1)+R)+(1-p(\hat{e}))C^*(b_{open})
\]
Substracting $p(\hat{e})C^*(b_1)$ from both sides and then dividing both sides by $(1-p(\hat{e}))$, we get:
\[
C^*(b_1)=R\frac{p(\hat{e})}{1-p(\hat{e})}+C^*(b_{open})
\]  
Since $R(\frac{p(\hat{e})}{1-p(\hat{e})})\geq 0$ we get 
\[
C^*(b_1)\geq C^*(b_{open})
\]
Trivially, it can be shown by induction that $C^*(b)\geq C^*(b_{open})$, for any set of edges $\hat{E}$ such that edges in $\hat{E}$ are unknown in $b$ and open in $b_{open}$.  
\end{itemize}
\end{proof}

\begin{cor}\label{corMoreOpen}
Suppose that $b_1,b_2$ satisfy $MoreOpen(b_1,b_2)$, then if $h(b_2)$ is an admissible heuristic(optimistic) of $b_2$, then $h(b_2)$ is an admissible heuristic of $b_1$ as well.
\end{cor}
\begin{proof}
Follows immediately from lemma \ref{MoreOpenTheorem}
\end{proof}

In the rest of the section we provide some new definitions and lemmas in order to prove another lower bound to the cost of optimal policy on belief state $b_1$ by a cost of the optimal policy on another belief state $b_2$ where, in contrast to the previous lemmas, the locations of the agents in $b_1$ and in $b_2$ are different.

\begin{defn}\label{DiffLoc}
Let $b_1,b_2\in B$. We define the predicate $DiffLoc(b_1,b_2)$ to be true if and only if $Loc(b_1)\neq Loc(b_2)$ and for every edge $e\in E$ $stb(e,b_1)=stb(e,b_2)$.
\end{defn}


\begin{defn}
Define the set $D_B$ to be the set of all pair $b_1,b_2\in B$ such that $DiffLoc(b_1,b_2)$. Meaning $D_B=\setm{<b_1,b_2>}{b_1,b_2\in B, DiffLoc(b_1,b_2)}$. We call $D_B$ the \textit{DiffLoc} of B.
\end{defn}

\begin{defn}
Define the set $Open_b\subseteq E$ to be the set of all edges that are known to be open in belief state $b$. Meaning $E_b=\setm{e}{e\in E,b\in B,stb(e,b)=Open}$  
\end{defn}

\begin{defn}
Let $\cal{P}$ be the set of all paths in G and let $D_B$ be the \textit{DiffLoc} of B. We define the function $shortestPath:D_B\rightarrow \cal{P}$ such that for $<b_1,b_2>\in D_B$ $p=shortestPath(<b_1,b_2>)$ is the shortest path between $v_1=Loc(b_1)$ and $v_2=Loc(b_2)$ in graph $G'=(V,E_{b_1})$.(Note that $G'=(V,E_{b_1}$ is a subgraph of G=(V,E) since $E_{b_1}\subseteq E$)
\end{defn}
Note that $E_{b_1}=E_{b_2}$ for every $<b_1,b_2>\in D_B$, since the status of edges specified by $b_1$ and $b_2$ are equal.

\begin{defn}
Let $\cal{P}$ be the set of all paths in G. We define a path cost function $C_P:\cal{P}\rightarrow \cal{R}$ as follows: Let $p$ be a path, then $C_P(p)=\sum_{e \in p}{c(e)}$.
\end{defn}

\begin{definition}
We define the set $KE_b\subseteq E$ to be the set of all known edges in belief state $b$. Meaning $KE_b=\setm{e}{e\in E,b\in B,stb(e,b)=Open\vee stb(e,b)=Blocked}$. $KE_b$ is called the knowledge in b.
\end{definition}

\begin{lemma}
The value of information in the Canadian Traveler Problem is never less than zero.
\end{lemma}
\begin{proof}
Let $b_1,b_2\in B$ such that $b_2$ is reached from $b_1$ immediately after performing SENSE(e) (Suppose hypothetically that an agent in $b_1$ is allowed to perform action SENSE(e) once, on any edge $e\in E$, with no cost) and we get $KE_{b_1}\subseteq KE_{b_2}$. Hence, this lemma is true if and only if $C^*(b_1)\geq C^*(b_2)$ (by definition of value of information). Since $KE_{b_1}\subseteq KE_{b_2}$, we can simulate any policy of $b_1$ on $b_2$ by ``ignoring'' the information received from SENSE(e) in $b_2$ and in particular the optimal policy $\pi^*$. 
Therefore, we can define the policy $\pi_{b_1}^*$ such that $\pi_{b_1}^*(b_2')=\pi^*(b_1')$ for every belief state $b_1'$ reachable from $b_1$ and belief state $b_2'=b_{{1'}_e=st_e}$ where $st_e=SENSE(e)$.  Since $b_1$ and $b_2$ are referred to the same physical environment, the execution of $\pi^*$ on $b_1$ will be equal to the execution of $\pi_{b_1}^*$ on $b_2$(will produce the same sequence of actions). Hence, $C^{\pi^*_{b_1}}(b_2)= C^*(b_1)$, and in general $C^*(b_1)\geq C^*(b_2)$.
\end{proof}

\begin{lemma}\label{diffLoc}
Let $b_1,b_2\in B$ such that $<b_1,b_2>\in D_B $, then $C^*(b_1)+C_p(shortestPath(<b1,b2>))\geq C^*(b_2)$.
\end{lemma}
\textbf{Proof outline:} In the next lemma we show that $C^*(b_1)+C_p(shortestPath(<b1,b2>))\geq C^*(b_2)$. This gives us a lower bound to $C^*(b_1)$ since $C^*(b_1)\geq C^*(b_2) - C_p(shortestPath(<b1,b2>))$. We show this by
defining a policy $\hat{\pi}$ such that when executing $\hat{\pi}$ on $b_2$ we have the following: An agent under $\hat{\pi}(b_2)$ moves through the shortest path(under assumption that all unknown edges in $b_2$ are blocked) to the location referred by $b_1$ $(Loc(b_1))$, reaching belief state b', and then under $\hat{\pi}(b')$ the agent is followed by the execution of the optimal policy $\pi^*$. 

\begin{proof}
Assume(by negation) that 
\begin{equation}\label{negAssumption}
C^*(b_1)+C_p(shortestPath(<b1,b2>))< C^*(b_2).
\end{equation} 
Let $v_1=Loc(b_1)$ and $v_2=Loc(b_2)$.
We define a new policy $\hat{\pi}$ such that executing it on $b_2$ gives the following:
\begin{enumerate}
\item An agent under $\hat{\pi}(b')$ traverses the path $p=shortestPath(<b1,b2>)$ straightforward from $v_2$ to $v_1$(which is always possible since all edges in p are open).  Let $b'$ be the belief state that the agent reaches when arriving $v_1$.
\item Immediately after reaching $b'$, the agent under $\hat{\pi}(b')$ acts according to the optimal policy until reaching t. Meaning, 
for any belief state b'' reachable from $b'$ $\hat{\pi}(b'')=\pi^*(b'')$.
\end{enumerate}
Clearly, 
\begin{equation}\label{diff1}
C^{\hat{\pi}}(b_2)=\overbrace{C_p(shortestPath(<b1,b2>))}^{Cost\; of\; traversing\; p}+\overbrace{C^*(b')}^{Exp.\;cost\;of\;\pi^*\;on\;b'}
\end{equation}
We claim that $KE_{b_1}\subseteq KE_{b'}$. This result from:
\begin{enumerate}
\item $KE_{b_1}=KE_{b_2}$. The knowledge in $b_1$ and $b_2$ is equal by definition of element pairs $<b_1,b_2>\in D_B$. 
\item $KE_{b_2}\subseteq KE_{b'}$. An agent in $b_2$ that follows the shortest path p  may obtain information if a vertex in path p(a vertex that is incident to two edges in p) is incident to an edge that has not been sensed yet.
 \end{enumerate}
 Since an agent A1 in $b_1$ and an agent A2 in $b'$ are at the same physical state $s\in S$, and the knowledge of A1 about s is a subseteq of the knowledge of A2 about s $(KE_{b_1}\subseteq KE_{b'})$ we get $C^*(b')\leq C^*(b_1)$(followed by the lemma of value of information).
 
Thus, 
\begin{equation}
C^*(b')+C_p(shortestPath(<b1,b2>))\leq C^*(b_1)+C_p(shortestPath(<b1,b2>))
\end{equation} 
 by equation \ref{diff1} we get,
 \begin{equation}
C^{\hat{\pi}}(b_2)\leq C^*(b_1)+C_p(shortestPath(<b1,b2>))
 \end{equation} 
Following assumption \ref{negAssumption} we get, 
\begin{equation}
C^{\hat{\pi}}(b_2)\leq C^*(b_2) 
\end{equation} 
which is a contradiction to the optimality of policy $\pi^*$. 
Hence, 
\begin{equation}
C^*(b_1)+C_p(shortestPath(<b1,b2>))\geq C^*(b_2)
\end{equation} 
\end{proof}

\begin{cor}\label{corDiffLoc}
Let $b_1,b_2\in B$ such that $<b_1,b_2>\in D_B $, then if $h(b_1)$ is a lower bound to $C^*(b_1)$, then $h(b_1)$ is a lower bound to $C^*(b_2)+C(shortestPath<b_1,b_2>)$ as well.
\end{cor}
\begin{proof}
Follows immediately from lemma \ref{diffLoc}
\end{proof}

We want to define two relations which will be used in the next section:

%
%
%
%
%
\chapter{Generalizing PAO*}
\section{General Propagation AO*}\label{sec:Gen-PAO}
In many cases, PAO* lowers dramatically the running time by reducing the state space, however, it assumes that each vertex is connected to at most {\em one} unknown edge, such that each AND node in the AND/OR graph has at most two successors.
We present the {\em generalized propagation AO*} algorithm  (Gen-PAO in short), a generalization of PAO*, which does not assume any preknown knowledge of the graph (except the edges incident to $s$ which are always defined as Open). 
Gen-PAO solves the Sensing-CTP as well. Each sensing action is associated with a sensing AND node, where each sense node has only two children nodes for the two possible statuses of the sensed edge (Open/Blocked). This variant is extremely harder than the basic CTP since the agent can sense any unknown edge in any state and hence, the branching factor of the OR nodes is significantly larger.  
\subsection{Gen-PAO Heuristics}
Similarly to AO* and PAO*, each iteration of Gen-PAO is based on two phases: Expansion and Propagation. Gen-PAO differs from AO* and PAO* only in the Propagation phase (i.e, the Main and Expand method as presented in algorithm \ref{alg1} are part of Gen-PAO as well). However, in the propagation phase, Gen-PAO propagates the heuristic values not only upwards to the ancestors as AO*, but to the entire state space, incorporating three novel heuristics: HBlocked, HOpen, and HDiffLoc (line 13). The heuristic HBlocked is based on the predicate MoreBlocked (definition $\ref{MoreBlocked}$), HOpen is based on the predicate MoreOpen (definition $\ref{MoreOpen}$), and HDiffLoc is based on the predicate Diffloc(definition $\ref{DiffLoc}$).
Let $Z$ be the set of belief states that expanded by Gen-PAO and let $b\in Z$. The heuristics are defined as follows: 
 \begin{itemize}
\item HBlocked(b):\textit{ If there is a belief state $b'\in Z$ that satisfies $MoreBlocked(b,b')$	and $h(b) > h(b')$ then $h(b') \gets h(b)$}.
\item HOpen(b):\textit{ If there is a belief state $b'\in Z$ that satisfies $MoreOpen(b',b)$	and $h(b') < h(b)$ then  $h(b') \gets h(b)$}.
\item HDiffLoc(b):\textit{ If there is a belief state $b'\in Z$ that satisfies $DiffLoc(b,b')$	and $h(b') < h(b)-CSP(b,b')$ then  $h(b') \gets h(b)-CSP(b,b')$}, where CSP is the cost of the shortest path from b to b'.
\end{itemize}
\begin{figure}[h!]
        \centering
        \centerline{\includegraphics[width=3.5in,height=2.5in]{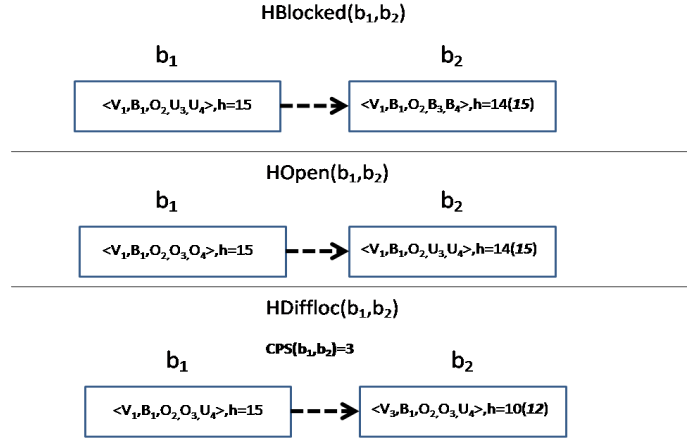}}
        \caption{Demonstration of the heuristic methods}
        \label{fig:heuristicMethods}
\end{figure} 
Belief states of which values are updated due to propagation from b are called {\em propagated belief state of b}.  
Notice that $HBlocked(b),HOpen(b)$, and $HDiffLoc(b)$ always raise up the heuristic value of the propagated belief states of b. However, 
due to corollaries $\ref{corMoreBlocked},\ref{corMoreOpen}$, and $\ref{corDiffLoc}$ respectively, heuristics $HBlocked(b),HOpen(b)$, and $HDiffLoc(b)$ are admissible, and thus they are upper bounded by $V^*(b)$.\newline
Figure \ref{fig:heuristicMethods} illustrate an update of belief state $b_2$ by the three heuristic methods (the the new heuristic values of $b_2$ are notified in parenthesis) 


\begin{algorithm}
\caption{Gen-PAO}
\label{alg2}
procedure Propagate(Graph $G_{AO}$,  NodeSet Z) 
\begin{algorithmic}[1]
\WHILE{$Z\neq \phi$}
\STATE select $z_i\in Z$ such that $z_i$ has no children in Z;
\IF {$z_i$ is AND node}
\STATE $ f(z_i)\gets \sum\limits_{\substack{z_j\in successor(z_i)\\ a=(z_i,z_j)}}[tr(z_i,a,z_j)h(z_j)+c(a)]$; 
\IF {all successors of $z_i$ are marked SOLVED}
\STATE MarkSolved($z_i$)
\ENDIF
\ENDIF
\IF {$z_i$ is OR node}
\STATE $f(z_i)\gets \min\limits_{\substack{z_j\in successor(z_i)\\ a=(z_i,z_j)}}{[tr(z_i,a,z_j)h(z_j)+c(a)]}$;
\STATE $MarkAction(z_i,z_j)$; 
\IF {SOLVED($z_j$)}
\STATE MarkSolved($z_i$);
\ENDIF
\STATE invokeHeuristics(Graph $G_{AO}$,  Node $z_i$);
\ENDIF
\STATE remove $z_i$ from Z;
\ENDWHILE
\end{algorithmic}
\end{algorithm}

The heuristic methods are invoked when a value of belief state is updated(procedure Propagate, line 13). The heuristics methods are ineffective on a major part of the expanded states (i.e.  most of the expanded states $b'\in Z$  do not satisfy the predicates $MoreBlocked(b,b'),MoreOpen(b,b')$, and $DiffLoc(b,b')$, for a given expanded state $b$, and their values are not updated by their compatible heuristic methods).
In order to reduce the number of expanded states that are checked for update, we use two data structures:
$BlockedStructue$ and $OpenStucture$. For defining these structures we define new equivalence relations:
The equivalence relations $\equiv_o$ and $\equiv_b$ on belief states $b_1,b_2$ are defined as follows:
\begin{defn}\label{similarOpen}
$b_1\equiv_o b_2$ if and only if: 
\begin{enumerate}
\item $Loc(b_1)=Loc(b_2)$ 
\item $Open(b_1)=Open(b_2)$
\end{enumerate}
\end{defn}
Similarly, 
\begin{defn}\label{similarBlocked}
$b_1\equiv_b b_2$ if and only if: 
\begin{enumerate}
\item $Loc(b_1)=Loc(b_2)$ 
\item $Blocked(b_1)=Blocked(b_2)$
\end{enumerate}
\end{defn} 

 Each of these structures is a hash table that contains the entire expanded state space $Z$ (more precisely the hash table {\em refers} to $Z$), where the entires of each table divide $Z$ into equivalence classes called ``buckets''. Namely, the set of buckets $\{m_{o1},...,m_{on}\}$, in $BlockedStructue$, partitions $Z$ in to equivalence classes by the relation $\equiv_o$, while the set of buckets $\{m_{b1},...,m_{bn}\}$, in $OpenStucture$, partitions $Z$ into equivalence classes by the relation $\equiv_b$. By definition above, $HBlocked(b)$, never updates the heuristic value $h(b')$ if b and b' do not share the same bucket of $BlockedStructue$, and similarly, $HOpen(b)$ never update a value of b' if b and b' do not share the same bucket of $OpenStructue$.\newline
Procedure $propBlocked$ (Algorithm \ref{alg2}) implements the heuristic $HBlocked$. $classBlocked(b)$ (line 1) returns the set $Z_B$ of all expanded nodes whose belief states are in the same bucket of $BlockedStructue$ as the belief state of node $z$.
Similarly, procedure $propOpen$ is implementation of the heuristic $HOpen$. $classOpen(b)$ (line 1) returns the set of all expanded nodes $Z_O$ whose belief states that are at the same bucket of $OpenStructue$ as the belief state of node $z$.
%
\begin{algorithm}
\caption{Heuristic methods for Gen-PAO}
\label{alg2}

procedure propBlocked(NodeSet Z, Node z) 
\begin{algorithmic}[1]
\STATE $Z_B \gets classBlocked(Z); $
\FORALL{($z_i\in Z_B|z_i\neq z$)}
\IF {$MoreBlocked(z,z_i)$ and $h(z)>h(z_i)$}
\STATE $h(z_i)\gets h(z)$;
\ENDIF
\ENDFOR
\end{algorithmic}
procedure propOpen(NodeSet Z,Node s)
\begin{algorithmic}[1]
\STATE $Z_O \gets stateStructureOpen(Z)$
\FORALL{($z_i\in Z_O|z_i\neq z$)}
\IF {$MoreOpen(z_i,z)$  and $h(z)>h(z_i)$}
\STATE $z_i\gets z$;
\ENDIF
\ENDFOR
\end{algorithmic}
\end{algorithm}

\subsection{Eliminating Duplicate Nodes}\label{sec:Gen-PAO-Union}
In most cases Gen-PAO expands the same node more than once. This may lead to a large expense of memory and run time when it is generated on large graphs. Taking this into consideration, we introduce the Gen-PAO-EDN (short for Gen-PAO Eliminating Duplicate Nodes) algorithm, a variation of Gen-PAO, that maintains a single OR node for every state, by eliminating all duplicate OR nodes (more precisely, preventing the expansion of duplicate nodes) which shares the same state into one OR node. 
There are two key differences between Gen-PAO-EDN an Gen-PAO:
\begin{itemize}
\item Gen-PAO maintains one representative OR node for each state. When Gen-PAO-EDN expands an AND node, it creates a new OR node only if its associated state is not represented by any OR node in the AND/OR graph. Otherwise, if a representative OR node to this state already exists, then the expanded AND node becomes an additional parent of the representative OR node. 
\item The AND/OR graph may contain cycles(not a tree as in AO* and Gen-PAO). A special type of cycle, called {\em strongly connected} (defined below), induces loops in the propagation phase if the cycle is a subgraph of the partial solution. 
\begin{defn}
Let $AO$ be an AND/OR graph, $A_1,...A_n\in AO$ be AND nodes, and $O_1,...,O_n\in AO$ be OR nodes. A cycle $O_1->A_1->O_2->A_2->...->O_n->A_n->O_1$ is strongly connected if for every $1\leq i\leq n$ $A_i$ is the preferred son of $O_i$. 
\end{defn}
If the propagate method enters a strongly connected circle $C$ (which occurs when the propagation goes upwards to the ancestors), the heuristic values are re-updated every iteration, where each update raises up a bit the values of the nodes in . \newline
In some point on of the following eventually happens:
\begin{enumerate}
\item The value of one of the AND nodes in $C$ is raised up to a level that it ceased to be the preferred successor of its OR parent. Namely, in some point in the process of update, there is an AND node $n$, with a sibling $n'$, such that $h(n')<h(n)$.
Then $n'$ becomes the preferred son. Hence, the cycle is no longer a strong connected and the loop ends.
\item The propagation process in $C$ raises up the values of the nodes in $C$ until the values are converged to a certain finite limit.
\end{enumerate}
Clearly, if the values of nodes in $C$ are not converged then case 1 must hold.
In case 2, the propagation may enter into an endless loop if the values are not converged in any finite iteration. In order to overcome this,
each time the value of a node $n\in Z$ is updated, we check the delta $\Delta(n)=h(n)-h_{prev}(n)$, and stop the loop if $\Delta(n)<\epsilon$, where $\epsilon$ is a small positive constant which is chosen before the run, and $h_{prev}(n)$ is the value of $n$ before the update. 
It should be noted that $\epsilon$ is defined to be so small, that it does not change the propagation process(i.e if case 2 holds then case 1 does not hold even if the loop would have never been stopped).\newline
As we proposed a unifying approach to Gen-PAO, we now propose a unifying approach to the AO*. The algorithm AO-EDN is an improvement of the AO* algorithm in which unifies the OR nodes that associate with the same state. In fact AO-EDN is the same algorithm as Gen-PAO-EDN despite that it does not include heuristic HBlocked and HOpen in the propagation phase. 
\end{itemize}
 
%
\begin{figure}[h!]
        \centering
        \centerline{\includegraphics[width=6.5in,height=5in]{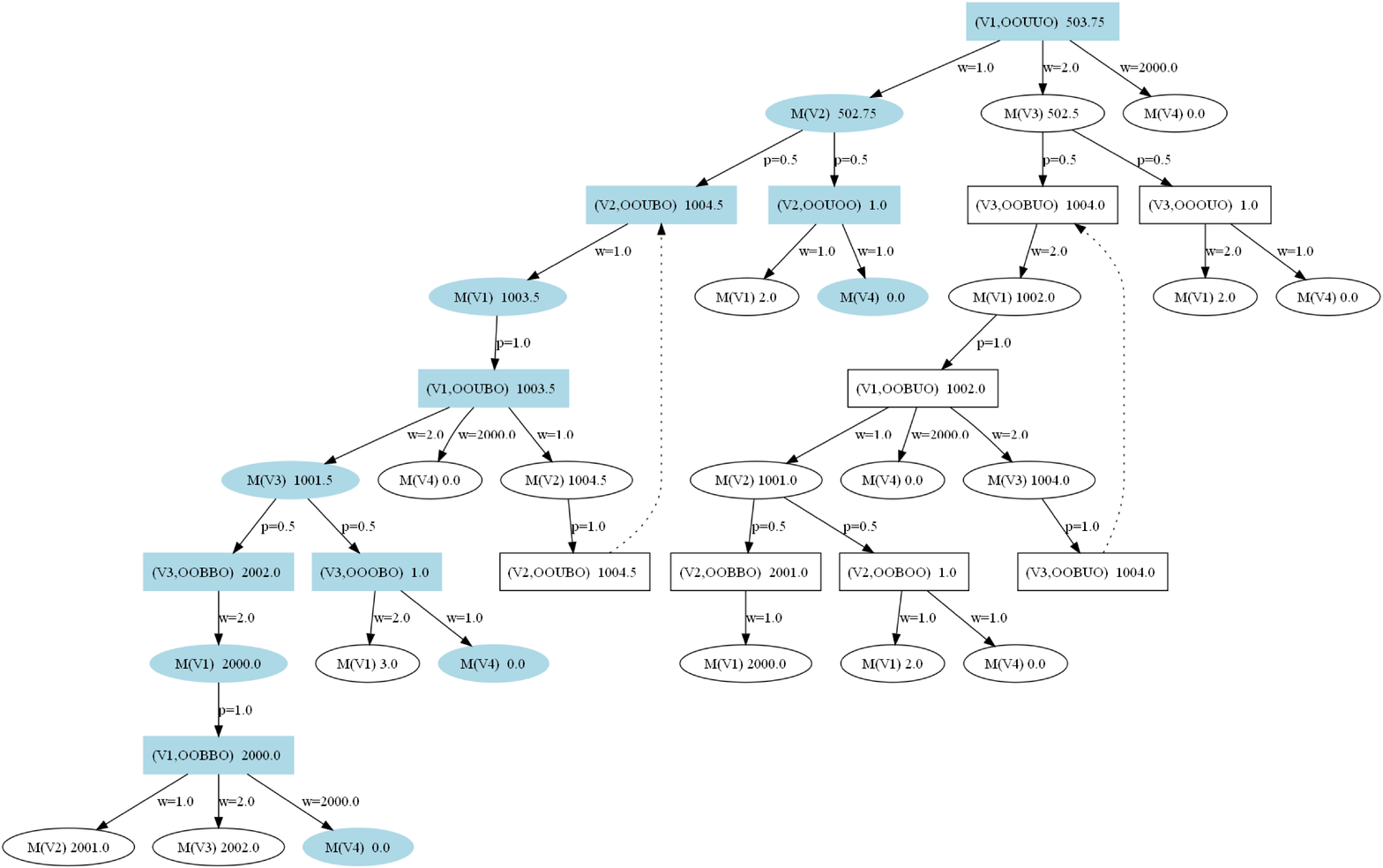}}
        \caption{AND/OR graph for $G_{4V5E}$}
        \label{fig:4V5E}
\end{figure} 


\newpage
\chapter{Empirical Results}

In order to evaluate our scheme we implemented alternative algorithms for the Gen-PAO and compared them by their execution 
time and by the size of their generated AND/OR graph (defined as the number of its nodes).
Note that although the size of the AND/OR graph and the run time of Gen-PAO-EDN may decrease as a result of the heuristic propagation and nodes unification, still the algorithms described in this section requires a time exponential in the number of unknown edges, which makes this approach prohibitive for graphs with large sets of unknown edges.

\section{Varying the Uncertainty of the Graph}

In the first two experiment we explored how the uncertainty of the graph affect the performance of 
 AO* (Section \ref{sec:AO*}),  GenPAO (Section \ref{sec:Gen-PAO}), and AO-EDN (Section \ref{sec:Gen-PAO-Union}). The performance of each algorithm was measured for different graph sizes where each graph had different number of unknown edges.
  To ensure that the experiments could be performed within a reasonable time frame, the parameters were chosen so that a single run takes no more than few minutes.

Figure \ref{fig:UnionComparison} compares the performance of the algorithms above on instances of basic-CTP.
 Figure \ref{fig:ComparisonUnionTime} and figure \ref{fig:ComparisonUnionSize} show respectively the change in the execution time and in the size of AND/OR graph as the number of unknown edges ascend from 2 to 12. This comparison indicates that Gen-PAO has a significant advantage in execution time over AO* since the embedded heuristics in Gen-PAO lowers dramatically the size of the AND/OR graph. Moreover, Gen-PAO has a slight advantage in execution time over AO-EDN although the size of the AND/OR graph generated by AO-EDN is smaller than the graph generated by Gen-PAO. The increased execution time of AO-EDN is incurred by the overhead of the iterative propagation in the redundancy elimination process (Section \ref{sec:Gen-PAO-Union}) in which depends on the value of default edge (default edge cost was chosen to be 100).
\begin{figure}[h!]
        \centering
         \subfloat[Execution time ]{\label{fig:ComparisonUnionTime}\includegraphics[width=3in,height=3in]{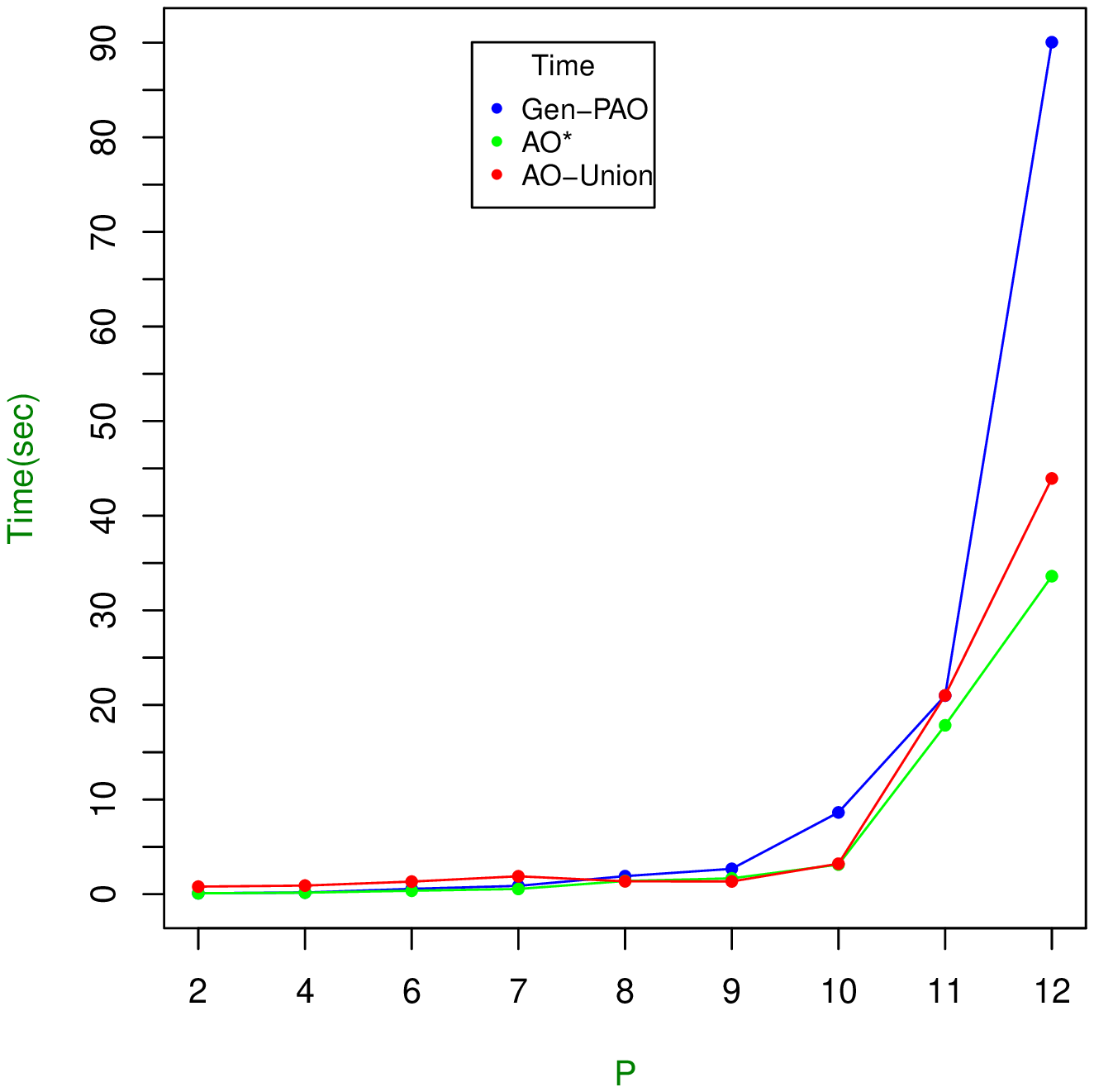}}
        \subfloat[AND/OR graph size]{\label{fig:ComparisonUnionSize}\includegraphics[width=3in,height=3in]{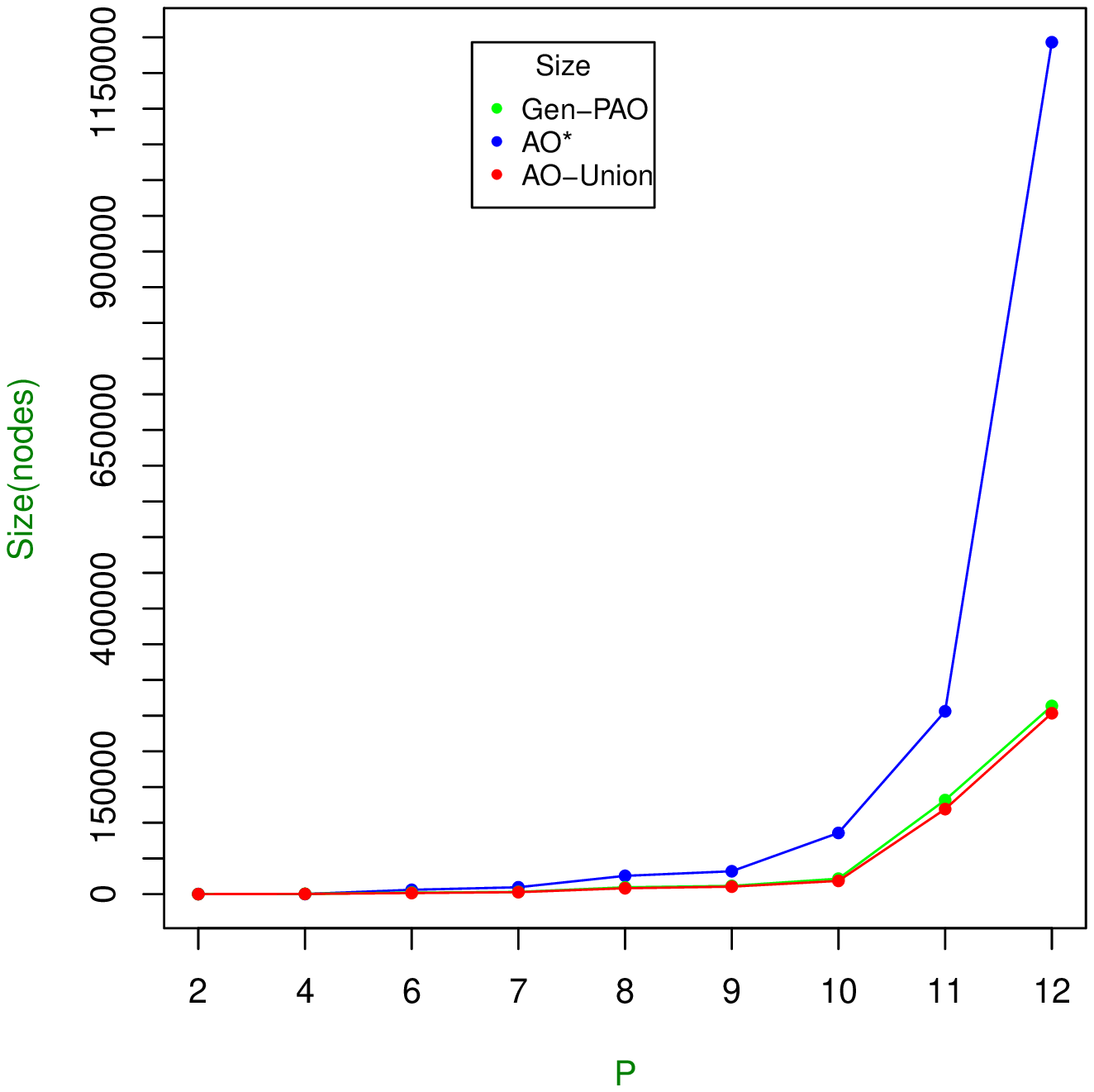}}
					   \caption{Performance of AO*, Gen-PAO and AO-EDN on instances of basic-CTP}
        \label{fig:UnionComparison}
\end{figure} 

Figure \ref{fig:UnionComparisonSense} shows the comparison between AO-EDN and Gen-PAO on instances of Sense-CTP (the sensing cost was fixed to 0.5 for all edges). AO* was discarded from this comparison due to an extremely large execution time. 
 In contrast to previous comparison, here AO-EDN outperforms Gen-PAO in AND/OR graph size (figure \ref{fig:ComparisonUnionSenseSize}) as well as in execution time (figure \ref{fig:ComparisonUnionSenseTime}).  The elimination of redundancy nodes provides an advantage despite the overhead, since the number of expansions saved by the unification increases considerably as the number of unknown edges ascends. The plot does not contain more than 7 unknown edges since Gen-PAO consumes all the RAM on larger graphs.

\begin{figure}[h!]
        \centering
				 \subfloat[Execution time of sensing-CTP (SC=0.5)] {\label{fig:ComparisonUnionSenseTime}\includegraphics[width=3in,height=3in]{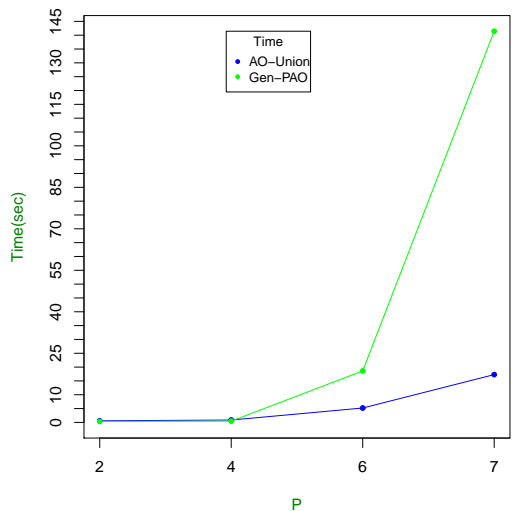}}
				  \subfloat[AND/OR graph size of sensing-CTP (SC=0.5)]  {\label{fig:ComparisonUnionSenseSize}\includegraphics[width=3in,height=3in]{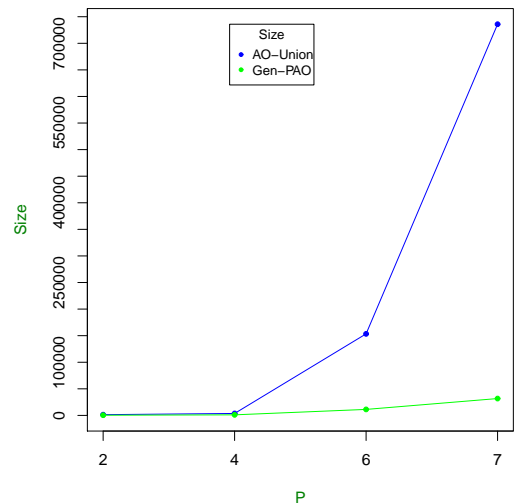}}
				   \caption{Performance of  Gen-PAO and AO-EDN on instances of Sense-CTP}
        \label{fig:UnionComparisonSense}
\end{figure} 

It should be mentioned that since the performances of Gen-PAO-EDN (Section \ref{sec:Gen-PAO-Union}) and Gen-PAO are almost identical on instances of  basic-CTP and Sense-CTP, the performance of  Gen-PAO-EDN is not presented. 
\section{Gen-PAO Heuristic Estimate}
\subsection{Experimental Setting}
We now define a variant of the Canadian Traveler Problem called {\em Expensive Edges CTP} (Exp-CTP in short). Exp-CTP is defined as CTP, except that  each edge $e\in E$ can be $expensive/cheap$ instead of $blocked/unblocked$. 
Formally, Expensive-Edge-CTP is a 6 tuple $I=(G,P,w,s,t,DC)$ where $G=(V,E)$ is a graph, P and w are respectively the probability and cost functions over the edges, $s, t \in V$ are the start and goal vertices, and $DC$ is a positive real number.  $P(e)$ denote the probability that $e$ is cheap and $1-P(e)$ denote the probability that $e$ is expensive. An agent can traverse edge $e\in E$ whether its cheap or expensive. However, if the agent traverses $e$ and $e$ is cheap then it pays $w(e)$, and if $e$ is expensive then it pays $DC$, where $DC$ (short for Detour cost) is a fixed cost which is higher than any edge cost(except the cost of the default edge $\form{s,t}$). In fact Exp-CTP can be defined as a subclass of CTP as well, where every unknown edge $\form{v_i,v_j}$ in G, has a parallel path $l_{ij}=\form{\form{v_i,v_k},\form{v_k,v_j}}$ called \textit{detour path} such that the path cost of $l_{ij}$ is DC and $l_{ij}$ is always traversable. Namely,
\begin{itemize}
	\item  $w(\form{v_i,v_k})=DC$ and $w(\form{v_k,v_j})=0$ 
	\item  $P(\form{v_i,v_k})=1$ and $P(\form{v_k,v_j})=1$ 
\end{itemize}
To evaluate the performance of  Gen-PAO heuristics we implemented four alternative algorithms for Gen-PAO-EDN, where on each algorithm, different heuristic was embedded in the propagation phase. Since the heuristics has almost no impact when Gen-PAO-EDN is applied on instances of basic-CTP and sense-CTP, the algorithms were executed on instances of Exp-CTP.
The implemented algorithms are as follows:

\begin{itemize}
	\item PAO-Blocked - Gen-PAO-EDN which propagates the heuristic values according to HBlocked (Section \ref{sec:Gen-PAO}).
	\item PAO-Open - propagates the heuristic values according to HOpen (Section \ref{sec:Gen-PAO}).
	\item PAO-All - propagates the heuristic values according to HOpen and HBlocked.
	\item PAO-None - basic propagation with no heuristic included. 
\end{itemize}
\subsection{Varying the Sensing Cost}
In order to learn the effect of the sensing cost on the algorithms performance we conducted several runs using different fixed sensing cost(the sensing cost was equal for of all edges) on a graph that consists 8 vertices and 13 edges (10 edges are unknown).
In all experiments, the probability of all unknown edges was fixed to 0.5. 
Figure \ref{fig:7V13ESize} shows the change in the size of AND/OR graph as the sensing cost ascends from 0.1 to 1.1. This result indicates that the size of AND/OR graph (generated by all variants of Gen-PAO) decreases, as the sensing cost increases . We believe that this can be attributed to the increased number of expanded states in the AND/OR graph incurred by the low sensing cost, in which makes the sensing action worthwhile. In particular, there exists a limit $m$, such that every sensing cost below $m$ makes the Sense actions always preferable over the Move actions. This causes many expansions of Sense nodes and expansion of new belief state (that are not reachable without preforming Sense) which results in a large AND/OR graph. 
The comparison of the algorithms shows that PAO-None generates a relatively small AND/OR graph for low sensing cost, while PAO-Blocked and PAO-All has advantage on high sensing cost. This is also true for larger graphs that contain larger sets of unknown edges. 
We believe that this effect can be explained by the fact that on low levels of sensing cost, it is worthwhile to sense unknown edges, in which improves the estimate accuracy of the heuristic values(on low cost). The high accuracy level of the heuristic estimate leads to low rates of pruning since the heuristics HBlocked and HOpen are based the gap between the real and estimated value, which is small in this case. Thus, a large AND/OR graph was obtained.
 A comparison of the run time (figure \ref{fig:7V13ETime}) shows
that the run time extends as the size of the AND/OR graph increases. The reason for this positive correlation is obvious: the increased size of the graph leads to larger computation time required for expanding the states, as well as for propagating the heuristic values to a larger set of states.  
\begin{figure}[h!]
        \centering
            \subfloat[AND/OR graph size]{\label{fig:7V13ESize}\includegraphics[width=3in,height=3in]{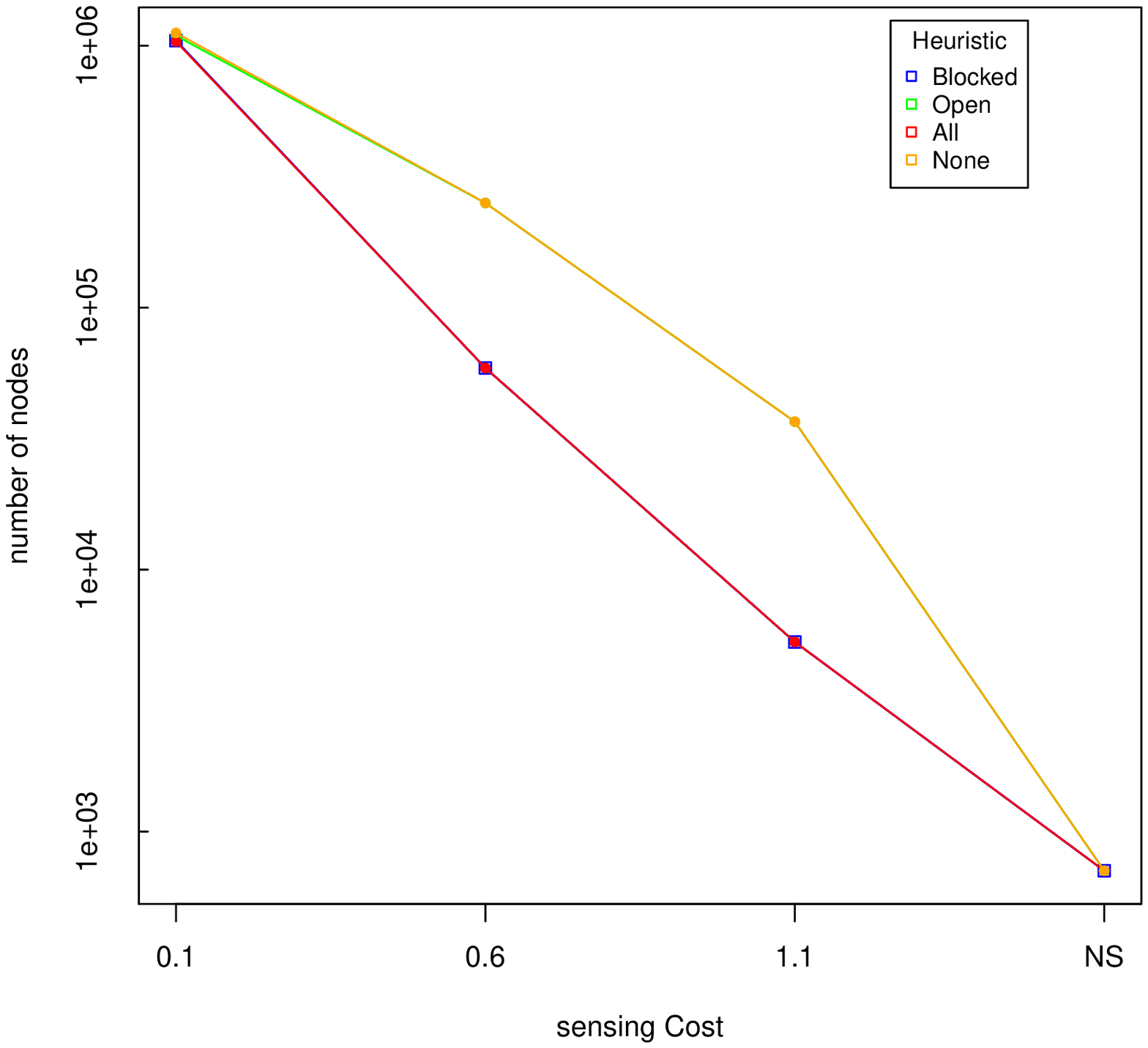}}
         \subfloat[Execution time]{\label{fig:7V13ETime}\includegraphics[width=3in,height=3in]{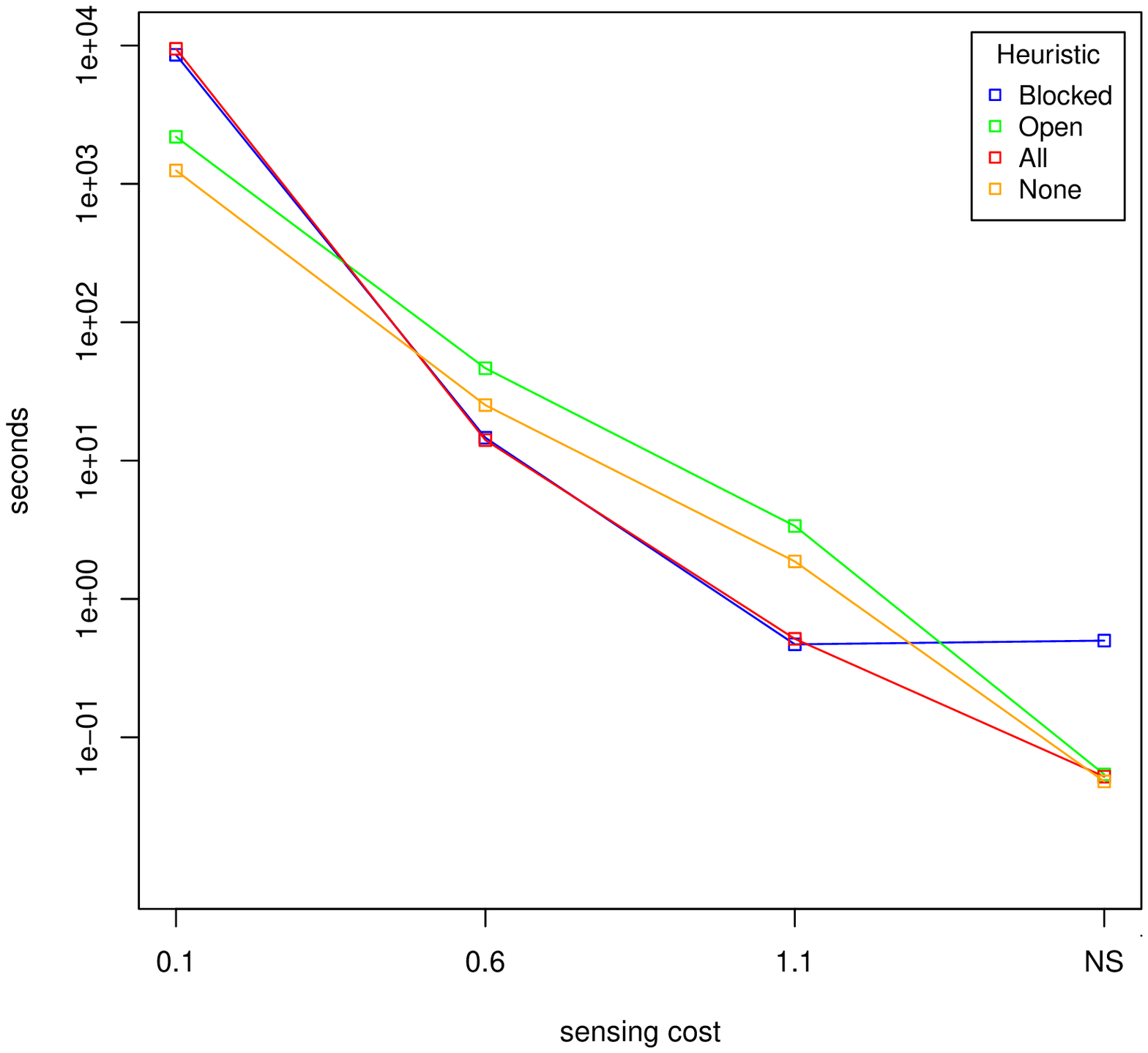}} 
        \caption{Execution time and AND/OR graph size generated by PAO-Blocked, PAO-Open, PAO-All and PAO-None for different values of sensing cost}
        \label{fig:GenPAOSensing}
\end{figure} 
\subsection{Varying the Open Probability}
In this experiment we investigated the effect of distribution over the edges on the performance of variants of Gen-PAO-EDN (Section \ref{sec:Gen-PAO-Union}).
In order to perform simple experiment that analyzes this effect, we configured the graph such that all unknown edges was open with the same value of fixed probability, called $\textit{open probability}$, which is given as an input. 
Figure \ref{fig:GenPAOProb} illustrates the performance of different heuristics on a graph that consists 19 unknown edges for DC=7 and DC=9 . Figures \ref{fig:Exp7Size} and \ref{fig:Exp9Size} show the change in the size of the AND/OR graph size as the open probability ascends from 0.1 to 0.9. These results indicates that for all algorithms there exists a certain value of open probability $p$ (p=0.5 on figure \ref{fig:Exp7Size} and p=0.3 on figure \ref{fig:Exp9Size}) such that for any value of open probability $p'$ (called low open probability) smaller than $p$ the size of the AND/OR graph increases as $p'$ rises, while for any value of open probability $p''$ larger than $p$ (called high open probability) the size of the AND/OR graph decreases as the $p''$ rises. We call $p'$ \textit{low open probability} and $p''$ \textit{high open probability}

This can be explained by the following reasons(referred to AND/OR graphs generated by all algorithms):
\begin{itemize}
\item On high open probability most of the decision nodes (OR nodes) decides correctly their best action node when first expanded without changing their decision afterwards, and thus, relatively large portion of the expanded states is also a part of the optimal policy graph and the AND/OR graph is relatively small. However, as the open probability lowers, the AND/OR graph size increases since the heuristic estimates are less accurate and more alternative actions are considered for the optimal policy. This leads to an excessive expansion of nodes and a larger AND/OR graph.
\item On low open probability, as the open probability lowers, the graph becomes ``more blocked'', and the default path becomes preferable. In such cases, all variants of Gen-PAO-EDN tend to prune action nodes that are not associated with the default path (sensing or traversing edges that are not in the default path) and, as a result, a smaller AND/OR graph is obtained.
\end{itemize}
\begin{figure}[h!]
        \centering
        \subfloat[AND/OR graph size (DC=7)]{\label{fig:Exp7Size}\includegraphics[width=3in,height=3in]{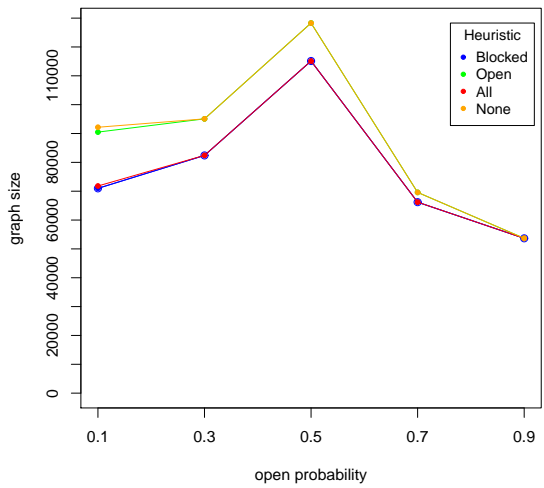}}
        \subfloat[Execution time (DC=7) ]{\label{fig:Exp7Time}\includegraphics[width=3in,height=3in]{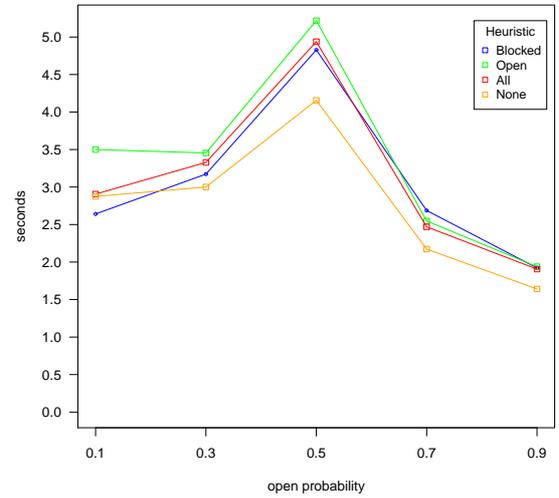}}
				 \\
				  \subfloat[AND/OR graph size (DC=9) ]{\label{fig:Exp9Size}\includegraphics[width=3in,height=3in]{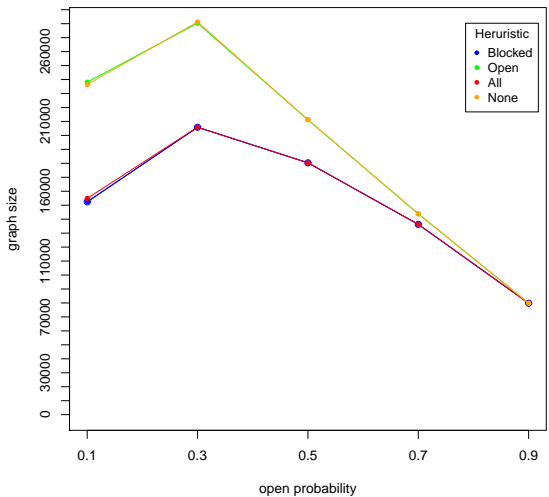}}
          \subfloat[Execution time (DC=9) ]{\label{fig:Exp9Time}\includegraphics[width=3in,height=3in]{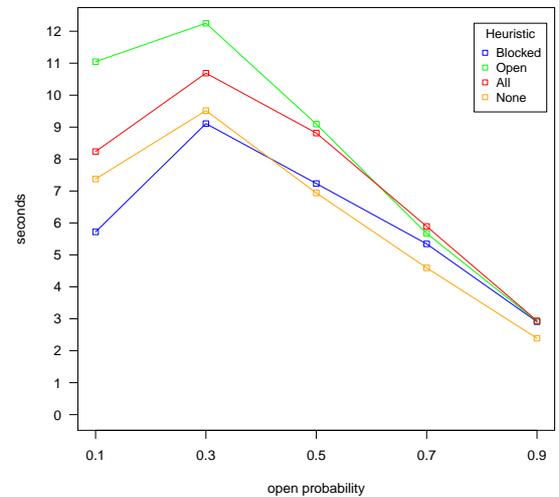}}
        \caption{Execution time and AND/OR graph size generated by PAO-Blocked, PAO-Open, PAO-All and PAO-None for different}
        \label{fig:GenPAOProb}
\end{figure}  
The comparison between the heuristic of Gen-PAO-ELN shows advantage of PAO-Blocked and PAO-All on low open probabilities. This is due to the high pruning rate incurred by HBlocked on low open probability, where the gap between the heuristic values the real values are high. Again HBlocked is effective since the chances that heuristic value of different belief state will be updated are high (see conditions of HBlocked in section \ref{sec:Gen-PAO}). 

%

Figures \ref{fig:Exp7Time} and \ref{fig:Exp9Time} show the time spent by the four algorithms. As in previous experiments, there is a tight correlation between the execution time and the AND/OR graph size. The size of the AND/OR graphs generated by PAO-Blocked and PAO-All are smaller then PAO-None on all levels of open probability, however, the advantage on runtime of PAO-Blocked and PAO-All occurs on low open probability. 

\section{Value of Clairvoyance}
In order to get some general indication of the total value of information, we checked the $\textit{ratio}$ (see Papadimitriou 1991), denoted by $RV$, on instances of basic CTP and Exp-CTP. $RV$ is defined as $\frac{C^*}{AS}$ where $C^*$ is the expected cost of the optimal policy and AS is the expected cost of the optimal policy given that the graph is fully observable (can be also described as the expected cost of the policy \textit{Always Sense} when the sensing cost is 0 (see [Bnaya,Felner and Shimony]). Formally, Let $l_1,l_2,...,l_n$ be the paths in the graph ordered by their path cost, $P_i$ be the probability that path $l_i$ is traversable, and $C_i$ be the path cost of path $l_i$ then $AS$ can be described as follows:
\[
AS=\sum\limits_{i=1}^{n}\prod\limits_{j=1}^{i-1}(1-P_j){P_iC_i}
\]
We performed experiments on instance of basic CTP for different values of open probabilities and values of the default edge. Results for graph 7V11E (figure \ref{fig:VOI7V11Eblocked}) shows that $RV$ is relatively high on low values of the default edge (where default edge cost is 20). This can be explained by the fact that $AS$ is relatively low since the agent would not traverse the default edge if there exists an open path to the target (in addition to the default edge) however $C^*$ is almost high as the cost of default edge since it is usually worthwhile to traverse the default edge when MaxEdge is low (note that $C^*$ is always lower than  MaxEdge). In addition, $RV$ is high on low open probabilities (i.e. on $p\in [0.1,0.3]$) , since the the graph ``tends'' to be blocked and the default edge is preferable over the ``cheap'' paths. Tough on extremely high cost of the default edge (not illustrated in the figure), i.e. on $MaxEdge>300$ , RV is low even on low open probability (around 1.3), since the agent takes the default edge only if there is no open path other then the default edge.

%

\begin{figure}[h!]
        \centering
        \includegraphics[width=5in,height=3in]{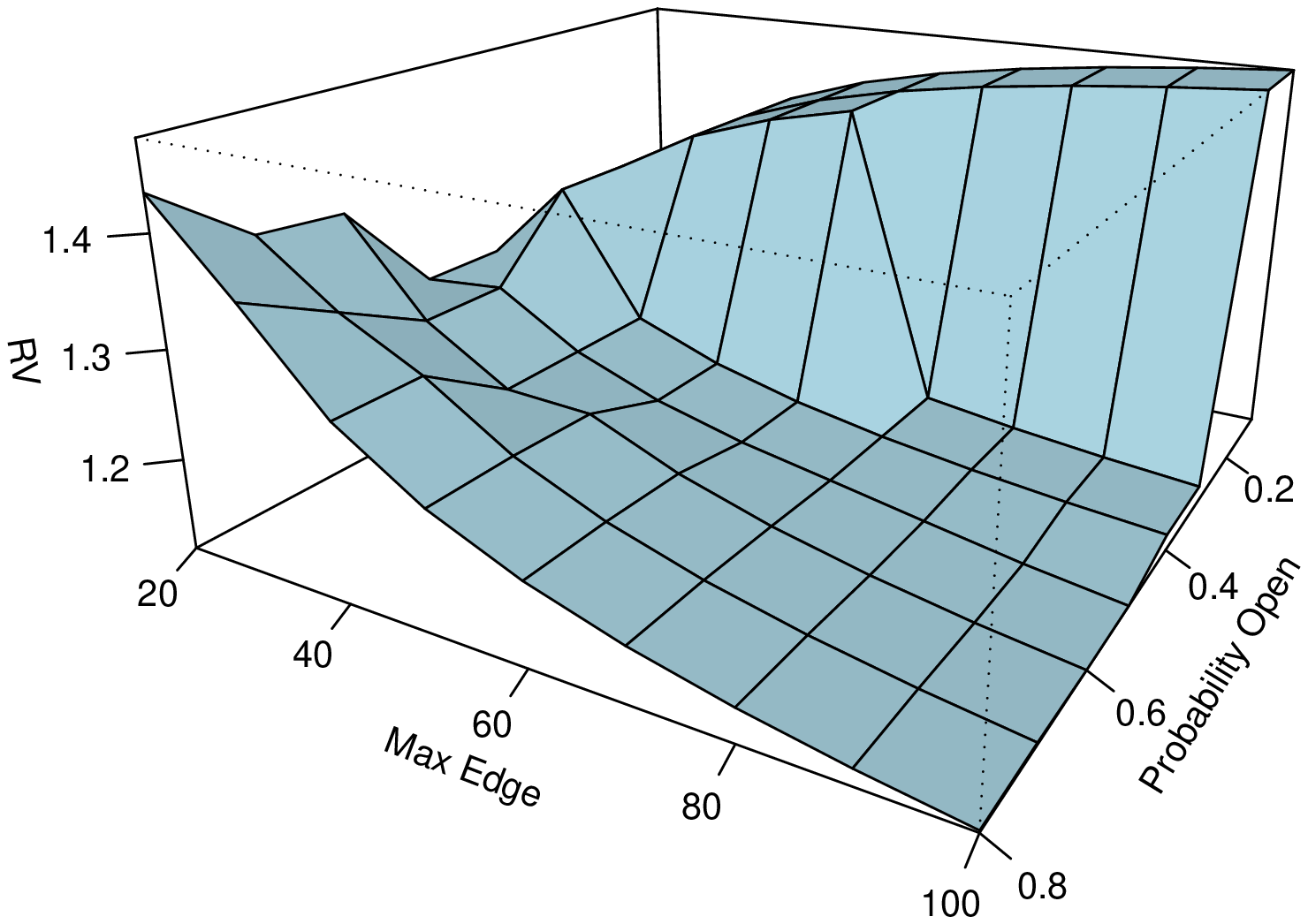}
         \caption{ratio value for graph 7V11E}
        \label{fig:VOI7V11Eblocked}
\end{figure} 

An analogue experiment was performed on instanfce of Exp-CTP for the same graph as used on previous experiment. RV was measured for different values of DC  and open probabilities while default edge cost remained fixed (default edge cost is 200). Figure \ref{fig:VOI7V11Eexpensive} shows that the result is qualitatively similar to the results of the previous experiment, however lower value of RV were obtained in all domain. The reason for this similarity is the same as in the previous experiment, despite that now, the agent prefers to traverse the detour path instead of the default edge. RV is lower than in previous experiment since the paths cost, on average, is higher (it is sometimes required to pay DC several times) and thus $AS$ is higher. 
\begin{figure}[h!]
        \centering
        \includegraphics[width=5in,height=3in]{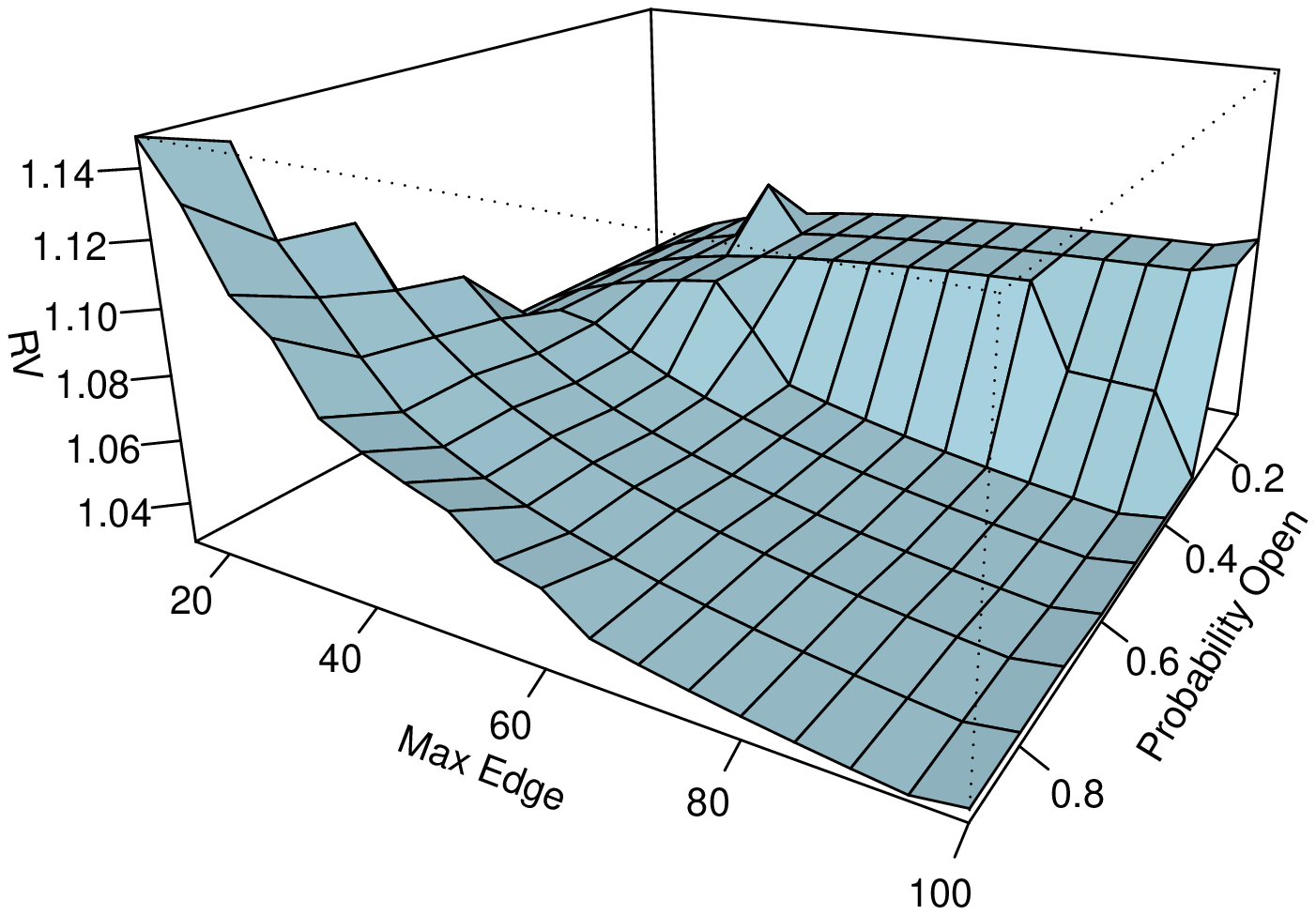}
         \caption{Ratio RV for graph 7V11E}
        \label{fig:VOI7V11Eexpensive}
\end{figure}

\begin{figure}[h]
        \centering
        \centerline{\includegraphics[width=6in,height=9in]{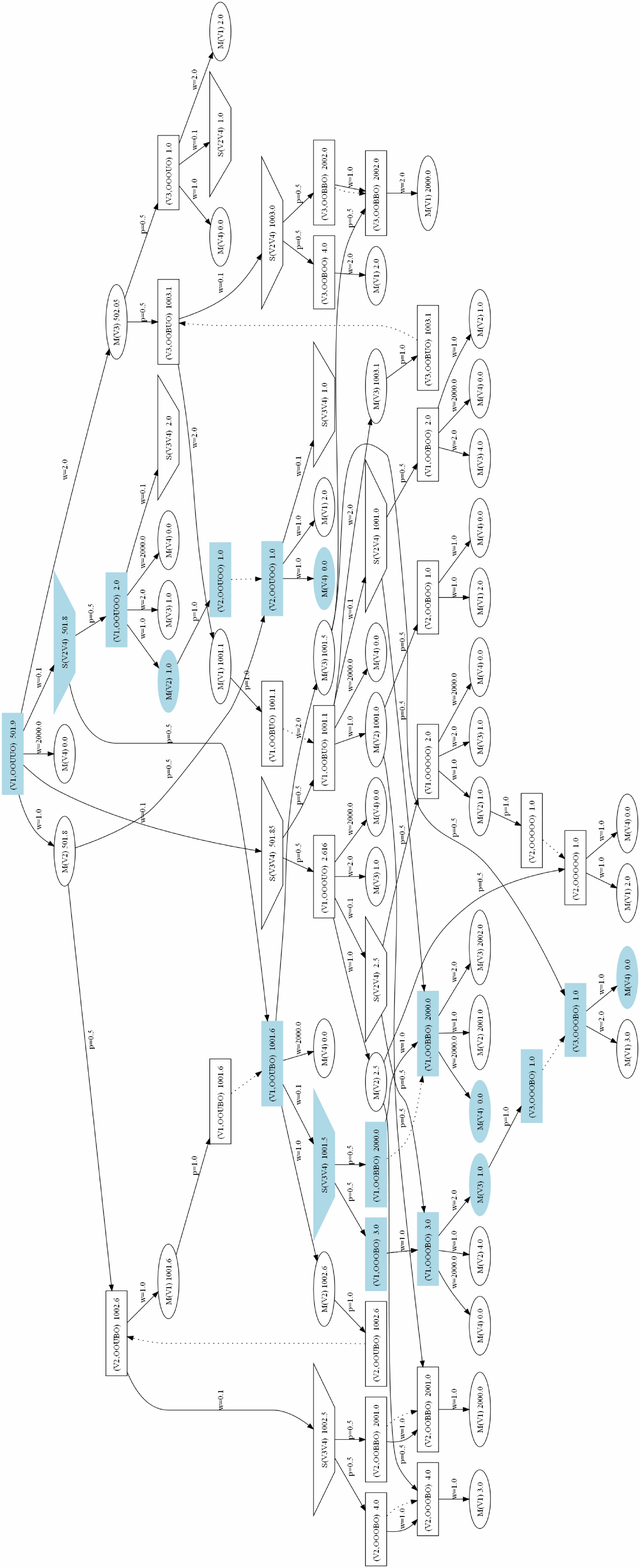}}
        \caption{AND/OR graph for $G_{4V5E}^s$}
        \label{fig:4V5ESensing}
\end{figure} %

\chapter{Summary}

\section{Contributions}
In this thesis we explored the Canadian traveler problem theoretically and empirically. In the context of theoretical analysis the following theorems has been proved:
\begin{itemize}
	\item Correlated-CTP is at least as hard as Sensing-CTP. 
	\item CTP-PATH-DEP is NP-hard. 
	\item CTP-FOR-DEP is solvable in polynomial time. 	
	\item Properties of Belief MDP for CTP.
\end{itemize}
The main aspect of the practical analysis is the framework of Gen-PAO, where its main contributions are:
\begin{itemize}
	\item Gen-PAO extends the PAO* algorithm such that it is not restricted to special types of graphs.
	\item	Gen-PAO optimally solves instances Exp-CTP and sensing CTP in addition to basic CTP. 
	\item Two heuristics HBlocked and HOpen have been proposed. HBlocked and HOpen can be plugged in Gen-PAO and in some cases reduces the size of the AND/OR graph and the execution time.
	\end{itemize}
In addition, we analyzed the parameter RV for instances of Exp-CTP and basic CTP and showed its general behivior.
\section{Future work}
There is a lot remained to be done in theoretical analysis of the CTP, and in particular classifying other subclasses of the CTP.
On the practical aspect,  Gen-PAO can be further modified to solve other type of CTP such as Correlated CTP and multi-agent CTP.
Moreover, we believe that Gen-PAO can be further enhanced by aiming it to other type of POMDP problems. 
It might be worth consideration to improve the performance of Gen-PAO by implementing heuristics that specialize in specific type of graphs.

\bibliographystyle{abbrv}
\bibliography{Refrences}
\appendix{Table of Notations}
\begin{center}
    \begin{tabular}{| l | p{10cm} |}
    \hline
    Notation & Denotation \\ \hline
    $st(e,s)$ & The status of edge $e\in E$, in state $s \in S$. \\ \hline
   	$sts(E,s)$ & The set of all edges status in $E$, in state $s \in S$.\\ \hline
   	$Estatus(s)$ & The set of all edges status which represented by state $s\in S$\\ \hline
    $Loc_S(s)$ & vertex $v\in V$ in which the agent is located, in state $s\in S$ \\ \hline
    $Loc(b)$ & Loc(s) in which $b(s)>0$ \\ \hline
   	$stb(e,b)$ & $
\begin{array}{ll}
O_e &\quad if\; B(e,b) $ - Edge e is know to be Open$ \\
B_e &\quad if\; O(e,b) $ - Edge e is know to be Blocked$\\
U_e &\quad\; otherwise $ - The status of edge e is Unknown$
\end{array}$\\ \hline
 $Unknown(b)$ & The set of all edges $e\in E$ in which $stb(e,b)=U_e$\\ \hline
 $Blocked(b)$ & The set of all edges $e\in E$ in which $stb(e,b)=B_e$\\ \hline

		$P(e,b)$ & $\sum_{i}{b(s_i)}$ such that $st(e,s_i)=B_e$ - the probability that edge e is blocked given b\\ \hline
			$\form{b}$ & $\form{stb(e_1,b),stb(e_2,b),...,stb(e_n,b),Loc(b)}$ - the ``form'' of belief state b.\\ \hline 
$Inc(v)$ & The set of edges incident to a vertex v\\ \hline  
\end{tabular}
\end{center}

\begin{center}
    \begin{tabular}{| l | p{10cm} |}
    \hline
    Predicate & Denotation \\ \hline
		$B(e,b)$ &   True if and only if for every $s\in S$ in which $b(s)>0$ $st(e,s)=B$. Predicate is true if edge e is known to be blocked \\ \hline
		$O(e,b)$ &   True if and only if for every $s\in S$ in which $b(s)>0$ $st(e,s)=O$. Predicate is true if edge e is known to be open\\ \hline
    $V(v,b)$ &   True if and only if for every $s\in S$ in which $b(s)>0$ $Loc(s)=v$. Predicate is true if agent is located in vertex v\\ \hline
	 	$pathB(p,b)$ &  True if $\exists e_i\in p$ such that $B(e_i,b)$. Predicate is true if some edge in path p is known to be blocked\\ \hline
		$unreachable(v,b)$ & True if $Loc(b)=u$ and all paths $\bigcup_{i=1}^{k1}{p_{i}}\in G$ from u to v $pathB(p_{i},b)$. Predicate is true if agent is located in vertex u and all paths from u to v are known to be blocked  \\ \hline
		$unreachable(e,b)$ & True if $Loc(b)=v$ and all paths $\bigcup_{i=1}^{k1}{p_{i}}\in G$ from u to	$v_1\; pathB(p_{i},b)$, all path $\; \; \bigcup_{i=1}^{k2}{p'_{i}}\in G$ from u to $v_2\; pathB( p'_{i},b)$. Predicate is true if agent is located in vertex v and all paths from v to incident vertices of e are known to be blocked\\ \hline
  	 $Obs(b,a)$ & True if and only if b' is reached after taking a in b and there exist at least one edge e in which $stb(b,e)\neq stb(b',e)$.\\ \hline
  	 $REACH_M(b_n,b_1)$ & True if and only if $b_n$ is reachable from $b_1$ in belief-MDP M. i.e. there exist $b_1,...,b_{n-1}$ such that $\prod \limits_{0\leq i\leq n-1,a\in A}Tr(b_i,a,b_{i+1})>0$\\ \hline
  		 \end{tabular}
\end{center}

\end{document}